\documentclass[sigconf]{acmart}

\AtBeginDocument{%
  \providecommand\BibTeX{{%
    \normalfont B\kern-0.5em{\scshape i\kern-0.25em b}\kern-0.8em\TeX}}}

\copyrightyear{2024}
\acmYear{2024}
\setcopyright{acmlicensed}\acmConference[KDD '24]{Proceedings of the 30th
ACM SIGKDD Conference on Knowledge Discovery and Data Mining}{August
25--29, 2024}{Barcelona, Spain}
\acmBooktitle{Proceedings of the 30th ACM SIGKDD Conference on Knowledge
Discovery and Data Mining (KDD '24), August 25--29, 2024, Barcelona, Spain}
\acmDOI{10.1145/3637528.3671718}
\acmISBN{979-8-4007-0490-1/24/08}

\usepackage{url}            
\usepackage{booktabs}       
\usepackage{amsfonts}       
\usepackage{amsmath}

\usepackage{amssymb}
\usepackage{mathtools}
\usepackage{amsthm}
\usepackage{nicefrac}       
\usepackage{microtype}      
\usepackage{algorithm}
\usepackage{algorithmic}
\usepackage{balance}
\usepackage{makecell}

\usepackage{caption}
\usepackage{graphicx}
\usepackage{float} 
\usepackage{subcaption}
\usepackage{pifont}
\usepackage{dsfont}

\newcommand{\BLUE}[1]{{\color{black} #1}}

\newcommand{\bw}{\boldsymbol{w}}
\newcommand{\bu}{\boldsymbol{u}}
\newcommand{\bom}{\boldsymbol{m}}

\usepackage{wasysym}

\newcommand{\bv}{\boldsymbol{v}}

\newcommand{\bnu}{\boldsymbol{\nu}}
\newcommand{\bone}{\beta_1}
\newcommand{\btwo}{\beta_2}

\theoremstyle{plain}
\newtheorem{theorem}{Theorem}[section]

\newtheorem{lemma}[theorem]{Lemma}
\newtheorem{corollary}[theorem]{Corollary}
\theoremstyle{definition}

\newtheorem{assumption}[theorem]{Assumption}

\theoremstyle{remark}
\newtheorem{remark}[theorem]{Remark}
\usepackage{wrapfig}  
\begin{document}

\title{Provable Adaptivity of Adam under Non-uniform Smoothness}

\author{Bohan Wang}
\authornote{Both authors contributed equally to this research.}
\email{bhwangfy@gmail.com}
\orcid{1234-5678-9012}
\affiliation{%
  \institution{University of Science and Technology of China \& Microsoft Research Asia}
  \state{Beijing}
  \country{China}
}

\author{Yushun Zhang}
\authornotemark[1]
\email{yushunzhang@link.cuhk.edu.cn}
\affiliation{%
  \institution{The Chinese University of Hong Kong, Shenzhen}
  \city{Shenzhen}
  \state{Guangdoug}
  \country{China}
}

\author{Huishuai Zhang}
\email{zhanghuishuai@pku.edu.cn}
\authornote{Corresponding authors}
\affiliation{%
  \institution{Peking University}
  \city{Beijing}
  \country{China}}

\author{Qi Meng}
\email{meq@amss.ac.cn}
\affiliation{%
  \institution{Chinese Academy of Mathematics and Systems Science}
  \state{Beijing}
  \country{China}
}

\author{Ruoyu Sun}
\email{sunruoyu@cuhk.edu.cn}
\affiliation{%
 \institution{The Chinese University of Hong Kong, Shenzhen}
 \city{Shenzhen}
 \state{Guangdong}
 \country{China}}

\author{Zhi-Ming Ma}
\email{mazm@amt.ac.cn}
\affiliation{%
  \institution{Chinese Academy of Mathematics and Systems Science}
  \state{Beijing}
  \country{China}}

  \author{Tie-Yan Liu}
  \email{tie-yan.liu@microsoft.com}
  \affiliation{%
    \institution{Microsoft}
    \state{Beijing}
    \country{China}}

\author{Zhi-Quan Luo}
\email{luozq@cuhk.edu.cn}
\affiliation{%
  \institution{The Chinese University of Hong Kong, Shenzhen}
  \city{Shenzhen}
  \state{Guangdong}
  \country{China}
}

\author{Wei Chen}
\email{chenwei2022@ict.ac.cn}
\authornotemark[2]
\affiliation{%
  \institution{Institute of Computing Technology, Chinese Academy of Sciences}
  \state{Beijing}
  \country{China}
  }

\renewcommand{\shortauthors}{Wang and Zhang et al.}

\begin{abstract}

Adam is widely adopted in practical applications due to its fast convergence. However, its theoretical analysis is still far from satisfactory. Existing convergence analyses for Adam rely on the bounded smoothness assumption, referred to as the \emph{L-smooth condition}. Unfortunately, this assumption does not hold for many deep learning tasks. Moreover, we believe that this assumption obscures the true benefit of Adam, as the algorithm can adapt its update magnitude according to local smoothness. This important feature of Adam becomes irrelevant when assuming globally bounded smoothness. 
This paper studies the convergence of randomly reshuffled Adam (RR Adam) {with diminishing learning rate}, which is the major version of Adam adopted in deep learning tasks. We present the first convergence analysis of RR Adam without the bounded smoothness assumption. We demonstrate that RR Adam can maintain its convergence properties when smoothness is linearly bounded by the gradient norm, referred to as the \emph{$(L_0, L_1)$-smooth condition}. {We further compare Adam to SGD when both methods use diminishing learning rate. We refine the existing lower bound of SGD  and show that SGD can be slower than Adam.} 
To our knowledge, this is the first time that Adam and SGD are rigorously compared in the
same setting and the advantage of Adam is revealed.




\end{abstract}



\received{08 February 2024}
\received[revised]{16 June 2024}
\received[accepted]{5 June 2024}

\maketitle

\section{Introduction}
\label{sec: introduction}
Machine learning tasks are often formulated as solving the following finite-sum problem: 
\begin{equation} \label{finite_sum}
  \min _{\bw \in \mathbb{R}^{d}} f(\bw)=\frac{1}{n}\sum_{i=0}^{n-1} f_{i}(\bw),
\end{equation}
where $n$ denotes the number of samples or mini-batches, and $\bw$ denotes the trainable parameters. 
Recently, it is noted that adaptive gradient methods including  Adaptive Moment estimation (Adam) \citep{kingma2014adam}  
are widely used to train modern  deep neural networks including GANs \citep{brock2018large}, BERTs \citep{kenton2019bert}, GPTs \citep{brown2020language} and ViTs \citep{dosovitskiy2020image}. It is often observed that Adam converges considerably faster than vanilla Stochastic Gradient Descent (SGD) for the training of Transformers, as seen in Figure \ref{fig: attention}(a). Similar phenomena are also reported in  \citep{zhang2024transformers}.

Despite its practical success, the theoretical analysis of Adam is less than satisfactory. Existing analyses rely on bounded smoothness assumption, i.e., the Lipschitz coefficient of gradients (or the spectrum norm of the Hessian) is globally upper bounded by constant $L$, referred to as \emph{$L$-smooth condition}. However, recent studies show that the $L$-smooth condition does {\it not} hold in  practical deep learning tasks such as LSTM 
 \citep{zhang2019gradient} and Transformers \citep{crawshaw2022robustness}. 

Moreover, such an assumption hides the benefit of Adam. 
Intuitively, Adam can  overcome the issue of unbounded smoothness using adaptive learning rate. First, Adam uses the reciprocal of the square root of the exponential moving averages of past squared gradients as an effective learning rate (see Algorithm \ref{alg:def_adam} for the update rule). 
Thus, the effective learning rate would be adapted to the local gradient norm. Second, there is a strong correlation between the Lipschitz coefficient and the gradient norm of deep neural networks \citep{zhang2019gradient,cohen2021gradient,crawshaw2022robustness}.  As a result, Adam can adapt the update magnitude to the local Lipschitz coefficient and is empirically observed to converge fast (Figure \ref{fig: attention}(a) and \citep{zhang2019gradient}). Unfortunately, 
such benefit is hidden because existing theories of Adam are built upon $L$-smooth condition.

\BLUE{To reveal the theoretical benefit of Adam, we analyze its convergence under a relaxed smoothness condition called $(L_0,L_1)$-smooth condition \citep{zhang2019gradient}: } 
\begin{equation}
    \label{intro_L0L1}
\|\nabla^2 f_i(\bw)\| \leq L_0 + L_1 \|\nabla f_i(\bw) \|.
\end{equation}
\BLUE{When $L_1 =0$, Eq. \eqref{intro_L0L1} degenerates into classical $L$-smooth condition.}
The $(L_0,L_1)$-smooth condition allows the spectral norm of the Hessian (Lipschitz coefficient of gradients) to linearly grow with the gradient norm of $\bw$, so it is a relaxed version of  $L$-smooth condition. The $(L_0,L_1)$-smooth condition is empirically observed to hold in LSTM \citep{zhang2019gradient, zhang2020improved} and Transformers (Figure \ref{fig: attention}(b) and \citep{crawshaw2022robustness}).

\begin{figure*}[htbp]  
	\centering  
	\begin{subfigure}{0.49\textwidth}  
		\centering  
		\includegraphics[width=0.99\textwidth]{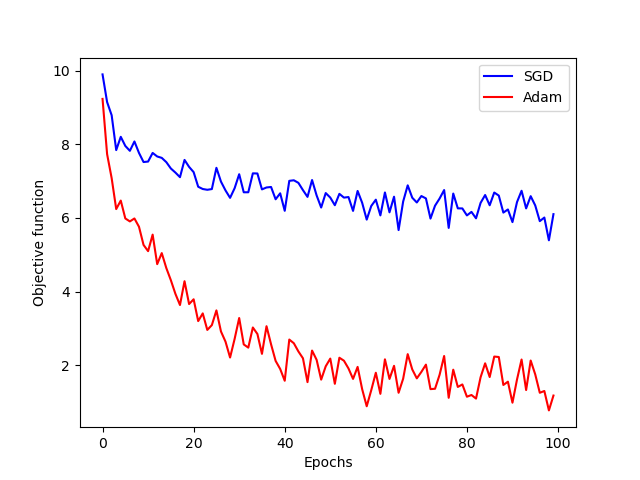}  
		\caption{Training loss}  
	\end{subfigure}  
	\hfill 
	\begin{subfigure}{0.49\textwidth}  
		\centering  
		\includegraphics[width=0.99\textwidth]{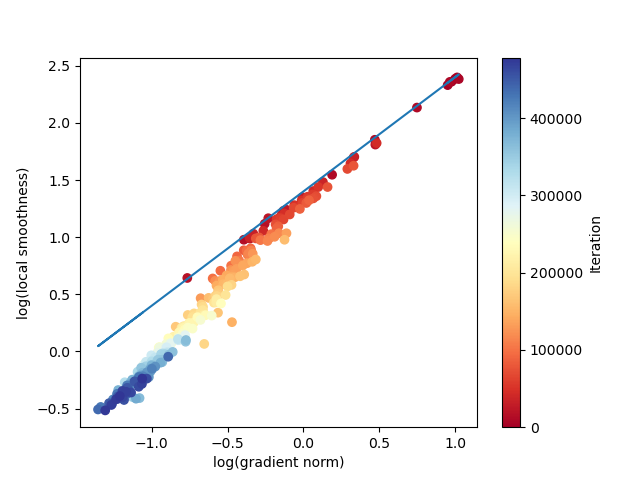}  
		\caption{Gradient vs. smoothness}  
	\end{subfigure}  
	\caption{Experiments on the WMT 2014 dataset trained with the transformer. {\bf (a):} The training loss of SGD and Adam. {\bf (b):} The gradient norm vs. the local smoothness on the training trajectory. The blue line in (b) stands for $\log(\text{local smoothness})= \log(\text{gradient norm})+1.4$. It can be observed that $(e^{1.4},0)$-smooth condition holds in this task. Similar results can be seen in \citet{zhang2019gradient}.}  
	\label{fig: attention}  
\end{figure*}  

\textbf{Our Contribution:} Under the $(L_0, L_1)$-smooth condition, we establish the convergence of randomly-reshuffled Adam. Specifically, our contributions are summarized as follows.
\BLUE{
    \begin{itemize}
        \item We establish the first convergence result of Adam without ``$L$-smoothness". We prove that  Adam converges under the $(L_0,L_1)$-smooth condition.
        
        \item Our convergence result enjoys several good properties.   First,there is no need for the bounded gradient assumption  (i.e. $\|\nabla f(\bw)\| \leq C$). Eliminating this assumption is essential since the $(L_0, L_1)$-smooth condition would otherwise degenerate to the $L$-smooth condition. Second, our result does not rely on other assumptions such as a bounded adaptor or a large regularizer for numerical stability. Lastly, the convergence holds for every possible trajectory, which is not only technically demanding but also much stronger than ``convergence in expectation''.

        \item  {We further compare Adam to SGD when both methods use diminishing learning rate.} We present an improved lower bound for  (S)GD under the $(L_0,L_1)$-smooth condition. In this lower bound,  there is a factor related to the gradient norm of the initial point, which does not exist in the  upper bound of Adam. This  indicates that (S)GD can converge slow under the $(L_0,L_1)$-smooth condition,  showing the advantage of Adam over (S)GD. To our knowledge, this is the  first time that Adam and SGD are rigorously compared in the same setting where the advantage of Adam can be revealed. 
We believe these results shed new light on understanding the benefit of  Adam.
        \end{itemize}
}

\textbf{Organization of this paper.} The rest of this paper is organized as follows: In Section \ref{sec: related_works}, we review related works on the convergence analysis for Adam, the relaxed smoothness assumption, and the variants of Adam. In Section \ref{sec: preliminary}, we define notations, present the psedocode of Adam, and provide the assumptions that our result rests on. In Section \ref{sec: convergence_rate_adam}, we provide our main result on the convergence of RR Adam under non-uniform smoothness together with explanations regarding the result. In Section \ref{sec: proof ideas}, we then state the proof ideas of the main result. In Section \ref{sec: discuss}, we provide discussions on intuitions of why non-adaptive optimizers can be used for fine-tuning tasks, comparison of Adam and Clipped SGD, insights for practioners and limitations of Theorem \ref{thm:rate}.







\section{Related works}
\label{sec: related_works}

\textbf{Convergence analysis for Adam.} Adam is firstly proposed in \citet{kingma2015adam} with a convergence proof. However, the proof is pointed out to have flaws by \cite{reddi2018convergence} and \cite{reddi2018convergence} further provide simple counterexamples with which Adam diverges. This discovery caused the convergence analysis of Adam to stagnate for a while and motivated a series of works developing variants of Adam without divergent issues (see discussion latter in this section). On the other hand, vanilla Adam works well in practice and \BLUE{divergence is not empirically observed.} This phenomenon motivates researchers to rethink the counterexamples. The counterexamples states ``for every $\beta_1 <\sqrt{\beta_2} $, there exists a problem that Adam diverges". That is to say, the divergence statement requires picking $(\beta_1,\beta_2)$ before fixing the problem, while in practice, the algorithmic parameters are often picked according to the problem. Based on this observation, a recent work \citep{Zhang2022Adam} proves that Adam can converge with $(\beta_1,\beta_2)$ picked after the problem is given. 

We categorize the existing results of Adam into two classes based on the sampling strategy: \textbf{with-replacement sampling (a.k.a., i.i.d. sampling, abbreviated as ``WR") and RR Adam.}  We believe both sampling strategies are worth studying: WR is more favored among the theory community due to its simple form, whereas RR is widely used among practitioners because it is easy to implement. Further, RR guarantees to pass each data at least once and brings good performance \citep{bottou2009curiously, bottou2012stochastic}.

The first line of work analyzes WR Adam. For instance,
 \cite{zaheer2018adaptive} shows that WR RMSProp (a simplified version of Adam with $\beta_1 = 0$) converges to the neighborhood of the stationary points. \cite{de2018convergence} prove the convergence of WR RMSProp by assuming the signs
of the gradients to remain the same along the trajectory. However, this condition is not guaranteed to hold in practice. 
\cite{defossez2020simple} prove the convergence of WR Adam with $\beta_1 < \btwo$. However, their convergence bound is inversely proportional to $\xi$, which is the hyperparameter for numerical stability.
 Consequently, their bound becomes vacuous as $\xi$ approaches zero. This result does not match practical observations because small values of $\xi$, like $10^{-8}$, often yield satisfactory performance. Moreover, employing large values of $\xi$ obscures the effect of  $\sqrt{v_k}$, and thus the proof is largely reduced to the proof of SGD.
\cite{huang2021super, guo2021novel} provide simple convergence proof for WR Adam with $\bone$ close to $1$.  However, their results require the  $\sqrt{v_k}$ to be bounded in a certain interval $[C_l, C_u]$. This condition changes Adam into AdaBound \citep{luo2019adaptive}.
In summary, all the above works require certain strong conditions such as bounded  $\sqrt{v_k}$  or large $\xi$. Further, they all require bounded gradient ($\|\nabla f(x)\| \leq C$) and bounded smoothness ($L$-smooth) condition.

Our analysis falls into the second line of works, which focus on RR Adam. \cite{shi2021rmsprop} prove the trajectory-wise convergence of RR RMSProp and \cite{Zhang2022Adam} prove the in-expectation convergence of RR Adam. However, these works both require $L$-smooth condition.  Our analysis follows this line of works and provides the first convergence result of RR Adam under relaxed smoothness condition. 



\noindent
\textbf{Relaxed smoothness assumption.} 
There are several attempts on relaxing $L$-smooth condition. 
\citet{zhang2019gradient} proposes $(L_0,L_1)$-smooth condition  to theoretically explain the acceleration effect of clipped SGD
over SGD. Similar results are also extended to  clipped SGD with momentum \citep{zhang2020improved}, distributionally-robust optimization \citep{jin2021non}
, differentially-private SGD \citep{yang2022normalized} and generalized SignSGD \citep{crawshaw2022robustness}.
   \BLUE{ However, they did not theoretically analyze Adam  in this setting.  Considering the great empirical impact of Adam, we believe it is important to study Adam in its original form.}

   One concurrent work \citep{li2023convergence} studies the convergence of WR Adam under $(L_0,L_1)$-smooth condition by cleverly constructing certain stopping time. They also propose a variance-reduced variant with better convergence rate. However, their bound on Adam not only assumes the noise is deterministically bounded, but also has polynomial dependence over $1/\xi$ (the hyperparameter for numerical stability). Similarly to \citep{de2018convergence}, this result does not match practice observations, since Adam performs well even when  $\xi$ is as small as $10^{-8}$.
   
\noindent \textbf{Variants of Adam.} Ever since the counter-example of the convergence of Adam raised by \cite{reddi2018convergence}, many new variants of Adam have been designed. For instance,  \citet{zou2019sufficient,gadat2020asymptotic,chen2018convergence,chen2021towards} 
replaced the constant hyperparameters by iterate-dependent ones e.g. $\beta_{1t}$ or $\beta_{2t}$. AMSGrad \citep{reddi2019convergence} and AdaFom \citep{chen2018convergence} enforced $\{v_t\}$ to be  non-decreasing.  Similarly,
AdaBound \citep{luo2019adaptive}  imposed constraints  $v_t \in [C_l, C_u]$ to prevent the learning rate from  vanishing or exploding.  Similarly, \cite{zhou2018adashift}  adopted a new estimate of  $v_t$ to correct the bias.
In addition, there are attempts to combine Adam with Nesterov momentum \citep{dozat2016incorporating} as well as warm-up techniques \citep{Liu2020On}. 
There are also some works providing theoretical analysis on the variants of Adam.  For instance, \citet{zhou2018convergence} studied the convergence of  AdaGrad and AMSGrad. \citet{gadat2020asymptotic} studied the asymptotic behavior of a subclass of adaptive gradient methods from landscape point of view. 
Their analysis applies to RMSprop-variants with iterate-dependent $\beta_{2t}$. In summary, all these works study variants of Adam, which is different from our work since we focus on vanilla Adam.

\section{Preliminaries}
\label{sec: preliminary}
This section introduces notations, definitions, and assumptions that are used throughout this work. 


\textbf{Notations.} We list the notations that are used in the formal definition of the randomly-shuffled Adam and its convergence analysis.
\begin{itemize}
        \item (Vector) We define $\boldsymbol{a} \odot \boldsymbol{b}$ as the Hadamard product (i.e., component-wise product) between two vectors $\boldsymbol{a}$ and $\boldsymbol{b}$ with the same dimension. We also define $\langle \boldsymbol{a},\boldsymbol{b}\rangle$ as the $\ell^2$ inner product between $\boldsymbol{a}$ and $\boldsymbol{b}$. We define $\mathds{1}_d$ as an all-one vector with dimension $d$. 
    
    
    \item (Array) We define $[m_1,m_2]\triangleq \{m_1,\cdots,m_2\}$, $\forall m_1,m_2\in \mathbb{N}, m_1\le m_2$. Specifically, we use $[m]$ $\triangleq \{1,\cdots,m\}$.

    \item (Asymptotic notation) We define $A_1(x)=\mathcal{O}_{x\rightarrow a}(A_2(x))$ if $\left\vert\frac{A_1(x)}{A_2(x)} \right\vert $ is bounded when $x\rightarrow $ $a$. We define $A_2(x)=\Omega_{x\rightarrow a}(A_1(x))$ when $A_1(x) $$ =\mathcal{O}_{x\rightarrow a}(A_2(x))$. We use $\tilde{\mathcal O}$ to denote $\mathcal{O}$ with logarithmic factors hidden, i.e., $A_1(x)=\tilde{\mathcal{O}}_{x\rightarrow a}(A_2(x))$ if $A_1(x)=\mathcal{O}_{x\rightarrow a}(A_2(x)\log \vert A_2(x)\vert )$. When the context is clear, we hide "$x\rightarrow a$" and only use $\mathcal{O},\Omega,\tilde{\mathcal{O}}$.
\end{itemize}

\noindent\textbf{Pseudocode.} To facilitate the  analysis, we provide the pseudocode of Adam  in  Algorithm \ref{alg:def_adam}.
\begin{algorithm}[htb!]
   \caption{Randomly reshuffled Adam (RR-Adam)}
   \label{alg:def_adam}
\begin{algorithmic}
   \STATE {\bfseries Input:} Objective function $f(\bw):=\frac{1}{n}\sum_{i=0}^{n-1} f_i(\bw)$, learning rate series $\{\eta_{k}\}_{k=1}^{T}$ and hyperparameters $(\beta_1,\beta_2)\in [0,1)^2$. Initialize the parameter $\bw_{1,0} \in \mathbb{R}^d$, the conditioner $\bnu_{1,-1}\in \mathbb{R}^{d,\ge 0}$, and the momentum $\bom_{1,-1}\in \mathbb{R}^{d}$.
   \FOR{$k=1$ {\bfseries to} $T$}
   \STATE Randomly shuffle $[0,n-1]$ to get $\{\tau_{k,j}\}_{j=0}^{n-1}$
  \FOR{$i=0$ {\bfseries to} $n-1$}
  \STATE  Calculate $g_{k,i}=\nabla f_{\tau_{k,i}} (\bw_{\tau_{k,i}})$
   \STATE Update $\bnu_{k,i}=\btwo\bnu_{k,i-1}+(1-\btwo)g_{k,i}^{\odot2}$,
   \STATE Update $\bom_{k,i}=\bone\bom_{k,i-1}+(1-\bone)g_{k,i}$
   \STATE Update $\bw_{k,i+1}= \bw_{k,i}-\eta_k \frac{1}{\sqrt{\bnu_{k,i}}+\xi}\odot \bom_{k,i}$
   \ENDFOR
   \STATE Update \small $\bnu_{k+1,-1}= \bnu_{k,n-1}$, $\bom_{k+1,-1}= \bom_{k,n-1}$, $\bw_{k+1,0}=\bw_{k,n}$
   \ENDFOR
\end{algorithmic}
\end{algorithm}

 $\bom_{k,i}$ and $\bnu_{k,i}$ are weighted averages with hyperparamter $\beta_1 \in [0,1)$ and $\beta_2 \in [0,1)$, respectively.
$\xi$ is
adopted for numerical stability and it is often chosen to be $10^{-8}$ in practice. In our theory, we allow $\xi$ to be an arbitrary non-negative constant including 0.

 Algorithm \ref{alg:def_adam} follows a without-replacement sampling strategy (also known as shufﬂing), which is the default strategy used in CV, NLP, GANs, etc. However, it is not necessarily easy to analyze shuffling strategy, because the stochastic gradients sampled by random-shuffling lack statistical unbiasedness, i.e. $\mathbb{E}\left[\nabla f_{k,i}(x_{k,i})\vert x_{k,i} \right] \neq \nabla f(x_{k,i})$. This bias requires a much different analysis from its with-replacement counterpart. Even for SGD, the  analysis for shuffling is often known to be ``more challenging"
 \citep{tran2021smg,mishchenko2020random}.  However, we choose to study this version as it is closer to the practice.
 


\noindent\textbf{Assumptions.} Here we state the assumptions that our result will rest on. The first one is the $(L_0,L_1)$-smooth condition introduced in Section \ref{sec: introduction}.
\begin{assumption}[$(L_0,L_1)$-smooth condition]
\label{assum:regular} We assume that  $f_i(\bw)$ is lower bounded by $0$, and $f_i(\bw)$ satisfies  $(L_0,L_1)$-smooth condition, i.e.,  there exist positive constants ($L_0$, $L_1$), such that, $\forall \bw_1,\bw_2\in \mathbb{R}^d$ satisfying $\Vert \bw_1-\bw_2\Vert \le \frac{1}{L_1}$,
\begin{equation}
\label{eq: smoothness}
    \Vert \nabla f_i(\bw_1)-\nabla f_i(\bw_2)\Vert \le (L_0+L_1 \Vert \nabla f_i(\bw_1)\Vert)\Vert \bw_1-\bw_2\Vert .
\end{equation}
\end{assumption}
 Eq. (\ref{eq: smoothness}) 
 is firstly introduced by \citet{zhang2020improved}, and is the weakest version of $(L_0,L_1)$-smooth condition to our best knowledge since it does not require $f_i(\bw)$ to be twice  differentiable.
When $f_i(\bw)$ is twice differentiable, Eq. (\ref{eq: smoothness}) is equivalent to Eq. \eqref{intro_L0L1} \citep{zhang2020improved}. $(L_0,L_1)$-smooth condition generalizes the $L$-smooth condition (i.e., $(L_0,L_1)$-smooth condition with $L_0=L$ and $L_1=0$) in classical non-convex optimization literature \cite{ghadimi2013stochastic,liu2020improved} and allows the smoothness to be unbounded globally.



\begin{assumption}[Affine Noise Variance]
\label{assum:GC} $\forall \bw \in \mathbb{R}^d$, the gradients of  $\{f_i(\bw)\}_{i=0}^{n-1}$ has the following connection with the gradient of $f(\bw)$:
\begin{equation*}
\label{eq:GC}
    \frac{1}{n}\sum_{i=0}^{n-1}\left\|\nabla f_{i}(\bw)\right\|^{2} \leq D_{1}\|\nabla f(\bw)\|^{2}+D_{0}.
\end{equation*}
\end{assumption}
Assumption \ref{assum:GC} is one of the weakest assumption on gradient noise in existing literature. It not only generalizes the ``bounded variance" assumption (which requires $D_1 = 1/n$, and thus further generalizes the "bounded gradient" assumption \citep{defossez2020simple}) \citep{ghadimi2016mini, Manzil2018adaptive, huang2021super}, but also is weaker than the ``strongly growth condition" (which requires $D_0 = 0$) \citep{schmidt2013fast,vaswani2019fast}. Assumption \ref{assum:GC}  allows flexible choices of $D_0$ \& $D_1$ and thus it is among the weakest assumption of this kind.

\section{ Adam Converges under the $(L_0,L_1)$-smooth condition}
\label{sec: convergence_rate_adam}

In this section, we provide our main result on the convergence of RR Adam under $(L_0,L_1)$-smooth condition. As discussed in Section \ref{sec: introduction}, even for the simpler with-replacement sampling Adam, the establishment of the convergence under $(L_0,L_1)$-smooth condition requires restrictive assumptions such as a large $\xi$ (the constant introduced for numerical stability and is as small as $10^{-8}$ in practice), and deterministically bounded noise \citep{li2023convergence}. Such assumptions make the corresponding results hard to apply to practical setting. As for the harder randomly-reshuffled setting, there is no convergence result for Adam under non-uniform smoothness. Our result tackles the limitation in existing works and propose the first convergence result for RR Adam under non-uniform smoothness, provided as follows. 

\begin{theorem}
\label{thm:rate}
Consider RR Adam defined as Algorithm~\ref{alg:def_adam} with diminishing learning rate $\eta_k=\frac{\eta_1}{\sqrt{k}}$. Let Assumptions \ref{assum:regular} and \ref{assum:GC} hold. Suppose the hyperparamters satisfy:  $0\le \bone^2<\btwo<1$ and  $\beta_2$ is larger than a threshold $\gamma ( D_1)$. Then, we have
\begin{small}
    \begin{align}
    \nonumber
    \min_{k\in [1,T]}\left\{\frac{\Vert\nabla f(\bw_{k,0})\Vert }{\sqrt{D_1 } },\frac{\Vert\nabla f(\bw_{k,0})\Vert^2 }{\sqrt{D_0 } }\right\}
    \le &\tilde{\mathcal{O}}\left(\frac{f(\bw_{1,0})-\min_{\bw}f(\bw)}{
    \sqrt{T}}\right)
    \\
   \label{eq: convergent_rate}
    &+\mathcal{O}((1-\btwo)^2\sqrt{D_0}). 
\end{align}  
\end{small}
\end{theorem}

For simplicity, we defer the concrete form of $\gamma$ to Appendix \ref{appen: restate_thm}. We provide some remarks on the results as follows, and state the proof idea in the next section.

\textbf{Explanation for Theorem \ref{thm:rate}.} Theorem \ref{thm:rate} is pioneering in demonstrating that RR Adam is capable of converging under the non-uniform smoothness condition, a finding that is novel to our best knowledge. Observing the right-hand side of inequality (\ref{eq: convergent_rate}), one can see that as $T \to \infty$, it approaches $\mathcal{O}((1-\beta_2)^2\sqrt{D_0})$. This suggests that Adam's convergence to the vicinity of stationary points is inversely related to the proximity of $\beta_2$ to 1. This theoretical insight corroborates the common practice of choosing $\beta_2$ close to $0.99$. A counterexample provided later will further illustrate that convergence to a neighborhood, rather than an exact point, is an intrinsic characteristic of the algorithm.   
  
Beyond the $(L_0,L_1)$-smooth condition, Theorem \ref{thm:rate} presupposes only that the gradient noise exhibits affine variance as per Assumption \ref{assum:GC}, which is a relatively mild constraint that eschews the need for a bounded gradient norm. This is crucial, as imposing a bound would reduce the $(L_0, L_1)$-smooth condition to an $(L_0+L_1M)$-smooth condition with a gradient norm capped by $M$. Additionally, we do not require the adaptive learning rate $\eta_k/\sqrt{\hat{\nu}_k}$ to be upper bounded, nor do we stipulate a large regularizer $\xi$—aligning with common practices in deep learning libraries where a small $\xi$ such as $10^{-8}$ is often effective. Our theorem permits any non-negative $\xi$, including zero. Finally, Theorem \ref{thm:rate} asserts convergence for every possible trajectory, a guarantee that exceeds the typical "convergence in expectation" results and poses a significant technical challenge.

\textbf{On the Comparison to Existing Analyses of RR Adam.}  
Our analysis extends the applicability of RR Adam by ensuring convergence under the $(L_0,L_1)$ smooth condition, which inherently encompasses the traditional $L$-smooth condition. This broadened perspective allows our results to guarantee convergence for RR Adam even under the more general $L$-smooth scenario. When juxtaposed with the state-of-the-art analysis of RR Adam under the $L$-smooth condition by \citet{Zhang2022Adam}, our findings advance the field in two significant ways. Firstly, we elevate the notion of convergence from the expected sense to the more stringent trajectory-wise convergence. Secondly, we refine the estimated convergence neighborhood, tightening it from $(1-\beta_2) \sqrt{D_0}$ to $(1-\beta_2)^2 \sqrt{D_0}$. Collectively, our analysis not only operates under a less restrictive assumption—the $(L_0,L_1)$-smooth condition—but also delivers substantively enhanced convergence results.

\textbf{On the range of hyperparameters.}
Theorem \ref{thm:rate} indicates that Adam can work when  $\btwo$ is close enough to $1$. This matches the practical choice of $\beta_2$ (e.g., $0.999$ in default setting, $0.95$ in the GPT-3 training \citep{brown2020language}). Note that our result does not contradict the  counterexamples of Adam's non-convergence \citep{reddi2018convergence,Zhang2022Adam}, as these divergence results require $\btwo$ to be small and thus not close to $1$.  Rather, these counterexamples suggest that large $\btwo$ is necessary for convergence. As for $\bone$, Theorem \ref{thm:rate} needs $\beta_1^2 <\beta_2$. When $\beta_2$ is large,  Theorem \ref{thm:rate} allows a wide range of candidates of $\beta_1$ (e.g., $0.9$ in default setting and $0.5$ in GAN \citep{radford2015unsupervised}).

\begin{figure}[htb!]  
  \centering  
  \includegraphics[width=0.25\textwidth]{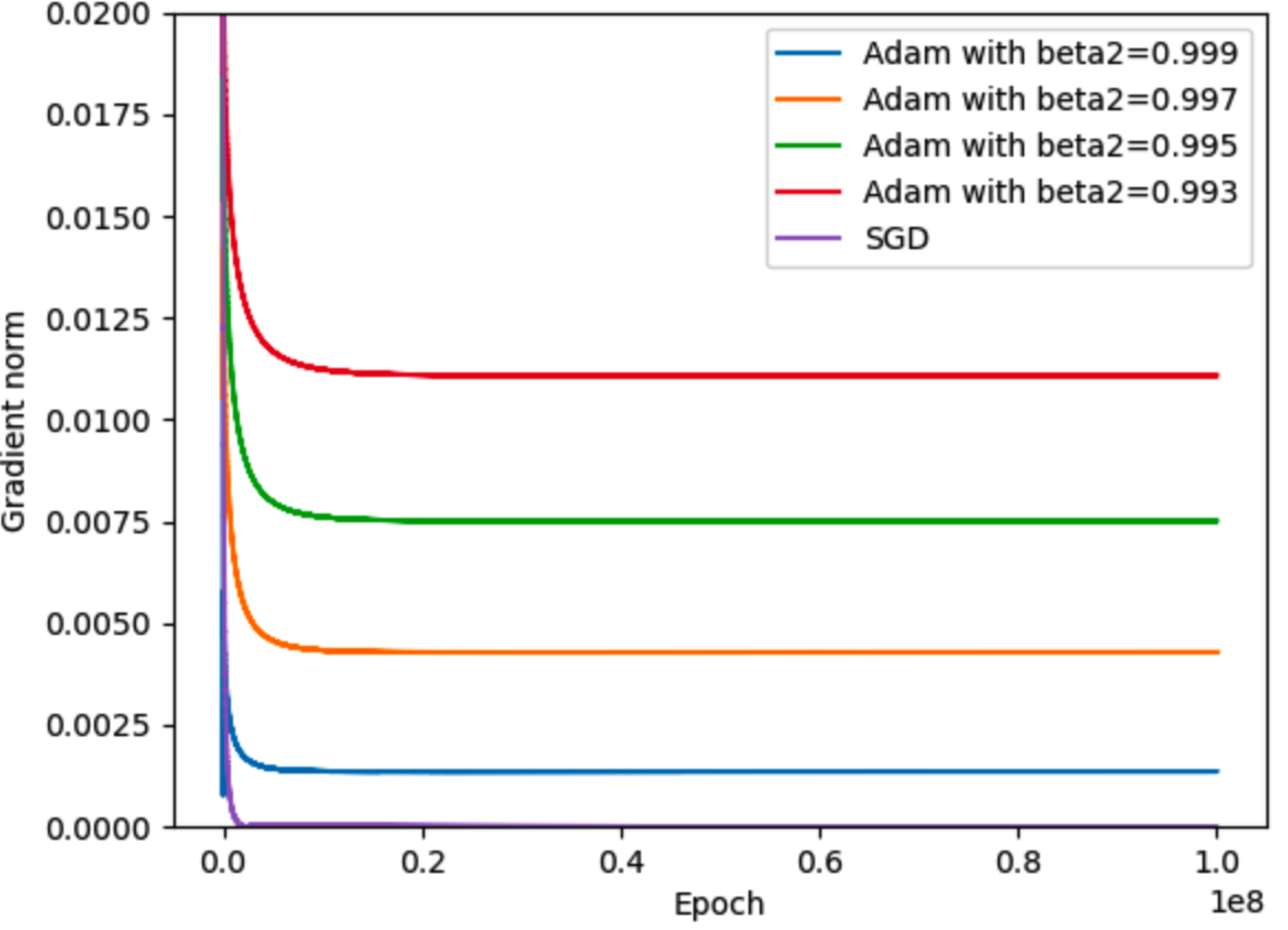}  
  \caption{ Reconduct of experimental results from \citep{Zhang2022Adam}. The objective function is defined in Eq. (\ref{eq: counter_example}). One can observe that while letting $\btwo$ closer to $1$ can make the limiting gradient norm smaller, the limiting gradient norm always stabilizes beyond 0.  }  
  \label{fig: D0}  
\end{figure} 

\textbf{On the neighborhood of stationary points.} When $D_0\ne 0$,  Theorem \ref{thm:rate}  only ensures that Adam converges to a neighborhood of stationary points $\{\bw: \min$ $\{\frac{\Vert\nabla f(\bw))\Vert }{\sqrt{D_1 } },\frac{\Vert\nabla f(\bw)\Vert^2 }{\sqrt{D_0 } }\}\le \mathcal{O}((1-\btwo)\sqrt{D_0})\}$. Since SGD converges to the stationary points with diminishing learning rate, one may wonder if Theorem \ref{thm:rate} can be improved to obtain the same conclusion as SGD. Unfortunately, there is a counterexample in the existing literature ( function (9) in \citet{Zhang2022Adam}) showing that \textbf{Adam does {\it not} converge to stationary points} even if all the conditions in Theorem \ref{thm:rate} are satisfied.
Specifically, \cite{Zhang2022Adam} consider the following function:
\begin{equation}
\label{eq: counter_example}
\begin{aligned}
&f(x)=\sum_{j=0}^{9} f_{j}(x)=\frac{1}{10} x^{2}-1,
\\
&\text{ where } f_{j}(x)=\left\{\begin{array}{l}(x-1)^{2} \text { if  } j=0 \\ -0.1\left(x-\frac{10}{9} \right)^{2} \text { if  } 1 \leq j \leq 9\end{array}\right..
\end{aligned}
\end{equation}
One can easily verify such an example satisfies Assumptions \ref{assum:GC} and \ref{assum:regular} with $D_0 >0$.  
As shown in Figure \ref{fig: D0}, when running Adam (with $\beta_1 = 0.9$, $\eta_k = 0.1 /\sqrt{k}, a= 3, x_0 =-2$), it does not converge to exact stationary points. Instead, it converges to a neighborhood of stationary points with size inversely proportional to $\beta_2$. Therefore, the non-vanishing term in Theorem \ref{thm:rate} is \textit{not} due to the limitation of the proof. Rather, it is an intrinsic property of Adam.

Why cannot Adam converge to exact stationary points when $D_0 >0$?
Intuitively, this is because even with diminishing $\eta_k$, the effective learning rate $\frac{\eta_k}{\xi\mathds{1}_d+\sqrt{\bnu_{k,i}}}$ may not diminish due to the potentially decreasing $\sqrt{\bnu_{k,i}}$.
The good news is that $\mathcal{O}((1-\btwo)\sqrt{D_0})$  approaches $0$ as $\btwo$ gets close to $1$. This means that the neighborhood shrinks as $\btwo\rightarrow 1$ (this is also observed in Figure \ref{fig: D0}). As discussed above, the practical use of $\btwo$ is close to $1$, and thus $O((1-\btwo)\sqrt{D_0})$ is tolerable.

\textbf{On the Diminishing Learning Rate.}
In Theorem \ref{thm:rate}, we consider a diminishing learning rate of the form $\eta_k = \frac{\eta_1}{\sqrt{k}}$ to maintain consistency with existing works on RR Adam, such as those by \cite{shi2021rmsprop,Zhang2022Adam}, which also employ diminishing learning rates. Nonetheless, our results can be readily extended to RR Adam with a constant learning rate. By adhering to the same proof strategy outlined in Theorem 1, one can demonstrate that with a constant learning rate $\eta$, the conclusion (as given in Eq. (\ref{eq: convergent_rate})) of Theorem 1 is modified to $\min_{k\in [1,T]} \{\frac{\Vert \nabla f(w_{k,0}) \Vert}{\sqrt{D_0}}, \frac{\Vert \nabla f(w_{k,0}) \Vert2}{\sqrt{D_1}}\} \le \tilde{\mathcal{O}}\left(\frac{f(\bw_{1,0})-\min_{\bw}f(\bw)}{
    \eta T}\right) +  \mathcal{O}(\sqrt{D_0}(1-\btwo)^2) + \mathcal{O}(\eta)$. In essence, while Adam may converge more rapidly to a neighborhood (with the rate improving from $1/\sqrt{t}$ to $1/t$), the size of this neighborhood is increased by an additional term $\mathcal{O}(\eta)$, which is attributable to the constant step size.

\vspace{-2mm}

\section{Proof sketch}
\label{sec: proof ideas}
{
In this section, we briefly explain our proof idea for Theorem \ref{thm:rate}, which can be divided into two stages. In Stage I, we will prove Theorem \ref{thm:rate} for Adam with $\bone=0$ to show the challenge brought by $(L_0,L_1)$-smooth condition and how we tackle it. In Stage II, we then show the additional difficulty when adding the momentum and our corresponding intuition to solve it.

\textbf{Stage I: Convergence of Adam with $\bone =0$.} By the descent lemma,
\begin{align}
\nonumber
   \begin{matrix} f(\bw_{k+1,0}) -f(\bw_{k,0}) \le &\\
   \quad\vspace{2mm}
   \end{matrix}&  
   \begin{matrix} \underbrace{  \langle \bw_{k+1,0}-\bw_{k,0},\nabla f(\bw_{k,0}) \rangle} \\ \text{First Order} \end{matrix}
   \begin{matrix} \\
    \quad\vspace{2mm}
   \end{matrix}
   \\
   &\begin{matrix}\underbrace{+\frac{L_{loc} }{2}\Vert\bw_{k+1,0}-\bw_{k,0} \Vert^2,}
   \\
    \text{Second Order}
   \end{matrix}\label{eq: f_descent}
\end{align}
where $L_{loc}$ is the local smoothness.  We bound the first-order and the second-order term respectively. The upper bound on second-order term is relatively simple. Due to the limited space, we only show the idea of bounding first-order term here.


The ever-changing adaptive learning rate poses a challenge on deriving the bound. It is even noted that with small $\btwo$, the first order term can be positive \citep{reddi2018convergence}. However, we notice that if $\bnu_{k,i}$ is stationary, i.e., RMSProp degenerates to SGD with preconditioning, the first order term equals to  $-\eta_k\langle \sum_{i} \frac{1}{\xi\mathds{1}_d+\sqrt{\bnu_{k,0}}} \odot \nabla f_{\tau_{k,i}}(\bw_{k,i}), \nabla f(\bw_{k,0})\rangle \approx -\eta_k \langle \sum_{i} \frac{1}{\xi\mathds{1}_d+\sqrt{\bnu_{k,0}}} \odot \nabla f_{\tau_{k,i}}(\bw_{k,0}), \nabla f(\bw_{k,0})\rangle$, which is indeed negative. While that "$\bnu_{k,i}$ is stationary" is too good to be true, we prove that $\bnu_{k,i}$ changes little when $\btwo$ is close to $1$, assuming that the gradient is large. Below we denote $\bnu_{l,k,i}$ as the $l$-th component of $\bnu_{k,i}$.

\begin{lemma}[Informal]
\label{lem: main_difference_nu}
For any $l\in [d]$ and $i\in[0,n-1]$, if  $\max_{p\in[0,n-1]} \vert \partial_l f_p(\bw_{k,0})\vert= \Omega(\sum_{r=1}^{k-1}\btwo^{\frac{(k-1-r)}{2}}\eta_r$ $\Vert \nabla f(\bw_{r,0})\Vert +\eta_k)$, then $\vert \bnu_{l,k,i}-\bnu_{l,k,0} \vert =\mathcal{O}((1-\btwo)\bnu_{l,k,0}) $.
\end{lemma}

The idea of Lemma \ref{lem: main_difference_nu} is simple: since $\bnu_{k,i}=\btwo\bnu_{k,i-1}+(1-\btwo) \nabla  f_{\tau_{k,i}} (\bw_{\tau_{k,i}})^{\odot 2}$, the change of $\bnu_{k,i}$ w.r.t. $i$ should be small when $\btwo$ is large. However, we need to check that the relative size of $\nabla  f_{\tau_{k,i}} (\bw_{\tau_{k,i}})^{\odot 2}$ w.r.t. $\bnu_{k,i-1}$ is uniformly bounded across varying $\btwo$, otherwise the term $(1-\btwo) \nabla  f_{\tau_{k,i}} (\bw_{\tau_{k,i}})^{\odot 2}$ may not go to zero when $\btwo\rightarrow 1$. We resolve this challenge by expanding $\bnu_{k,i}$ in terms of squared gradients and bounding the gap between each of the terms and $\nabla  f_{\tau_{k,i}} (\bw_{\tau_{k,i}})^{\odot 2}$ by echoing $(L_0,L_1)$-smooth condition. We defer a detailed proof to Corollary \ref{coro: large_derivative} for details.

As a conclusion, if we denote those dimensions with large gradients (i.e., satisfying the requirement of Lemma \ref{lem: main_difference_nu}) as $\mathbb{L}_{large}^k$ and the rest as $\mathbb{L}_{small}^k$, Lemma \ref{lem: main_difference_nu} indicates that the $\mathbb{L}_{large}^k$ part (i.e., $\sum_{l\in \mathbb{L}_{large}^k} (\bw_{l,k+1,0}-\bw_{l,k,0})\partial_l f(\bw_{k,0})$) in the first order term can be bounded as 
\small
\begin{flalign*}
    & -\eta_k\sum_{l\in \mathbb{L}_{large}^k}\frac{\partial_l f(\bw_{k,0})}{\sqrt{\bnu_{l,k,i}}+\xi} \sum_{i} \partial_l f_{\tau_{k,i}} (\bw_{k,i})\\
    \approx& -\eta_k \sum_{l\in \mathbb{L}_{large}^k}
\left(\frac{\partial_l f(\bw_{k,0})^2}{\sqrt{\bnu_{l,k,0}}+\xi} +\mathcal{O}\left((1-\btwo)\frac{\partial_l \vert f(\bw_{k,0})\vert   \sum_{i} \vert\partial_l f_{\tau_{k,i}} (\bw_{k,i})\vert  }{\sqrt{\bnu_{l,k,0}}+\xi}\right)\right)
\\
    =& - \Omega \left( \eta_k\min\left\{\frac{\Vert\nabla f(\bw_{k,0})\Vert }{\sqrt{D_1 } },\frac{\Vert\nabla f(\bw_{k,0})\Vert^2 }{\sqrt{D_0 } }\right\}\right)
    +O(\eta_k(1-\btwo) \sqrt{D_0}).
\end{flalign*}
\normalsize

The last equation uses the affine noise assumption (Assumption \ref{assum:GC}), and we defer a detailed proof to Appendix \ref{appen: convergence}. A remaining problem is how to deal with those components in $\mathbb{L}_{small}^k$. We treat them as error terms. Concretely, $l\in \mathbb{L}_{small}^k$ indicates 
that $\partial_l f(\bw_{k,0}) =\mathcal{O}(\sum_{r=1}^{k-1}\btwo^{\frac{(k-1-r)}{2}}$ $\eta_r\Vert \nabla f(\bw_{r,0})\Vert +\eta_k)$. Applying it directly into $\sum_{l\in \mathbb{L}_{small}^k} (\bw_{l,k+1,0}-\bw_{l,k,0})\partial_l f(\bw_{k,0})$, we have
\begin{small}
\begin{align*}
 &-\eta_k\sum_{l\in \mathbb{L}_{large}^k}\frac{\partial_l f(\bw_{k,0})}{\sqrt{\bnu_{l,k,i}}+\xi} \sum_{i} \partial_l f_{\tau_{k,i}} (\bw_{k,i}) 
 \\
=&\mathcal{O}\left(\eta_k\left(\sum_{r=1}^{k-1}\btwo^{\frac{(k-1-r)}{2}}\eta
 _r\Vert \nabla f(\bw_{r,0})\Vert +\eta_k\right)\right),
\end{align*}
\end{small}
where the equation is because $\frac{\partial_l f_{\tau_{k,i}}(\bw_{k,i})}{\sqrt{\bnu_{l,k,i}}+\xi}$ is bounded (proved by Lemma \ref{lem: bounded_update}).

In order to upper bound the first order term, we then need to prove that $ - \Omega( \eta_k\min\{\frac{\Vert\nabla f(\bw_{k,0})\Vert }{\sqrt{D_1 } },\frac{\Vert\nabla f(\bw_{k,0})\Vert^2 }{\sqrt{D_0 } }\})$ dominates \\ $\mathcal{O}(\eta_k(\sum_{r=1}^{k-1}\btwo^{\frac{(k-1-r)}{2}}\eta
 _r\Vert \nabla f(\bw_{r,0})\Vert +\eta_k))$. This is not necessarily true, as the historical gradient norms in the latter term can be large.
 \begin{remark}
 We recognize this as the challenge brought by $(L_0,L_1)$-smooth condition, since the latter term degenerates to  $\mathcal{O}(\eta_k^2)$ with $L$-smooth condition, which is minor ($\sum_{k=1}^T\eta_k^2$ is only in order $\log T$).
 \end{remark}
 We address this challenge by noting that what we need to bound is the sum of the first order term. Fortunately, although we cannot upper bound the first order term in one single epoch, we can bound the sum of it across epochs. By a sum order change, the sum of  $\mathcal{O}(\eta_k(\sum_{r=1}^{k-1}\btwo^{\frac{(k-1-r)}{2}}\eta
 _r\Vert \nabla f(\bw_{r,0})\Vert +\eta_k))$ over $k$ equals to  $\mathcal{O}(\sum_{k=1}^T$ $\eta_k^2\Vert \nabla f(\bw_{k,0})\Vert+\ln T)$. This is smaller by the sum of $ - \Omega( \eta_k\min\{\frac{\Vert\nabla f(\bw_{k,0})\Vert }{\sqrt{D_1 } },\frac{\Vert\nabla f(\bw_{k,0})\Vert^2 }{\sqrt{D_0 } }\})$ by order of $\eta_k$ except a $\ln T$ term due to the mean value inequality, i.e., 
 \begin{equation*}
     \eta_k^2 \Vert \nabla f(\bw_{k,0})\Vert\le \mathcal{O}(\eta_k^2)+\mathcal{O}\left(\eta_k^2\sqrt{\frac{D_1}{D_0}} \Vert \nabla f(\bw_{k,0})\Vert^2\right).
 \end{equation*} We then conclude the sum of the first order term can be bounded by $ - \Omega( \eta_k\min\{\frac{\Vert\nabla f(\bw_{k,0})\Vert }{\sqrt{D_1 } },\frac{\Vert\nabla f(\bw_{k,0})\Vert^2 }{\sqrt{D_0 } }\})+\mathcal{O}(\ln T)$.

\begin{remark}[Difficulty compared to the analysis under $L$-smooth condition] Here we illustrate the challenge brought by stepping beyond $L$-smooth condition.  First of all, the change of $\bnu_{k,i}$ is easier to bound without the historical gradient term due to the absence of the gradient norm in the bound of local smoothness. Secondly, under $L$-smooth condition, the error does not contain historical gradient information and is only in order of $\mathcal{O}(\eta_k^2)$, which is easy to bound.
 \end{remark}




\textbf{Stage II: adding the momentum.} The second order term of Adam can be bounded similarly. However, the analysis of the first order term becomes more challenging even though we still have $\bnu_{k,i}\approx \bnu_{k,0}$. Specifically, even with constant $\bnu_{k,i}=\bnu_{k,0}$, $-\eta_k\langle\sum_i \frac{\bom_{k,i}}{\sqrt{\bnu_{k,i}}+\xi},- \nabla f(\bw_{k,0})\rangle>0$  is not necessarily correct, as the momentum $\bom_{k,i}$ contains a heavy historical signal, and may push the  update away from the negative gradient direction.

We resolve this challenge by observing that the alignment of $\bw_{k+1,0}-\bw_{k,0}$ and $-\nabla f(\bw_{k,0})$ is required due to that our analysis is based on the potential function $f(\bw_{k,0})$. However, while this potential function is suitable for the analysis of RMSProp, it is no longer appropriate for Adam based on the above discussion. We need to construct another potential function. Our construction of the potential function is based on the following observation: we revisit the update rule in Algorithm~\ref{alg:def_adam} and rewrite it  as
   $\frac{ \bom_{k,i}-\bone \bom_{k,i-1}}{1-\bone}=\nabla f_{\tau_{k,i}} (\bw_{k,i}).$

Notice that the right-hand-side of the above equation contains no historical gradients but only the gradient of the current step! By dividing $(\sqrt{\bnu_{k,i}}+\xi)/\eta_k$ above,
\small
\begin{align*}
 \frac{\bw_{k,i+1}-\bw_{k,i}-\bone(\bw_{k,i}-\bw_{k,i-1}) }{1-\bone}
 \approx &-\frac{\eta_k}{\sqrt{\bnu_{k,0}}+\xi \mathds{1}_d}
 \odot \frac{ \bom_{k,i}-\bone \bom_{k,i-1}}{1-\bone}
 \\
 =&-\frac{\eta_k}{\sqrt{\bnu_{k,0}}+\xi\mathds{1}_d}\odot\nabla f_{\tau_{k,i}} (\bw_{k,i}).
\end{align*}
\normalsize
After simple rearrangement, one can see that the sequence $\{\frac{\bw_{k,i}-\bone\bw_{k,i-1}}{1-\bone}\}$ are (approximately) doing SGD within one epoch (with coordinate-wise but constant learning rate $\bnu_{k,i}$)! We define
\begin{equation*}
    \bu_{k,i}\triangleq \frac{\bw_{k,i}-\bone\bw_{k,i-1}}{1-\bone}.
\end{equation*}
Then, further notice that the distance between $\bu_{k,i}= \bw_{k,i}+\bone\frac{\bw_{k,i}-\bw_{k,i-1}}{1-\bone}$ and $\bw_{k,i}$ is in order of one step's update, and thus $\bu_{k,i}\approx \bw_{k,i}$ .  Therefore, we choose our potential function as $f(\bu_{k,i})$. The Taylor's expansion of $f$ at $\bu_{k,0}$ then provides a new descent lemma, i.e.,
\begin{small}
\begin{align}
\nonumber
   \begin{matrix} f(\bu_{k+1,0}) -f(\bu_{k,0}) \le\\
   \quad\vspace{2mm}
   \end{matrix}&
   \begin{matrix} \underbrace{  \langle \bu_{k+1,0}-\bu_{k,0},\nabla f(\bu_{k,0}) \rangle} \\ \text{First Order} \end{matrix} 
   \begin{matrix} \\
    \quad\vspace{2mm}
   \end{matrix}
   \\
&\begin{matrix}\underbrace{~+~\frac{L_{0}+L_1 \Vert\nabla f(\bw_{k,0}) \Vert  }{2}\Vert\bw_{k+1,0}-\bw_{k,0} \Vert^2,}
   \\
    \text{Second Order}
   \end{matrix}\label{eq: u_loss}
\end{align}
\end{small}

By noticing $\bw_{k,i}\approx \bu_{k,i}\approx\bu_{k,0}$, the first order term can be further approximated by $
 -\langle \frac{\eta_k}{\sqrt{\bnu_{k,0}}+\xi \mathds{1}_d}\odot\nabla f (\bw_{k,0}),\nabla f(\bw_{k,0}) \rangle$
which is negative. The rest of the proof is the same as that of Stage I.

\begin{remark}[On Why State-of-the-Art Results Do Not Achieve Trajectory-Wise Convergence as Ours]
The state-of-the-art analysis of RR Adam under the $L$-smooth condition, as presented by \citet{Zhang2022Adam}, also addresses the misalignment between $\bw_{k+1,0}-\bw_{k,0}$ and $-\nabla f(\bw_{k,0})$. However, their approach does not employ a potential function, resulting in convergence results that are restricted to in-expectation guarantees. Specifically, \citet{Zhang2022Adam} manage this misalignment by assuming a uniform distribution over all possible shuffling orders and demonstrating that, under this assumption, the expected value of $\bw_{k+1,0}-\bw_{k,0}$ is approximately equal to $-\nabla f(\bw_{k,0})$. In contrast, our methodology introduces an auxiliary function, $f(\bu_{k,i})$, and examines the dynamics of $\bu_{k,i}$. This approach shifts the challenge from aligning $\bw_{k+1,0}-\bw_{k,0}$ with $-\nabla f(\bw_{k,0})$ to aligning $\bu_{k+1,0}-\bu_{k,0}$ with $-\nabla f(\bw_{k,0})$. Such a strategy simplifies the analytical process and facilitates the demonstration of trajectory-wise convergence.
\end{remark}

\begin{remark}[Similar potential functions in the existing literature.] We notice that similar potential functions have already been applied in the analysis of other momentum-based optimizers, e.g., momentum (S)GD in \citep{ghadimi2015global} and \citep{liu2020improved}. However, extending the proof to Adam is highly-nontrivial. The key difficulty lies in showing that the first-order expansion of $f(\bu_{k,0})$ is positive, which further requires that the adaptive learning rate does not change much within one epoch. This is hard for Adam as the adaptive learning rate of Adam can be non-monotonic. The lack of L-smooth condition makes the proof even challenging due to the unbounded error brought by gradient norms.
\end{remark}

}

\section{Comparison Between Adam and SGD}
\label{sec: counter_sgd}

Now we compare the convergence rate of Adam with SGD. To do so, we need a lower bound of SGD {\bf in the same setting} as Theorem \ref{thm:rate}. There are several existing lower bounds of SGD under $(L_0,L_1)$ smoothness condition (e.g., \citep{zhang2019gradient,crawshaw2022robustness}). However, we find these lower bounds cannot be directly applicable for comparison with Adam. This is because:

\begin{itemize}
    \item 1) In the lower bound of \citep{zhang2019gradient,crawshaw2022robustness}, they pick the learning rate \emph{before} the construction of the objective function and initialization point (we restate their lower bound in Appendix B.1 for completeness). In other words, it is possible that if we fix the objective function and tune the learning rate (which is a common practice in the training of deep neural networks), SGD can converge very fast. For rigorous comparison with Adam, we need a lower bound with reversed ordering. That is, we need the following statement:
  ``consider a fixed objective function and initialization point, then no matter how we pick the learning rate, SGD suffers from a certain rate. "
   \item 2) The lower bounds of SGD in \citep{zhang2019gradient,crawshaw2022robustness} require  constant learning rate. However, since Adam in Theorem \ref{thm:rate} uses diminishing-learning-rate, we aim to establish a lower bound of SGD with diminishing learning rate.
\end{itemize}

Unfortunately, there is no existing lower bound that satisfies the above two properties. 
In the following theorem, we provide a refined lower bound of SGD in the setup that we desired.

\begin{theorem}
\label{thm: sgd_diminishing}
For any $L_0,L_1, T>0$, there exists an objective function $f$ obeying Assumption \ref{assum:regular}, and an initialized parameter $\bw_0$ satisfying $M=\sup\{\Vert\nabla f (\bw)\Vert: f(\bw)\le f(\bw_0) \}$, such that $\forall \eta_1>0$, the iterations of SGD $\{\bw_t\}_{t=0}^{\infty}$ satsifies $\min_{t\in T}\Vert \nabla f(\bw_t)\Vert^2 =\Omega(M(f(\bw_0)-\min_{\bw\in \mathbb{R}^d}f(\bw))/\sqrt{T})$.
\end{theorem}


The proof can be in Appendix \ref{appen: diverging}. The proof idea is mainly motivated by \cite{zhang2019adam}. We highlight some differences when we try to reach the two properties mentioned previously. 

\begin{itemize}
      \item To reverse the ordering of ``picking learning rate and functions \& initialization", we simply augment the worst-case example in \cite{zhang2019adam} into 2 dimensional space. It turns out this simple trick is effective in the proof.
    \item To change constant learning rate into diminishing learning rate, we show that:  when the initial learning rate $\eta_0$ is larger than a certain threshold, the decay rate of the learning rate cannot offset the curvature explosion along the iteration, causing divergence; on the other hand, when initial $\eta_0$  is small, it would lead to slow convergence. This is a new finding in $(L_0,L_1)$ setting. We prove this result by mathematical induction. This part of the discussion is not required in the lower bound of   \cite{zhang2019adam} with  constant learning rate.
\end{itemize}


\paragraph{Comparison between Adam and SGD.}
Finally, we discuss the implication the lower bound of SGD (Theorem \ref{thm: sgd_diminishing}) and the upper bound of Adam (Theorem \ref{thm:rate}). 
In the lower bound of  SGD, there is an extra constant $M$  which does not appear in the upper bound of Adam. 
This allows us to compare the convergence rates of these two algorithms. 

We summarize our findings as follows. We emphasize that Theorem \ref{thm:rate}  and Theorem \ref{thm: sgd_diminishing} share exactly the same setting: both consider function class under the same assumptions; both SGD and Adam use diminishing learning rate. Therefore, the following comparison is rigorous.  

\textbf{Finding 1: } When $D_0 = 0$. 
Adam converges to stationary point with rate $\mathcal{O}\left(\frac{1}{T}\right)$ while GD converges with rate $\mathcal{O}\left(\frac{1}{\sqrt{T}}\right)$. So Adam converges (to stationary points) faster.

\textbf{Finding 2: } When $D_0 > 0$. There exists a set of $\bw$ with infinite Lebesgue measure, such that, when starting at any $\bw$ in this set, Adam converges (to the neighborhood of stationary points) faster than SGD.

    Note that the above statement “algorithm 1 converges faster than algorithm 2” does not mean that algorithm 1  always converges faster than algorithm 2. For sure, rarely can anyone make such a strong statement. The above statement actually means that “the worst-case complexity of algorithm 1 is faster than that of algorithm 2, and both complexity bounds can be simultaneously achieved when working on the  same function and starting at the same initialization”. This definition is adopted from  \citep{sun2021worst}, and it is a   widely accepted definition in the optimization field.

    \begin{proof}
     \textbf{Finding 1} can be directly proved by plugging $D_0= 0$ into Theorem \ref{thm:rate} and squaring the inequality. 
    We now prove \textbf{Finding 2}. First, we  state an important fact from the proof of Theorem \ref{thm: sgd_diminishing}.

\paragraph{Fact 1:} For the counter-example we constructed in Theorem \ref{thm: sgd_diminishing}.  $M=\sup\{\Vert\nabla f (\bw)\Vert: f(\bw)\le f(\bw_0) \}$ goes to infinity as $\|\bw\|$ goes to infinity.  Further, for any $C >0$, the set $\{\bw: M >C \}$ is of infinite Lebesgue measure.

    Based on \textbf{Fact 1}, for the worst-case example in Theorem \ref{thm: sgd_diminishing}, there must exist a region in $\mathbb{R}^d$ where  $M$ is larger than all the constant terms in the upper bound of Adam in Theorem \ref{thm:rate}. Further, 
    Such region is of infinite Lebesgue measure. When running Adam and SGD simultaneously on this worst-case example starting from any $\bw$ in this region, the constants in the upper bound of Adam is smaller than the constants in the lower bound of SGD. Since the upper and lower bounds share the same rate,  so Adam is faster.  Note that there is an additional constant term in the upper bound of Adam \eqref{eq: convergent_rate}, so we conclude that Adam converges to the neighborhood of stationary points faster than SGD.
    \end{proof}
    
    Note that when  $D_0 > 0$, Adam is still guaranteed to converge faster, but only to the neighborhood in lieu of the exact stationary points. We emphasize that this ``neighborhood" \textit{cannot} be eliminated since there is a counter-example showing that \textbf{Adam cannot reach 0 gradient when $D_0 > 0$} (see Figure \ref{fig: D0}). So this is an intrinsic property of Adam, rather than the limitation of the theory. 
    Nevertheless, we believe the effect of ``not converging to exact stationary points" is  minor in practice. This is because: 1) As shown in Theorem \ref{thm:rate} and Figure \ref{fig: D0}, the size of the ``ambiguity zone" is inversely proportional to $\beta_2$. Since $\beta_2$ is often chosen to be close to 1, the ambiguity zone shrinks and becomes negligible.  2) Machine learning tasks do not pursue high-precision solutions (as much as other fields like PDE). Practitioners usually aim to efficiently find approximate solutions, rather than exact solutions that over-fit the training data.

To our knowledge, the discussion above is the first time that Adam and SGD are rigorously compared in the same setting where the advantage of Adam can be revealed. 
We believe these results shed new light on understanding the benefit of  Adam.

Finally, we briefly explain why the upper bound of Adam is independent of $M$. Intuitively, this is because: (1) it uses different learning rates for different components of $\bw$. (2) For each component  of $\bw$, the effective learning rate adjusts according to the gradient norm (thus according to the local smoothness). Even though the initial effective learning rate is small, it gets larger when moving in a flat landscape. Combining together, the initial learning rate of Adam can be independent of $M$, and so is its convergence rate.

\section{Discussion}  
  \label{sec: discuss}

\textbf{Adam's Advantage over Gradient Descent with Gradient Clipping.}  
\citet{zhang2019gradient} established that gradient descent (GD) and stochastic gradient descent (SGD) with gradient clipping are convergent under the $(L_0,L_1)$ smooth condition. A natural inquiry arises concerning the benefits of Adam over GD/SGD when gradient clipping is employed. While we lack robust theoretical backing to fully answer this question, one discernible advantage of Adam, as inferred from our results, is its capability to manage more intricate noise profiles that adhere to the affine variance noise assumption. In contrast, the current analyses of GD/SGD with gradient clipping within the $(L_0,L_1)$-smooth framework presuppose that the deviation between the stochastic gradient and the true gradient is uniformly bounded with certainty—an assumption more stringent than the one we consider. Indeed, a recent work \cite{faw2022power} demonstrates that there exists a counterexample satisfying Assumption \ref{assum:GC} over which SGD with gradient clipping fails to converge. Together with Theorem \ref{thm:rate}, their result demonstrate that a wide range of application scenario of Adam than SGD with gradient clipping.
  
\textbf{Insights for Practitioners.}  Here we discuss the insights our Theorem \ref{thm:rate} can provide to practitioners.
Firstly, the widespread adoption of Adam among practitioners, evidenced by its extensive citation record, underscores the importance of a theoretical understanding of the algorithm.  
  
Secondly, our findings offer theoretical support for a prevalent practice among practitioners: for tasks involving architectures like Transformers and LSTMs, Adam is often favored over SGD.  
  
Lastly, in light of the convergence criteria delineated in Theorem \ref{thm:rate}, we propose practical guidance for hyperparameter selection when employing Adam. Specifically, we recommend increasing $\beta_2$ and experimenting with various $\beta_1$ values that satisfy $\beta_1 < \sqrt{\beta_2}$. This heuristic could potentially reduce the computational burden associated with exhaustive hyperparameter exploration across the $(\beta_1,\beta_2)$ space.  

\textbf{Limitations of Theorem \ref{thm:rate}.} While Theorem \ref{thm:rate} represents a significant advancement in establishing the convergence of RR Adam under non-uniform smoothness conditions, it is not without its limitations. Specifically, Theorem \ref{thm:rate} applies to cases where $\bone = 0$, implying the absence of momentum. The theorem does not distinguish between the convergence rates when $\bone = 0$ and when $\bone > 0$, thus not demonstrating the benefits of momentum. The theoretical elucidation of momentum's advantage within Adam's convergence analysis remains a complex question. The role of momentum is not fully understood even in momentum SGD for non-convex optimization \cite{liu2020improved}, much less so for Adam. We acknowledge the importance of this question but consider it beyond the scope of this paper, leaving it for future research.

\section{Conclusions and  Future directions}
\label{sec: conclusion}

In this paper, we have taken a pioneering step towards a theoretical understanding of the adaptivity inherent in the Adam optimization algorithm. We present the first convergence results for RR Adam under the ($L_0$, $L_1$)-smooth condition, which is both realistic and closely aligned with practical scenarios. In contrast to existing analyses of RR Adam under the stronger $L$-smooth condition, our results further demonstrate a more robust form of convergence, specifically trajectory-wise convergence, and indicate a reduced distance to the stationary point.

\textbf{Future Directions.} An intriguing avenue for future research lies in delineating the advantages of incorporating momentum in Adam. Our Theorem \ref{thm:rate} indicates an identical convergence rate for both $\bone=0$ (RMSProp) and $\bone>0$ (Adam), implying that the current analysis does not differentiate between the iteration complexities of Adam and RMSProp. Consequently, the specific benefits of \textit{momentum} in Adam remain elusive. This presents a substantial challenge, given that the impact of momentum is not yet fully understood even in the context of SGD with momentum. A possible strategy could be to first establish a theoretical foundation for the advantages of momentum in SGD, followed by extending these insights to the analysis of Adam. Moreover, it would be compelling to explore whether Adam can effectively manage more severe smoothness conditions, such as those bounded by a higher-order polynomial of the gradient norm.


\paragraph{Supplementary materials} Supplementary materials including proofs can be found at \url{https://arxiv.org/abs/2208.09900}.

\paragraph{Acknowledgement}
This work is founded by the Strategic Priority Research Program of the Chinese Academy of Sciences under Grant No. XDB0680101, CAS Project for Young Scientists in Basic Research under Grant No. YSBR-034, Innovation Project of ICT CAS under Grants No. E261090, NSFC under Grant No. 12326608, Hetao Shenzhen-Hong Kong Science and Technology Innovation Cooperation Zone Project under Grant No.HZQSWS-KCCYB-2024016.

\bibliographystyle{ACM-Reference-Format}
\balance
\bibliography{related}

\appendix

\appendix
\onecolumn

\section{Proof of Theorem \ref{thm: sgd_diminishing}}
\label{appen: diverging}
In this section, we prove  Theorem \ref{thm: sgd_diminishing}. We consider the following function with variable $\bw = (x,y) \in \mathbb{R}^2$: $f(\bw)=f((x,y))=f_1(x)+ f_2(y)$, where

\begin{equation}
		f_1(x)=\left\{
		\begin{aligned}
		&\frac{L_0\exp^{L_1x -1 }}{L_1^2}&,  x\in [\frac{1}{L_1},\infty), \\
		&\frac{L_0x^2}{2} + \frac{L_0}{2L_1^2}&,  x\in [-\frac{1}{L_1},\frac{1}{L_1}], \\
		  &\frac{L_0\exp^{-L_1 x-1}}{L_1^2}&,  x\in (-\infty,-\frac{1}{L_1}].
		\end{aligned}
		\right.
   \label{lowerbound_f1}
\end{equation}
\begin{equation}
		f_2(y)=\left\{
		\begin{aligned}
		&\varepsilon(y-1)+\frac{\varepsilon}{2}&,  y\in [1,\infty), \\
		&\frac{\varepsilon }{2} y^2&,  y\in [-1,1], \\
		  &-\varepsilon(y+1)+\frac{\varepsilon}{2}&,  y\in (-\infty,-1].
		\end{aligned}
		\right.
  \label{lowerbound_f2}
\end{equation}

The construction of both functions  \eqref{lowerbound_f1} and \eqref{lowerbound_f2} are motivated by \citet{zhang2019gradient}. One improvement here is that we introduce a single function  with variable $\bw \in \mathbb{R}^2$ $f(\bw)=f((x,y))=f_1(x)+ f_2(y)$, which  helps us to derive a stronger conclusion, i.e., the constructed $f$ is independent of $\eta_1$. 
It is easy to see that this $f(\bw)$ satisfies $(L_0,L_1)$ condition with $L_0 = L_0$ and $L_1 = L_1$. We now restate Theorem \ref{thm: sgd_diminishing} as follows with constants specified.

\begin{theorem}[Theorem \ref{thm: sgd_diminishing}, restated]
\label{theorem_decay_lr_lower_bound_appendix}
    Consider function  $f(\bw)=f((x,y))=f_1(x)+ f_2(y)$ with  $f_1(x)$ and $f_2(y)$ defined in \eqref{lowerbound_f1} and \eqref{lowerbound_f2}. Consider gradient descent with diminishing learning rates: $\bw_{k+1} =\bw_{k} - \eta_k \nabla f(\bw_{k})  $, where $ \eta_k = \frac{\eta_1}{\sqrt{k}} $. Then for any $T, M, \bar{f}>0$, denote $\varepsilon = \sqrt{\frac{\frac{L_1M}{2}+\frac{L_0}{4}}{2(1+\sqrt{2}) (\log (\frac{L_1M}{2L_0}+\frac{1}{4})+1)}\frac{\bar{f}}{4\sqrt{T}}}$. As long as $M>\{\frac{2(e^{\frac{\log 2}{\sqrt{2}-1}-1}-\frac{1}{4})L_0}{L_1 },\varepsilon\}$ and $\frac{\bar{f}}{\varepsilon} > 6$, there exists an initialization $\bw_0= (x_0,y_0)$ such that, $M = \sup\{ \|\nabla f(\bw) \| :  \bw \text{ such that } f(\bw) \leq f(\bw_0)\} $ and $f(\bw_0)-f(\bw) =\bar{f}$, and for any $\eta_1>0$,  $\|\nabla f(\bw_k)\|\geq \epsilon$ whenever $k < T$.
   
    
\end{theorem}

\textbf{One can immediately see from the above theorem that $\varepsilon^2 = \tilde{\Theta}(\frac{M (f(\bw_0)-f^*)}{\sqrt{T}})$, which gives Theorem \ref{thm: sgd_diminishing}}. Before giving the proof of Theorem \ref{theorem_decay_lr_lower_bound_appendix}, we briefly discuss the difference between ours and   \citep[Theorem 4]{zhang2019gradient}.  Generally speaking, our result is stronger than \citep[Theorem 4]{zhang2019gradient}. This is because we pick the function before the learning rate: we prove that there exists a function $f$ and an initialization, such that with any learning rate GD takes a long time to reach the stationary point, while \citep[Theorem 4]{zhang2019gradient} picks the learning rate before the function: they prove that with any learning rate, there exists a function $f$ and an initialization, such that  GD takes a long time to reach the stationary point. 

We next present the proof of Theorem \ref{theorem_decay_lr_lower_bound_appendix} in {\bf Part I} and {\bf Part II} as follows. For simplicity, we let $\Vert \cdot \Vert$ to be the $\ell_{\infty}$ norm, and the proof can be easily extended to other norms given the equivalence between norms in $\mathbb{R}^2$. We pick $x_0= \frac{\log (\frac{L_1M}{2L_0}+\frac{1}{4})+1}{L_1}$ and $y_0 = \frac{f_1(x_0)-\frac{L_0}{2L_1^2}}{\varepsilon}-\frac{1}{2}$. We have $f_1(x_0)-\min_x f_1(x)=f_2(y_0)-\min_y f_2(y)$, and thus $f((x_0,y_0))-\min_{x,y} f((x,y))=2\left(f_1(x_0)-\min_x f_1(x)\right)$. As $M>\varepsilon$, $\sup\{ \|\nabla f(\bw) \| :  \bw \text{ such that } f(\bw) \leq f(\bw_0)\}$ is achieved at $(x_0',0)$ where $x_0'$ satisfies $f_1(x_0')-\min_x f_1(x) =2\left(f_1(x_0)-\min_x f_1(x)\right)$. By simple calculation, we have $\sup\{ \|\nabla f(\bw) \| :  \bw \text{ such that } f(\bw) \leq f(\bw_0)\}=M$.

In the proof, we use $x_k$ to denote the value of $x$ (i.e., the first component of $\bw \in \mathbb{R}$) in the $k$-th iteration of gradient descent. Similarly for $y_k$.

\paragraph{Part I: Large $\eta_1$ can cause divergence.} In this part, we prove that: when using the large initial learning rate $\eta_1 \geq   \frac{L_1 (1+\sqrt{2}) |x_{0}|}{L_0 \exp^{L_1 |x_{0}|-1}}$, decay-learning-rate gradient descent  will never reach stationary points.

We prove this claim by induction. When $k =1$, we claim the following two statements are true:

(1-I): $|x_{1}| \geq \sqrt{2}|x_{0}|$.

(1-II): $\eta_2= \frac{\eta_1}{\sqrt{2}} \geq \frac{L_1(1+\sqrt{2})|x_1|}{L_0 \exp^{L_1 |x_1|-1}}$.


We first prove (1-I): without loss of generality, we assume $x_1>0$.  By the update rule of gradient descent, we have

\begin{eqnarray*}
    x_1 & = & x_0 - \eta_1 \frac{\partial f(x_0 ) }{\partial x }\\
    &\overset{\eqref{lowerbound_f1}}{=}& x_0 - \eta_1 \frac{L_0 \exp^{L_1x_0-1}}{L_1} \\
    &\leq & x_0 - \frac{L_1 (1+\sqrt{2}) |x_0|}{L_0 \exp^{L_1 |x_0|-1}} \frac{L_0 \exp^{L_1x_0-1}}{L_1} = -\sqrt{2}x_0. 
\end{eqnarray*}

So  $|x_1| \geq \sqrt{2}|x_0|$ and (1-1) is proved. 
We now prove (1-II). Before that, we introduce the following lemma.

\begin{lemma}
\label{lemma_lower_bound_decay_lr}
    Consider any $x,y \in  \{ z: |z| \geq \frac{\log 2 }{(\sqrt{2}-1)L_1}, z \in \mathbb{R}\}$. When $|y|> \sqrt{2}|x|$, then we have $\frac{|y|}{\exp^{L_1 |y|}} \leq \frac{1}{\sqrt{2}} \frac{|x|}{\exp^{L_1 |x|}}$.
\end{lemma}

\begin{proof}
    Let $g(z)  = \frac{z}{\exp^{L_1 z}}$. It is easy to see that $\nabla g(z) <0$ when $z>0$. Therefore, when $z_1 \geq \sqrt{2}z_2$, we have $g(z_1) \leq g(\sqrt{2}z_2)$. When $z_2 \geq  \frac{\log 2 }{(\sqrt{2}-1)L_1}$, we have
\begin{eqnarray*}
    &&z_2 \geq  \frac{\log 2 }{(\sqrt{2}-1)L_1}\\
     &\Leftrightarrow& \sqrt{2}L_1 z_2 > \log 2 + L_1 z_2 \\
    &\Leftrightarrow& \exp^{\sqrt{2}L_1z_2} \geq 2\exp^{L_1 z_2} \\
   &\Leftrightarrow& \frac{1}{\exp^{\sqrt{2}L_1z_2}} \geq \frac{1}{2 \exp^{L_1 z_2}} \\
    &\Leftrightarrow& \frac{\sqrt{2}z_2}{\exp^{\sqrt{2}L_1z_2}} \geq \frac{z_2}{\sqrt{2} \exp^{L_1 z_2}}. \\
\end{eqnarray*}
    
Therefore, we have $\frac{z_1}{\exp^{z_1 L_1}} \leq \frac{\sqrt{2}z_2}{\exp^{\sqrt{2}L_1z_2}} \geq \frac{z_2}{\sqrt{2} \exp^{L_1 z_2}}. $ Proof of Lemma \ref{lemma_lower_bound_decay_lr} is completed.

\end{proof}

Now we prove (1-II): 



\begin{eqnarray*}
    \eta_2 =   \frac{\eta_1}{\sqrt{2}}  &\geq& \frac{L_1 (1+\sqrt{2})|x_0|}{L_0 \exp^{L_1 |x_0| -1}} \frac{1}{\sqrt{2}} \\
    &\overset{\text{(1-I) and Lemma \ref{lemma_lower_bound_decay_lr}}}{\geq}& \frac{L_1 (1+\sqrt{2})|x_1|}{L_0 \exp^{L_1 |x_1| -1}} \sqrt{2} \frac{1}{\sqrt{2}}  \\
    & = & \frac{L_1 (1+\sqrt{2})|x_1|}{L_0 \exp^{L_1 |x_1| -1}} 
\end{eqnarray*}

So (1-II) is proved.  Now we suppose the following two claims hold for $k = 2m$ where $m \in \mathbb{N}^{+}$.

(2m-I): $|x_{2m+1}| \geq \sqrt{2} |x_{2m}|$. 

(2m-II): $\eta_{2m+2}= \frac{\eta_1}{\sqrt{2m+2}} \geq \frac{L_1(1+\sqrt{2})|x_{2m+1}|}{L_0 \exp^{L_1 |x_{2m+1}|-1}}$

Then for  $k = 2m+1$, we prove the following claims hold for $k = 2m+1$.

((2m+1)-I): $|x_{2m+2}| \geq \sqrt{2} |x_{2m+1}|$. 

((2m+1)-II): $\eta_{2m+3}= \frac{\eta_1}{\sqrt{2m+3}} \geq \frac{L_1(1+\sqrt{2})|x_{2m+2}|}{L_0 \exp^{L_1 |x_{2m+2}|-1}}$.

We first prove ((2m+1)-I): 

\begin{eqnarray*}
    x_{2m+2}  & = & x_{2m+1} - \eta_{2m+2} \frac{\partial f(x_{2m+1} ) }{\partial x }\\
    &\overset{\eqref{lowerbound_f1}}{=}& x_{2m+1} - \eta_{2m+2} \frac{L_0}{L_1} \exp^{L_1 x_{2m+1} -1 } \\
    &\overset{\text{(2m-II)}}{\leq} & x_{2m+1} - \frac{L_1 (1+\sqrt{2}) |x_{2m+1}|}{L_0\exp^{L_1 |x_{2m+1}|-1}} \frac{L_0 \exp^{L_1 x_{2m+1} -1}}{L_1} \\
    &\leq & -\sqrt{2} x_{2m+1}.
\end{eqnarray*}

So $|x_{2m+2}| \geq \sqrt{2} |x_{2m+1}|$ and ((2m+1)-I) is proved. We now prove ((2m+1)-II).

\begin{eqnarray*}
    \eta_{2m+3} &=& \eta_{2m+2} \sqrt{\frac{2m+2}{2m+3}} \\
    &\overset{\text{(2m-II)}}{\geq}&  \frac{L_1(1+\sqrt{2})|x_{2m+1}|}{L_0 \exp^{L_1 |x_{2m+1}|-1}}  \sqrt{\frac{2m+2}{2m+3}} \\
    &\overset{\text{(2m-I) and Lemma \ref{lemma_lower_bound_decay_lr}}}{\geq}& \frac{L_1(1+\sqrt{2})|x_{2m+2}|}{L_0 \exp^{L_1 |x_{2m+2}|-1}}   \sqrt{2} \sqrt{\frac{2m+2}{2m+3}} \\
    &\geq& \frac{L_1(1+\sqrt{2})|x_{2m+2}|}{L_0 \exp^{L_1 |x_{2m+2}|-1}} .
\end{eqnarray*}

So ((2m+1)-II) is proved. We can derive a similar claim when $k$ is odd. By the principle of induction, we know that $|x_{k+1}| \geq \sqrt{2} |x_{k}|$ for any $k \geq 1$. Since $f_1(x)$ grows exponentially, gradient descent in {\bf Part I} will never reach stationary points.

\paragraph{Part II: Small $\eta_1$ can cause slow convergence.}  In this part, we prove that:  when using initialization $\bw \in \Omega$, decay-learning-rate gradient descent with {\it small} initial learning rate $\eta_1 <   \frac{L_1 (1+\sqrt{2}) |x_{0}|}{L_0 \exp^{L_1 |x_{0}|-1}}=\frac{(1+\sqrt{2}) (\log (\frac{L_1M}{2L_0}+\frac{1}{4})+1)}{\frac{L_1M}{2}+\frac{L_0}{4}}$ will cause slow convergence. For any $k \geq 1$, we have

\begin{eqnarray*}
   y_{k} - y_{k+1} &= &   \eta_k \frac{\partial f(\bw_k)}{\partial y} \\
    & \overset{\eqref{lowerbound_f2}}{=} &   \varepsilon \frac{\eta_1}{\sqrt{k}} \\
    &< &  \varepsilon \frac{L_1 (1+\sqrt{2}) |x_0|}{L_0 \exp^{L_1 |x_0|-1}} \frac{1}{\sqrt{k}}
\end{eqnarray*}

Therefore, we have 

\begin{eqnarray*}
    \sum_{k =1}^K (y_{k+1} - y_{k}) = \sum_{k =1}^K  \varepsilon  \frac{\eta_1}{\sqrt{k}}  < 2\sqrt{k} \varepsilon \eta_1 < 2\sqrt{k} \varepsilon \frac{L_1 (1+\sqrt{2}) |x_0|}{L_0 \exp^{L_1 |x_0|-1}}
\end{eqnarray*}

When using initialization $y_0$, it is easy to have the following conclusion:  when $2\frac{(1+\sqrt{2}) (\log (\frac{L_1M}{2L_0}+\frac{1}{4})+1)}{\frac{L_1M}{2}+\frac{L_0}{4}}\sqrt{k}=2 \varepsilon \sqrt{k} \frac{L_1 (1+\sqrt{2}) |x_1|}{L_0 \exp^{L_1 |x_1|-1}} < y_{0} -1 = \frac{\left(f_1(x_0) - \min_{x}f_1(x)\right)}{\varepsilon}-\frac{3}{2}$, we have $\frac{\partial f(\bw_k)}{\partial y} = \varepsilon$. In other words, we have: $\|\nabla f(\bw_k)\| \geq \varepsilon$ for all $k < (\frac{\frac{L_1M}{2}+\frac{L_0}{4}}{2(1+\sqrt{2}) (\log (\frac{L_1M}{2L_0}+\frac{1}{4})+1)})^2\frac{(\frac{f_2(y_0) - \min_{y}f_2(y)}{\varepsilon}-\frac{3}{2})^2}{\varepsilon^2}   $.  Recall that $f(\bw_1)-\min_{\bw} f(\bw)=2(f_2(y_0)-\min_x f_2(x))$ and $(\frac{\frac{L_1M}{2}+\frac{L_0}{4}}{2(1+\sqrt{2}) (\log (\frac{L_1M}{2L_0}+\frac{1}{4})+1)})^2\frac{(\frac{f_2(y_0) - \min_{y}f_2(y)}{\varepsilon}-\frac{3}{2})^2}{\varepsilon^2} < T$ by the definition of $\varepsilon$, the proof is completed.

\section{Proof of Theorem \ref{thm:rate}}
This appendix provides the formal proof of Theorem \ref{thm:rate}, which is organized as follows. In Section \ref{appen: notations}, we  first introduce notations that are used in the proof.  In Section \ref{appen: restate_thm}, We  restate Theorem \ref{thm:rate} with constants specified. In Section \ref{appen: aux_lemma}, we then make preparations by proving auxiliary lemmas. Finally, in Section \ref{appen: convergence}, we prove Theorem \ref{thm:rate}. 

\subsection{Notations}
\label{appen: notations}

Here we provide a complete list of notations used in the appendix for a clear reference.
\begin{itemize}

\item We use $(k_1,i_1)\le(<) (k_2,i_2)$ for $\forall k_1,k_2\in \mathbb{N}^{+}$ and $i_1,i_2\in\{0,\cdots,n-1\}$, if either $k_1<k_2$ or $k_1=k_2$ and $i_1\le(<) i_2$

\item We define function $g(x):[0,1] \rightarrow \mathbb{R}^{/-}$ as
\begin{small}
\begin{align*}
    g(\btwo) \triangleq \max\left\{\frac{1}{\sqrt{\btwo^{n-1}}}-1,1-\frac{1}{\sqrt{\btwo^{n-1}+8n\frac{1-\btwo^{n-1}}{\btwo^n}}},1-\sqrt{\btwo},\sqrt{\frac{\btwo}{\left(1-(1-\btwo)\frac{2n}{\btwo^n}\right)}}-1\right\}.
\end{align*}
\end{small}
\item We define constants $\{C_i\}_{i=1}^{10}$ as follows:
\begin{scriptsize}
\begin{equation}
\label{eq: define_constants}
\begin{aligned}
     C_1&\triangleq\frac{(1-\bone)^2}{1-\btwo}\frac{1}{1-\frac{\bone^2}{\btwo}}+1,
     \\
     C_2 &\triangleq nC_1+\frac{\bone}{1-\bone}C_1\left(1+\sqrt{2}\right),
     \\
     C_3&\triangleq C_1\left(n(L_0+L_1\sqrt{D_0})+2\sqrt{2}  (L_0+L_1\sqrt{D_0})\frac{\sqrt{1-\btwo}}{1-\sqrt{\btwo}}\frac{\sqrt{\btwo}}{1-\sqrt{\btwo}} + 8\sqrt{2n} L_0 \frac{1}{1-\btwo^n}\right),
     \\
     C_4&\triangleq 4L_1C_1\sqrt{D_1}  \frac{\sqrt{1-\btwo}}{1-\sqrt{\btwo}}
     \\
     C_5&\triangleq n^2(1+n
        \sqrt{d} C_1\eta_1L_1\sqrt{n}\sqrt{D_1})\left(C_4+\frac{dC_4\sqrt{D_1}}{1-
     \sqrt{\btwo^n}}\right),
     \\
     C_6 &\triangleq  \left(dC_3+\frac{C_4n \sqrt{D_1}}{1-
     \sqrt{\btwo^n}}\right)\eta^2_1,
     \\
     C_7 &\triangleq 3n\left(C_4+\frac{dC_4}{1-
     \sqrt{\btwo^n}}\right)\left(nL_0+L_1\sqrt{n}\sqrt{D_0}\right)n^2
        \sqrt{d} C_1 \eta_1^3+ \left(dC_3+\frac{C_2C_4n \sqrt{D_1}}{1-
     \sqrt{\btwo^n}}\right)\eta^2_1,
     \\
     C_8&\triangleq \sqrt{\frac{2n^2}{\btwo^n}}L_1\sqrt{D_1}n\sqrt{n}+dg(\btwo)\left(n-1+\frac{1+\bone}{1-\bone}\right)\frac{\sqrt{2}n}{\btwo^{\frac{n}{2}}} L_1C_1\sqrt{D_1}\left(1+\frac{1}{1-\btwo^n}\right)(n+n^{\frac{5}{2}}
        \sqrt{d} C_1\eta_1L_1\sqrt{D_1})
    +2\frac{\beta_1}{(1-\beta_1)\eta_1}\sqrt{d}C_1,
    \\
    C_9&\triangleq\sqrt{\frac{2n^2}{\btwo^n}}d(n^2L_0+n\sqrt{n}L_1\sqrt{D_0}) C_1\eta_1^2+g(\btwo)\left(n-1+\frac{1+\bone}{1-\bone}\right)\frac{\sqrt{2}n}{\btwo^{\frac{n}{2}}} \left(n+\frac{2\sqrt{2}\bone}{1-\bone}\right)C_1(L_0+L_1\sqrt{D_0})d\sqrt{d}\eta^2_1,
    \\
   C_{10}&\triangleq 3dg(\btwo)\left(n-1+\frac{1+\bone}{1-\bone}\right)\frac{\sqrt{2}n}{\btwo^{\frac{n}{2}}} L_1C_1\sqrt{D_1}\left(1+\frac{1}{1-\btwo^n}\right)n\left(nL_0+L_1\sqrt{n}\sqrt{D_0}\right)n
        \sqrt{d} C_1\eta_1^3+C_9,
        \\
    C_{11}&\triangleq (\frac{1}{2}+C_2)C_5+C_8+\frac{3L_1\sqrt{n}\sqrt{D_1}C_2^2d}{2},
    \\
    C_{12}&\triangleq (\frac{1}{2}+C_2)C_6+C_9+\frac{nL_0+L_1\sqrt{n}\sqrt{D_0}}{2}3C_2^2d\eta_1^2,
    \\
    C_{13}&\triangleq (\frac{1}{2}+C_2)C_7+C_{10}+\frac{nL_0+L_1\sqrt{n}\sqrt{D_0}}{2}3C_2^2d\eta_1^2.
\end{aligned}
\end{equation}
\end{scriptsize}
\end{itemize}

 \subsection{Restate Theorem \ref{thm:rate}}
 \label{appen: restate_thm}
Here we restate Theorem \ref{thm:rate} with constants specified.

\begin{theorem}[Theorem \ref{thm:rate}, restated]
Consider Adam defined as Alg. (\ref{alg:def_adam}) with diminishing learning rate $\eta_k\equiv\frac{\eta_1}{\sqrt{k}} $. Let Assumptions \ref{assum:regular} and \ref{assum:GC} hold. Suppose the hyperparamters satisfy: $ \gamma <\btwo<1$ and $0\le \bone^2<\btwo$, where $\gamma$ is defined as the solution of $\sqrt{d}g(x)\frac{n}{x^{\frac{n}{2}}}=\frac{1}{2(4+\sqrt{2})\sqrt{D_1}\left(n-1+\frac{1+\bone}{1-\bone}\right)}$ with respect to $x$. Then, either
\begin{equation*}
\min_{k\in [1,T]}\Vert \nabla f(\bw_{k,0})\Vert
    \le
  2 \sqrt{d}(2\sqrt{2}+1)\sqrt{D_0} g(\btwo)\left(n-1+\frac{1+\bone}{1-\bone}\right) \sqrt{\frac{2n}{\btwo^n}},
\end{equation*}
or
\small
\begin{align*}
    &\min_{k\in [1,T]}\left\{\frac{\Vert\nabla f(\bw_{k,0})\Vert }{\sqrt{D_1 } },\frac{\Vert\nabla f(\bw_{k,0})\Vert^2 }{\sqrt{D_0 } + \xi}\right\}\le (4(2\sqrt{2}+1))\frac{f(\bw_{1,0})-\min_{\bw}f(\bw)}{
\eta_1\sqrt{T}}
   \\
    &~~~
   +4(2\sqrt{2}+1)\left(C_{12}+\frac{\sqrt{D_0}
   +\xi}{4\sqrt{D_1}}C_{11}\eta_1^2\right)\frac{\ln T}{
\eta_1\sqrt{T}} + 4(2\sqrt{2}+1)\frac{C_{13}+\frac{\sqrt{D_0}
   +\xi}{4\sqrt{D_1}}C_{11}}{\eta_1\sqrt{T}}.
\end{align*}
\normalsize
    
\end{theorem}

\subsection{Auxiliary Lemmas}
\label{appen: aux_lemma}

In this section, we introduce auxiliary lemmas that will be latter used.  In the remaining proof of this paper, we assume \textbf{without the loss of generality that $\eta_1$ is small enough}, such that the following requirements are fulfilled: ($C_1$ and $C_2$ are defined in Eq. (\ref{eq: define_constants})).
\begin{itemize}
    \item $2C_2\sqrt{d} \eta_1\le \frac{1}{L_1}$. This will latter ensure that we can directly apply the definition of $(L_0, L_1)$-smooth condition (Assumption \ref{assum:regular}) to parameter sequence $\{\bw_{k,i}\}_{k,i}$;
    
    \item  $\frac{1}{4(2\sqrt{2}+1)}\ge \sqrt{D_1}C_{11}\eta_1$. This will latter ensure the second-order term is smaller than the first-order term at the end of the proof.
\end{itemize}

The proof can be easily extended to general cases by selecting large enough $K$ and using the epoch $K$ as a new start point and derive the results after epoch $K$, because the epochs before epoch $K$ can be uniformly bounded due to $\eta_k$ decaying  and   $K$ finite, and we then derive the desired result for all epochs.

Without the loss of generality, we also take the following initialization: $\bw_{1,0}=\bw_0$, $\bom_{1,-1}=\nabla f_{\tau_{1,-1}}(\bw_0)$  where $\tau_{1,-1}$ can be any integer in $[0,n-1]$, and  $\bnu_{l,1,-1}=\max_j\{\partial_l f_j(\bw_0)^2\} ~\forall l$ where the maximum is taken component-wisely. We take the initialization to have a more concise proof, while the proof can be easily extended to all the initialization as the information of the initialization in the exponentially decayed average of Adam (both in $\bom_{k,i}$ and $\bnu_{k,i}$) decays rapidly with $k$ increasing.

The following lemma shows that $f$ is also $(L_0,L_1)$-smooth under Assumptions \ref{assum:regular} and \ref{assum:GC} (while the $L_0$ and $L_1$ are different from those of $f_i$).
\begin{lemma}
\label{lem: f_smooth}
With Assumptions \ref{assum:regular} and \ref{assum:GC}, $f$ satisfies $(nL_0+L_1\sqrt{n} \sqrt{D_0},L_1\sqrt{n}\sqrt{D_1})$-smooth condition.
\end{lemma}

\begin{proof}
$\forall \bw_1,\bw_2\in \mathbb{R}^d$ satisfying $\Vert \bw_1-\bw_2\Vert \le \frac{1}{L_1}$,
 \begin{small}
\begin{align*}
   & \Vert \nabla f(\bw_1)-\nabla f(\bw_2)\Vert\le  \sum_{i=0}^{n-1}\Vert \nabla f_i(\bw_1)-\nabla f_i(\bw_2)\Vert\le \sum_{i=0}^{n-1} (L_0+L_1 \Vert \nabla f_i(\bw_1)\Vert )\Vert \bw_1-\bw_2\Vert
   \\
   \le & \left(nL_0+L_1\sqrt{n} \sqrt{\sum_{i=0}^{n-1}\Vert \nabla f_i(\bw_1)\Vert^2} \right)\Vert \bw_1-\bw_2\Vert
   \le (nL_0+L_1\sqrt{n} \sqrt{D_0+D_1\Vert \nabla f(\bw_1)\Vert^2} )\Vert \bw_1-\bw_2\Vert
   \\
   \le& (nL_0+L_1\sqrt{n} (\sqrt{D_0}+\sqrt{D_1}\Vert \nabla f(\bw_1)\Vert) )\Vert \bw_1-\bw_2\Vert
   \le (nL_0+L_1\sqrt{n} \sqrt{D_0}+L_1\sqrt{n}\sqrt{D_1}\Vert \nabla f(\bw_1)\Vert) \Vert \bw_1-\bw_2\Vert.
\end{align*}
\end{small}

The proof is completed.
\end{proof}

The following lemma bounds the update norm of Adam.

\begin{lemma}[Bounded Update] 
\label{lem: bounded_update}
If $\bone< \sqrt{\btwo}$, we have $\forall k\in \mathbb{N}^{+}$, $i\in\{0,\cdots,n-1\}$, 
\begin{equation*}
    \frac{\vert\bom_{l,k,i}\vert}{\sqrt{\bnu_{l,k,i}}+\xi}\le C_1,
\end{equation*}
where $C_1$ is defined in Eq. (\ref{eq: define_constants}).

Furthermore, we have
    $\vert \bw_{l,k,i+1}-\bw_{l,k,i} \vert \le C_1\eta_k$, and thus $\Vert \bw_{k,i+1}-\bw_{k,i} \Vert \le C_1\eta_k\sqrt{d}$.
\end{lemma}
\begin{proof}
By the definition of $\bom_{k,i}$, we have
\begin{align*}
    &(\bom_{l,k,i})^2
    \\
    =&\left((1-\bone)\sum_{j=0}^i \bone^{(k-1)n+i-((k-1)n+j)}\partial_{l}f_{\tau_{k,j}}(\bw_{k,j})\right.
    \\
    &+\left. (1-\bone)\sum_{m=1}^{k-1}\sum_{j=0}^{n-1} \bone^{(k-1)n+i-((m-1)n+j)}\partial_{l}f_{\tau_{m,j}}(\bw_{m,j})+\bone^{(k-1)n+i+1}\partial_{l}f_{\tau_{1,-1}}(\bw_{1,0})\right)^2
    \\
    \le &\left((1-\bone)\sum_{j=0}^i \bone^{(k-1)n+i-((k-1)n+j)}\vert\partial_{l}f_{\tau_{k,j}}(\bw_{k,j})\vert\right.
    \\
    &+\left. (1-\bone)\sum_{m=1}^{k-1}\sum_{j=0}^{n-1} \bone^{(k-1)n+i-((m-1)n+j)}\vert\partial_{l}f_{\tau_{m,j}}(\bw_{m,j})\vert+\bone^{(k-1)n+i+1}\max_{s\in[n]}\vert\partial_{l}f_s(\bw_{1,0})\vert\right)^2
    \end{align*}
    \begin{align*}
    \overset{(\star)}{\le} &\left((1-\btwo)\sum_{j=0}^i \btwo^{(k-1)n+i-((k-1)n+j)}\vert\partial_{l}f_{\tau_{k,j}}(\bw_{k,j})\vert^2\right.
    \\
    &+\left. (1-\btwo)\sum_{m=1}^{k-1}\sum_{j=0}^{n-1} \btwo^{(k-1)n+i-((m-1)n+j)}\vert\partial_{l}f_{\tau_{m,j}}(\bw_{m,j})\vert^2+\btwo^{(k-1)n+i+1}\max_{s\in[n]}\vert\partial_{l}f_s(\bw_{1,0})\vert^2\right)
    \\
    &\cdot \left(\frac{(1-\bone)^2}{1-\btwo}\sum_{j=0}^{(k-1)n+i}\left(\frac{\bone^2}{\btwo}\right)^j+\left(\frac{\bone^2}{\btwo}\right)^{(k-1)n+i+1}\right)
    \\
    \overset{(\ast)}{=}&  \left(\frac{(1-\bone)^2}{1-\btwo}\sum_{j=0}^{(k-1)n+i}\left(\frac{\bone^2}{\btwo}\right)^j+\left(\frac{\bone^2}{\btwo}\right)^{(k-1)n+i+1}\right)\bnu_{l,k,i}
    \\
    \le &  \left(\frac{(1-\bone)^2}{1-\btwo}\frac{1}{1-\frac{\bone^2}{\btwo}}+1\right)\bnu_{l,k,i}= C_1\bnu_{l,k,i},
\end{align*}
where Eq. ($\star$) is due to the Cauchy-Schwartz's Inequality, and {\color{blue} Eq. ($\ast$) is due to the definition of $\bnu_{l,1,-1}$}. We complete the proof of the first claim. The second claim then follows directly from the update rule
\begin{equation*}
    \bw_{l,k,i+1}-\bw_{l,k,i}=\eta_k  \frac{\bom_{l,k,i}}{\sqrt{\bnu_{l,k,i}}+\xi}.
\end{equation*}

The proof is completed.
\end{proof}

Define $\bu_k\triangleq \frac{\bw_{k,0}-\bone\bw_{k,-1}}{1-\bone}$ (with $\bw_{1,-1}\triangleq \bw_{1,0}$), and let $\bu_{l,k}$ be the $i$-th component of $\bu_{k}$, $\forall k\in \mathbb{N}^{+}$, $l\in [d]$. The following lemma bounds the distance between $\bu_{l,k}$ and $\bw_{l,k,0}$ and the distance between $\bu_{l,k+1}$ and $\bu_{l,k}$.

\begin{lemma}
\label{lem: u_gap}
    $\forall k \ge 1$, 
    \begin{gather}
\label{eq: gap_u_w}
    \vert\bu_{l,k}-\bw_{l,k,0}\vert\le C_2\eta_k,
    \\
     \label{eq: gap_u_u}
    \vert\bu_{l,k+1}-\bu_{l,k}\vert
    \le C_2\eta_k,
\end{gather}
where $C_2$ is defined in Eq. (\ref{eq: define_constants}).
\end{lemma}
\begin{proof}
By Lemma \ref{lem: bounded_update}, we immediately have $\forall l\in [d]$, $ \vert\bu_{l,k}-\bw_{l,k,0}\vert$ is bounded as
\begin{align}
\nonumber
    &\vert\bu_{l,k}-\bw_{l,k,0}\vert=\left\vert \frac{\bw_{l,k,0}-\bone\bw_{l,k,-1}}{1-\bone}-\bw_{l,k,0}\right\vert
    \\
    \nonumber
    =&\frac{\bone}{1-\bone}\left\vert \bw_{l,k,0}-\bw_{l,k,-1}\right\vert\le\frac{\bone}{1-\bone} C_1\eta_1\frac{1}{\sqrt{k-1}}
    \le\frac{\sqrt{2}\bone}{1-\bone} C_1\eta_1\frac{1}{\sqrt{k}}\le \frac{\sqrt{2}\bone}{1-\bone} C_1\eta_k \le C_2\eta_k,
\end{align}
and
\begin{align}
\nonumber
    &\vert\bu_{l,k+1}-\bu_{l,k}\vert
    \\
\nonumber
    =&\left\vert \frac{\bw_{l,k+1,0}-\bone\bw_{l,k+1,-1}}{1-\bone}-\frac{\bw_{l,k,0}-\bone\bw_{l,k,-1}}{1-\bone}\right\vert
    \\
\nonumber
    =&\left\vert \left(\bw_{l,k+1,0}-\bw_{l,k,0}\right)+\frac{\bone}{1-\bone}\left(\bw_{l,k+1,0}-\bw_{l,k+1,-1}\right)-\frac{\bone}{1-\bone}\left(\bw_{l,k,0}-\bw_{l,k,-1}\right)\right\vert
    \\
\nonumber
    \le&\left\vert \left(\bw_{l,k+1,0}-\bw_{l,k,0}\right)+\frac{\bone}{1-\bone}\left(\bw_{l,k+1,0}-\bw_{l,k+1,-1}\right)-\frac{\bone}{1-\bone}\left(\bw_{l,k,0}-\bw_{l,k,-1}\right)\right\vert
    \\
\nonumber
    \le& nC_1\eta_1\frac{1}{\sqrt{k}}+\frac{\bone}{1-\bone}C_1\eta_1\left(\frac{1}{\sqrt{k}}+\frac{\sqrt{2}}{\sqrt{k}}\right)=C_2\eta_1\frac{1}{\sqrt{k}}=C_2\eta_k.
\end{align}
\end{proof}

In the following lemma, we bound the change of the gradient within one epoch.
\begin{lemma}
\label{lem: relationship_across_iteration}
$\forall k \in \mathbb{N}^+, i \in \{0,\cdots,n-1\}$,
\begin{equation*}
    \Vert \nabla f(\bw_{k,i}) \Vert \le  (1+n
        \sqrt{d} C_1\eta_1L_1\sqrt{n}\sqrt{D_1})\Vert \nabla f(\bw_{k,0}) \Vert+\left(nL_0+L_1\sqrt{n}\sqrt{D_0}\right)n
        \sqrt{d} C_1\eta_k,
\end{equation*}
where $C_1$ is defined in Eq. (\ref{eq: define_constants}).
\end{lemma}
\begin{proof}
    By Assumption \ref{assum:regular}  and Lemma \ref{lem: f_smooth}, we have
    \begin{align*}
    \Vert \nabla f(\bw_{k,i}) \Vert
        \le &\Vert \nabla f(\bw_{k,0}) \Vert+\left(nL_0+L_1\sqrt{n}\sqrt{D_0}+L_1\sqrt{n}\sqrt{D_1} \Vert \nabla f (\bw_{k,0})\Vert\right)\Vert \bw_{k,i}-\bw_{k,0} \Vert 
        \\
       \le &\Vert \nabla f(\bw_{k,0}) \Vert+\left(nL_0+L_1\sqrt{n}\sqrt{D_0}+L_1\sqrt{n}\sqrt{D_1} \Vert \nabla f (\bw_{k,0})\Vert\right) i
        \sqrt{d} C_1\eta_k
        \\
        \le & (1+n
        \sqrt{d} C_1\eta_1L_1\sqrt{n}\sqrt{D_1})\Vert \nabla f(\bw_{k,0}) \Vert+\left(nL_0+L_1\sqrt{n}\sqrt{D_0}\right)n
        \sqrt{d} C_1\eta_k.
    \end{align*}
    The proof is completed.
\end{proof}


We further need a descent lemma assuming $(L_0, L_1)$-smooth condition similar to the case assuming $L$ smoothness. Specifically, for a function $h$ satisfying $L$-smooth condition and two points $\bw$ and $\bv$, by Taylor's expansion, we have
\begin{equation*}
    h(\bw)\le h(\bv)+\langle \nabla h(\bv), \bw-\bv \rangle +\frac{L}{2} \Vert \bw-\bv\Vert^2.
\end{equation*}
This is called "Descent Lemma" by existing literature \cite{sra2014slides}, as it guarantees that the loss decreases with proper parameter update. Paralleling to the above inequality, we establish the following descent lemma under the $(L_0,L_1)$-smooth condition.

\begin{lemma}
\label{lem: descent}
Assume that function $h:\mathcal{X}\rightarrow \mathbb{R}$ satisfies the $(L_0, L_1)$-smooth condition, i.e., $\forall \bw,\bv\in \mathcal{X}$ satisfying $\Vert \bw-\bv\Vert \le \frac{1}{L_1}$,
\begin{equation*}
    \Vert \nabla h(\bw)-\nabla h(\bv)\Vert \le (L_0+L_1 \Vert \nabla h(\bv)\Vert )\Vert \bw-\bv\Vert.
\end{equation*}
Then, for any three points $\bu, \bw,\bv\in \mathcal{X}$  satisfying $\Vert \bw-\bu\Vert \le \frac{1}{L_1}$ and $\Vert \bv-\bu\Vert \le \frac{1}{L_1}$, we have
\begin{equation*}
    h(\bw)\le  h(\bv)+\langle \nabla h(\bu), \bw-\bv\rangle + \frac{1}{2}(L_0+L_1 \Vert \nabla h(\bu)\Vert)(\Vert \bv -\bu\Vert+  \Vert \bw-\bu\Vert) \Vert\bw-\bv\Vert. 
\end{equation*}
\end{lemma}
\begin{proof}
By the Fundamental Theorem of Calculus, we have
\begin{align*}
    h(\bw)=& h(\bv)+\int_{0}^1 \langle \nabla h(\bv+a(\bw-\bv)), \bw-\bv\rangle
    \mathrm{d}a
    \\
    =& h(\bv)+\langle \nabla h(\bu), \bw-\bv\rangle +\int_{0}^1 \langle \nabla h(\bv+a(\bw-\bv))-\nabla h(\bu), \bw-\bv\rangle
    \mathrm{d}a
    \\
    \le & h(\bv)+\langle \nabla h(\bu), \bw-\bv\rangle +\int_{0}^1 \Vert \nabla h(\bv+a(\bw-\bv))-\nabla h(\bu)\Vert \Vert\bw-\bv\Vert
    \mathrm{d}a
    \\
    \overset{(\star)}{\le} &  h(\bv)+\langle \nabla h(\bu), \bw-\bv\rangle +\int_{0}^1 (L_0+L_1 \Vert \nabla h(\bu)\Vert )\Vert\bv+a(\bw-\bv)-\bu\Vert \Vert\bw-\bv\Vert
    \mathrm{d}a
    \\
    \le&  h(\bv)+\langle \nabla h(\bu), \bw-\bv\rangle +\int_{0}^1 (L_0+L_1 \Vert \nabla h(\bu)\Vert)((1-a)\Vert \bv -\bu\Vert+ a \Vert \bw-\bu\Vert) \Vert\bw-\bv\Vert
    \mathrm{d}a
    \\
    \le & h(\bv)+\langle \nabla h(\bu), \bw-\bv\rangle + \frac{1}{2}(L_0+L_1 \Vert \nabla h(\bu)\Vert)(\Vert \bv -\bu\Vert+  \Vert \bw-\bu\Vert) \Vert\bw-\bv\Vert
   ,
\end{align*}
where Inequality $(\star)$ is due to 
\begin{equation*}
    \Vert \bv +a (\bw-\bv) -\bu\Vert =\Vert (1-a)(\bv -\bu)+a (\bw-\bu)\Vert\le (1-a)\Vert \bv -\bu\Vert+ a \Vert \bw-\bu\Vert\le \frac{1}{L_1}.
\end{equation*}
Thus the definition of $(L_0, L_1)$-smooth condition can be applied and the proof is completed.
\end{proof}

Based on Lemma \ref{lem: bounded_update}, we bound the momentum using the gradient of the current step plus some error terms.

\begin{lemma}[Estimation of the norm of the momentum]
\label{lem: estimation_momentum}
We have for all $l\in [d], k\in \mathbb{Z}^+, i \in [n]$,
\begin{align*}
    \vert \bom_{l,k,i}\vert \le& \max_{i'\in [n]} \vert \partial_l f_{i'}(\bw_{k,0})\vert+
    \left(n+\frac{2\sqrt{2}\bone}{1-\bone}\right)C_1(L_0+L_1\sqrt{D_0})\sqrt{d}\eta_k+L_1C_1\sqrt{D_1}\eta_k\sum_{j=0}^{i-1}\Vert \nabla f(\bw_{k,j}) \Vert
    \\
    &+L_1C_1\sqrt{D_1}\sum_{t=1}^{k-1}\eta_{k-t}\sum_{j=0}^{n-1} \bone^{tn+i-j}\Vert \nabla f(\bw_{k-t,j}) \Vert,
\end{align*}
where $C_1$ is defined in Eq. (\ref{eq: define_constants}).
Similarly, $l\in [d], k\in \mathbb{Z}^+/\{1\}$,
\small
\begin{equation*}
    \vert \bom_{l,k-1,n-1}\vert \le \max_{i'\in [n]} \vert \partial_l f_{i'}(\bw_{k,0})\vert+\sum_{t=1}^{k-1}\sum_{j=0}^{n-1}\bone^{tn-1-j} C_1\eta_{k-t}\sqrt{d} L_1\sqrt{D_1}\Vert \nabla f(\bw_{k-t,j})\Vert 
    +\frac{2\sqrt{2}(L_0+L_1\sqrt{D_0})C_1\sqrt{d}\eta_{k}}{1-\bone}.
\end{equation*}
\normalsize
\end{lemma}
\begin{proof}
    To begin with, for any $t\in [k-1]$ and any $j\in [0,n-1]$, we have the following estimation for $\partial_l f_i(\bw_{k-t,j})$:
    \begin{align*}
    \\
        &\vert \partial_l f_i(\bw_{k-t,j}) \vert 
    \\
    \le  & \vert \partial_l f_i(\bw_{k,0}) \vert+\sum_{p=j}^{n-1} \vert \partial_l f_i(\bw_{k-t,p})-\partial_l f_i(\bw_{k-t,p+1}) \vert+\sum_{r=1}^{t-1}\sum_{p=0}^{n-1} \vert \partial_l f_i(\bw_{k-r,p})-\partial_l f_i(\bw_{k-r,p+1}) \vert
    \\
    \overset{(\star)}{\le } & \vert \partial_l f_i(\bw_{k,0}) \vert+\sum_{p=j}^{n-1} (L_0+L_1\Vert \nabla f_i(\bw_{k-t,p})\Vert)\Vert \bw_{k-t,p} -\bw_{k-t,p+1} \Vert 
    \\
    &+
    \sum_{r=1}^{t-1}\sum_{p=0}^{n-1}  (L_0+L_1\Vert \nabla f_i(\bw_{k-r,p})\Vert)\Vert \bw_{k-r,p} -\bw_{k-r,p+1} \Vert 
    \\
    \le &
     \vert \partial_l f_i(\bw_{k,0}) \vert+\sum_{p=j}^{n-1} (L_0+L_1\Vert \nabla f_i(\bw_{k-t,p})\Vert)C_1\eta_{k-t}\sqrt{d}
    +
    \sum_{r=1}^{t-1}\sum_{p=0}^{n-1}  (L_0+L_1\Vert \nabla f_i(\bw_{k-r,p})\Vert)C_1\eta_{k-r}\sqrt{d}
    \\
     \le &
     \vert \partial_l f_i(\bw_{k,0}) \vert+\sum_{p=j}^{n-1} \left(L_0+L_1\sqrt{\sum_{i'\in [n]}\Vert \nabla f_{i'}(\bw_{k-t,p})\Vert^2}\right)C_1\eta_{k-t}\sqrt{d}
    \\
    &+
    \sum_{r=1}^{t-1}\sum_{p=0}^{n-1}  \left(L_0+L_1\sqrt{\sum_{i'\in [n]}\Vert \nabla f_{i'}(\bw_{k-r,p})\Vert^2}\right)C_1\eta_{k-r}\sqrt{d},
    \end{align*}
   where Inequality $(\star)$ is due to $(L_0,L_1)$-smooth condition.
    By Assumption \ref{assum:GC}, the RHS of the above inequality can be bounded as 
 \begin{align*}
 &
     \vert \partial_l f_i(\bw_{k,0}) \vert+\sum_{p=j}^{n-1} \left(L_0+L_1\sqrt{D_1}\Vert \nabla f(\bw_{k-t,p})\Vert+L_1\sqrt{D_0}\right)C_1\eta_{k-t}\sqrt{d}
    \\
    &+
    \sum_{r=1}^{t-1}\sum_{p=0}^{n-1}  \left(L_0+L_1\sqrt{D_1}\Vert \nabla f(\bw_{k-r,p})\Vert+L_1\sqrt{D_0}\right)C_1\eta_{k-r}\sqrt{d}
    \\
    \overset{(*)}{\le} & \vert \partial_l f_i(\bw_{k,0}) \vert+\sum_{p=j}^{n-1} L_1\sqrt{D_1}\Vert \nabla f(\bw_{k-t,p})C_1\eta_{k-t}\sqrt{d}+
    \sum_{r=1}^{t-1}\sum_{p=0}^{n-1}  L_1\sqrt{D_1}\Vert \nabla f(\bw_{k-r,p})\Vert C_1\eta_{k-r}\sqrt{d}
    \\
    &+2(L_0+L_1\sqrt{D_0})C_1\sqrt{d}\eta_{k-1} (tn-j)
    \\
    \le &\vert \partial_l f_i(\bw_{k,0}) \vert+\sum_{p=j}^{n-1} L_1\sqrt{D_1}\Vert \nabla f(\bw_{k-t,p})C_1\eta_{k-t}\sqrt{d}+
    \sum_{r=1}^{t-1}\sum_{p=0}^{n-1}  L_1\sqrt{D_1}\Vert \nabla f(\bw_{k-r,p})\Vert C_1\eta_{k-r}\sqrt{d}
    \\
    &+2\sqrt{2}(L_0+L_1\sqrt{D_0})C_1\sqrt{d}\eta_{k} (tn-j).
    \end{align*}
where  Inequality $(*)$ is due to $\forall a, b \in \mathbb{N}^{+}, a>b$, $\sum_{i=0}^b\frac{1}{\sqrt{a-i}}\le 2\frac{b+1}{a}$.
Similarly, we have that for any $j\in [0,n-1]$,
\begin{align*}
    &\vert \partial_l f_i(\bw_{k,j}) \vert 
    \le \vert \partial_l f_i(\bw_{k,0}) \vert+\sum_{p=0}^{j-1}\vert \partial_l f_i(\bw_{k,p+1}) - \partial_l f_i(\bw_{k,p})\vert
    \\
    \le &\vert \partial_l f_i(\bw_{k,0}) \vert+\sum_{p=0}^{j-1} \left(L_0+L_1\sqrt{D_1}\Vert \nabla f(\bw_{k,p})\Vert+L_1\sqrt{D_0}\right)C_1\eta_{k}\sqrt{d}
    \\
    =&\vert \partial_l f_i(\bw_{k,0}) \vert+\sum_{p=0}^{j-1} L_1\sqrt{D_1}\Vert \nabla f(\bw_{k,p})\Vert C_1\eta_{k}\sqrt{d}+j(L_0+L_1\sqrt{D}_0)C_1\sqrt{d}\eta_k.
\end{align*}

Therefore, the norm of $\bom_{l,k,i}$ can be bounded as 
\small
\begin{align*}
    &\vert \bom_{l,k,i} \vert
    \\
    \le & (1-\bone) \sum_{j=0}^i\bone^{(k-1)n+i-((k-1)n+j)} \vert \partial_l f_{\tau_{k,j}}(\bw_{k,j}) \vert +(1-\bone)\sum_{t=1}^{k-1}\sum_{j=0}^{n-1} \bone^{tn+i-j}\vert \partial_l f_{\tau_{k-t,j}} (\bw_{k-t,j}) \vert 
    \\
    &+\bone^{(k-1)n+i+1} \vert \partial_l f_{\tau_{1,0}}(\bw_{1,0}) \vert
    \\
    \le &(1-\bone)\sum_{j=0}^i\bone^{(k-1)n+i-((k-1)n+j)} \vert \partial_l f_{\tau_{k,j}}(\bw_{k,0}) \vert
    +(1-\bone)\sum_{t=1}^{k-1}\sum_{j=0}^{n-1} \bone^{tn+i-j}\vert \partial_l f_{\tau_{k-t,j}} (\bw_{k,0}) \vert
    \\
    &+\bone^{(k-1)n+i+1} \vert \partial_l f_{\tau_{1,0}}(\bw_{k,0}) \vert
    \\
    &+(1-\bone)\sum_{j=0}^i\bone^{(k-1)n+i-((k-1)n+j)} \left(\sum_{p=0}^{j-1} C_1\eta_{k}\sqrt{d}L_1\sqrt{D_1}\Vert \nabla f(\bw_{k,p})\Vert+(L_0+L_1\sqrt{D_0})C_1\eta_{k}\sqrt{d}j\right)
    \\
    &+(1-\bone)\sum_{t=1}^{k-1}\sum_{j=0}^{n-1} \bone^{tn+i-j}\left(\sum_{p=j}^{n-1} C_1\eta_{k-t}\sqrt{d}L_1\sqrt{D_1}\Vert \nabla f(\bw_{k-t,p})\Vert\right.
    \\
    &\left.+
    \sum_{r=1}^{t-1}\sum_{p=0}^{n-1} C_1\eta_{k-r}\sqrt{d} L_1\sqrt{D_1}\Vert \nabla f(\bw_{k-r,p})\Vert+2\sqrt{2}(L_0+L_1\sqrt{D_0})C_1\sqrt{d}\eta_{k} (tn-j)\right)
    \\
    &+\bone^{(k-1)n+i+1} \left(\sum_{t=1}^{k-1}\sum_{p=0}^{n-1} L_1\sqrt{D_1}\Vert \nabla f(\bw_{k-r,p})\Vert C_1\eta_{k-r}\sqrt{d}+2\sqrt{2}(L_0+L_1\sqrt{D_0})C_1\sqrt{d}\eta_{k} (k-1)n\right)
    \\
    \overset{(\star)}{\le} &  \max_{i\in [n]}\left\vert \partial_l f_i(\bw_{k,0}) \right\vert +\left(n+\frac{2\sqrt{2}\bone}{1-\bone}\right)\sqrt{d}C_1(L_0+L_1\sqrt{D_0})\eta_k+L_1C_1\sqrt{D_1}\eta_k\sum_{j=0}^{i-1}\Vert \nabla f(\bw_{k,j}) \Vert
    \\
    &+L_1C_1\sqrt{D_1}\sum_{t=1}^{k-1}\eta_{k-t}\sum_{j=0}^{n-1} \bone^{tn+i-j}\Vert \nabla f(\bw_{k-t,j}) \Vert,
\end{align*}
\normalsize
where {\color{blue}Inequality $(\star)$ is due to an exchange in the sum order}.

Following the same routine, we have 
\begin{align*}
    &\vert \bom_{l,k,-1} \vert
    \\
    \le &(1-\bone)\sum_{t=1}^{k-1}\sum_{j=0}^{n-1} \bone^{tn-1-j}\vert \partial_l f_{\tau_{k-t,j}} (\bw_{k-t,j}) \vert +\bone^{(k-1)n} \vert \partial_l f_{\tau_{1,0}}(\bw_{1,0}) \vert
    \\
    \le & (1-\bone)\sum_{t=1}^{k-1}\sum_{j=0}^{n-1} \bone^{tn-1-j}\vert \partial_l f_{\tau_{k-t,j}} (\bw_{k,0}) \vert +\bone^{(k-1)n} \vert \partial_l f_{\tau_{1,0}}(\bw_{k,0}) \vert
    \\
     &+(1-\bone)\sum_{t=1}^{k-1}\sum_{j=0}^{n-1} \bone^{tn-1-j} C_1\sqrt{d}\left(\sum_{p=j}^{n-1} L_1\sqrt{D_1}\Vert \nabla f(\bw_{k-t,p})\Vert\eta_{k-t}+  \sum_{r=1}^{t-1}\sum_{p=0}^{n-1}  L_1\sqrt{D_1}\Vert \nabla f(\bw_{k-r,p})\Vert\eta_{k-r}\right.
    \\
    &+\left.  2\sqrt{2}(L_0+L_1\sqrt{D_0})C_1\sqrt{d}\eta_{k} (tn-j)
  \right)
    \\
    &+ \bone^{(k-1)n} \left(\sum_{t=1}^{k-1}\sum_{p=0}^{n-1} L_1\sqrt{D_1}\Vert \nabla f(\bw_{k-r,p})\Vert C_1\eta_{k-r}\sqrt{d}+2\sqrt{2}(L_0+L_1\sqrt{D_0})C_1\sqrt{d}\eta_{k} (k-1)n\right)
    \end{align*}
    \begin{align*}
    \le & \max_{i\in [n]}\left\vert \partial_l f_i(\bw_{k,0}) \right\vert+\sum_{t=1}^{k-1}\sum_{j=0}^{n-1}\bone^{tn-1-j} C_1\eta_{k-t}\sqrt{d} L_1\sqrt{D_1}\Vert \nabla f(\bw_{k-t,j})\Vert~~~~~~~~~~~~~~~~~~~~~~~~~~~~~~~~~~~~~~~~~~~~~~~~~~~~
    \\
    &+\frac{2\sqrt{2}(L_0+L_1\sqrt{D_0})C_1\sqrt{d}\eta_{k}}{1-\bone}.
\end{align*}

The proof is completed.
\end{proof}

Similarly, we can  upper and lower bound the adaptor $\bnu_{k,0}$ by the gradient plus some error terms.

\begin{lemma}[Estimation of the norm of the adaptor]
\label{lem: estimation_adaptor}
We have for all $l\in [d], k\in \mathbb{Z}^+$,
\begin{align*}
    \vert \bnu_{l,k,0} \vert \ge& \btwo^n\frac{1-\btwo}{1-\btwo^n} \sum_{i\in [n]} \partial_l f_i(\bw_{k,0})^2
        -\sqrt{\sum_{i\in [n]} \vert \partial_l f_i(\bw_{k,0})^2\vert}\left( 8\sqrt{2n}\eta_k C_1L_0 \frac{1-\btwo}{(1-\btwo^n)^2}\btwo^n\right.
        \\
        &+\left.4L_1C_1\frac{1-\btwo}{1-\btwo^n}\frac{\sqrt{1-\btwo}}{1-\sqrt{\btwo}}   \left(\sum_{t=1}^{k-1}\btwo^{n} \sqrt{\btwo}^{(r-1)n}\eta_{k-t}\sum_{j=0}^{n-1} (\sqrt{D_1}\Vert \nabla f (\bw_{k-t,j}) \Vert+\sqrt{D_0})\right)\right),
\end{align*}
and 
\begin{align*}
    \vert \bnu_{l,k,0} \vert\le& 2\max_{i\in [n]}\partial_l f_i(\bw_{k,0})^2+2 \left( 2\sqrt{2} \eta_k C_1(L_0+L_1\sqrt{D_0})\frac{\sqrt{1-\btwo}}{1-\sqrt{\btwo}}\frac{\sqrt{\btwo}}{1-\sqrt{\btwo}} \right.
    \\
    &+\left. L_1C_1\sqrt{D_1}\sum_{t=1}^{k-1}\eta_{k-t}\frac{\sqrt{1-\btwo}}{1-\sqrt{\btwo}}\sum_{j=0}^{n-1}\sqrt{\btwo}^{(t-1)n} \Vert \nabla f(\bw_{k-t,j}) \Vert \right)^2,
\end{align*}
where $C_1$ is defined in Eq. (\ref{eq: define_constants}).
\end{lemma}

\begin{proof}
    By the definition of $\bnu_{l,k,0}$, we have 
    \begin{small}
    \begin{align*}
        &\bnu_{l,k,0}
        \\
        =&(1-\btwo) \partial_l f_{\tau_{k,0}} (\bw_{k,0})^2 + \sum_{t=1}^{k-1}\sum_{j=0}^{n-1} (1-\btwo)\btwo^{tn-j}  \partial_l f_{\tau_{k-t,j}} (\bw_{k-t,j})^2 +\btwo^{(k-1)n+1} \max_{i\in [n]}\partial_l f_i(\bw_{1,0})^2
        \\
        \ge & (1-\btwo) \partial_l f_{\tau_{k,0}} (\bw_{k,0})^2 + \sum_{t=1}^{k-1}\sum_{j=0}^{n-1} (1-\btwo)\btwo^{tn}  \partial_l f_{\tau_{k-t,j}} (\bw_{k-t,j})^2
        +\btwo^{(k-1)n+1} \frac{1}{n}\sum_{i=1}^n\partial_l f_i(\bw_{1,0})^2
        \\
        =&(1-\btwo) \partial_l f_{\tau_{k,0}} (\bw_{k,0})^2 + \sum_{t=1}^{k-1}\sum_{j=0}^{n-1} (1-\btwo)\btwo^{tn}  (\partial_l f_{\tau_{k-t,j}} (\bw_{k,0})+\partial_l f_{\tau_{k-t,j}} (\bw_{k-t,j})-\partial_l f_{\tau_{k-t,j}} (\bw_{k,0}))^2
        \\
        &+\btwo^{(k-1)n+1} \frac{1}{n}\sum_{i=1}^n(\partial_l f_i(\bw_{k,0})+\partial_l f_i(\bw_{1,0})-\partial_l f_i(\bw_{k,0}))^2
        \\
        \ge &(1-\btwo) \partial_l f_{\tau_{k,0}} (\bw_{k,0})^2 + \sum_{t=1}^{k-1}\sum_{j=0}^{n-1} (1-\btwo)\btwo^{tn}  \partial_l f_{\tau_{k-t,j}} (\bw_{k,0})^2
        +\btwo^{(k-1)n+1} \frac{1}{n}\sum_{i=1}^n\partial_l f_i(\bw_{k,0})^2
        \\
        &-\sum_{t=1}^{k-1}\sum_{j=0}^{n-1} (1-\btwo)\btwo^{tn}  \vert\partial_l f_{\tau_{k-t,j}} (\bw_{k,0})\vert \vert \partial_l f_{\tau_{k-t,j}} (\bw_{k,0})-\partial_l f_{\tau_{k-t,j}} (\bw_{k-t,j}) \vert
        \\
        &-\btwo^{(k-1)n+1} \frac{1}{n}\sum_{i=1}^n \vert \partial_l f_i(\bw_{k,0})\vert \vert \partial_l f_i(\bw_{k,0})-\partial_l f_i(\bw_{1,0})\vert 
        \end{align*}
        \end{small}
        Since $f_i$ is $(L_0,L_1)$-smooth, the RHS of the above inequality can be further lower bounded as follows:
\begin{small}
        \begin{align*}
        & \left(\btwo^n\frac{1-\btwo^{(k-1)n}}{1-\btwo^n}(1-\btwo)+\frac{\btwo^{(k-1)n+1}}{n}\right) \sum_{i\in [n]} \partial_l f_i(\bw_{k,0})^2
        \\
        &-\sum_{t=1}^{k-1}\sum_{j=0}^{n-1} (1-\btwo)\btwo^{tn}  \vert\partial_l f_{\tau_{k-t,j}} (\bw_{k,0})\vert\left(
    \sum_{r=1}^{t}\sum_{p=0}^{n-1}  L_1\sqrt{D_1}\Vert \nabla f(\bw_{k-r,p})\Vert C_1\eta_{k-r}\sqrt{d}+
2\sqrt{2}(L_0+L_1\sqrt{D_0})C_1\sqrt{d}\eta_{k} tn\right)
        \\
        &-\btwo^{(k-1)n+1} \frac{1}{n}\sum_{i=1}^n \vert \partial_l f_i(\bw_{k,0})\vert \left(
    \sum_{r=1}^{k-1}\sum_{p=0}^{n-1}  L_1\sqrt{D_1}\Vert \nabla f(\bw_{k-r,p})\Vert C_1\eta_{k-r}\sqrt{d}+
2\sqrt{2}(L_0+L_1\sqrt{D_0})C_1\sqrt{d}\eta_{k} (k-1)n\right)
\\
\ge & \btwo^n\frac{1-\btwo}{1-\btwo^n} \sum_{i\in [n]} \partial_l f_i(\bw_{k,0})^2
        \\
        &-\sum_{t=1}^{k-1}\sum_{j=0}^{n-1} (1-\btwo)\btwo^{tn}  \vert\partial_l f_{\tau_{k-t,j}} (\bw_{k,0})\vert\left(
    \sum_{r=1}^{t}\sum_{p=0}^{n-1}  L_1\sqrt{D_1}\Vert \nabla f(\bw_{k-r,p})\Vert C_1\eta_{k-r}\sqrt{d}+
2\sqrt{2}(L_0+L_1\sqrt{D_0})C_1\sqrt{d}\eta_{k} tn\right)
        \\
        &-\btwo^{(k-1)n+1} \frac{1}{n}\sum_{i=1}^n \vert \partial_l f_i(\bw_{k,0})\vert \left(
    \sum_{r=1}^{k-1}\sum_{p=0}^{n-1}  L_1\sqrt{D_1}\Vert \nabla f(\bw_{k-r,p})\Vert C_1\eta_{k-r}\sqrt{d}+
2\sqrt{2}(L_0+L_1\sqrt{D_0})C_1\sqrt{d}\eta_{k} (k-1)n\right),
\end{align*}
    \end{small}
where the last inequality we use $\btwo^n\frac{1-\btwo^{(k-1)n}}{1-\btwo^n}(1-\btwo)+\frac{\btwo^{(k-1)n+1}}{n}\ge \btwo^n \frac{1-\btwo}{1-\btwo^n}$.
\begin{small}
\begin{align*}
    \ge & \btwo^n\frac{1-\btwo}{1-\btwo^n} \sum_{i\in [n]} \partial_l f_i(\bw_{k,0})^2
        -8\sqrt{2}\eta_k C_1L_0 \frac{1-\btwo}{(1-\btwo^n)^2}\btwo^n\sum_{i\in [n]} \vert \partial_l f_i(\bw_{k,0})\vert
        \\
        &-4L_1C_1\frac{1-\btwo}{1-\btwo^n} \sum_{i\in [n]} \vert \partial_l f_i(\bw_{k,0}) \vert \left(\sum_{r=1}^{k-1}\btwo^{rn}\eta_{k-r}\sum_{j=0}^{n-1} \Vert \nabla f_i (\bw_{k-r,j}) \Vert\right)
        \\
        \ge & \btwo^n\frac{1-\btwo}{1-\btwo^n} \sum_{i\in [n]} \partial_l f_i(\bw_{k,0})^2
        -8\sqrt{2}\eta_k C_1L_0 \frac{1-\btwo}{(1-\btwo^n)^2}\btwo^n\sum_{i\in [n]} \vert \partial_l f_i(\bw_{k,0})\vert
        \\
        &-4L_1C_1\frac{1-\btwo}{1-\btwo^n}  \Vert \nabla f_i(\bw_{k,0}) \Vert \left(\sum_{r=1}^{k-1}\btwo^{rn}\eta_{k-r}\sum_{j=0}^{n-1} (\sqrt{D_1}\Vert \nabla f (\bw_{k-r,j}) \Vert+\sqrt{D_0})\right)
        \\
        \ge &\btwo^n\frac{1-\btwo}{1-\btwo^n} \sum_{i\in [n]} \partial_l f_i(\bw_{k,0})^2
        -8\sqrt{2n}\eta_k C_1L_0 \frac{1-\btwo}{(1-\btwo^n)^2}\btwo^n\sqrt{\sum_{i\in [n]} \vert \partial_l f_i(\bw_{k,0})^2\vert}
        \\
        &-4L_1C_1\frac{1-\btwo}{1-\btwo^n}  \sqrt{\sum_{i\in [n]} \vert \partial_l f_i(\bw_{k,0})^2\vert} \left(\sum_{r=1}^{k-1}\btwo^{rn}\eta_{k-r}\sum_{j=0}^{n-1} (\sqrt{D_1}\Vert \nabla f (\bw_{k-r,j}) \Vert+\sqrt{D_0})\right)
        \\
        \ge &\btwo^n\frac{1-\btwo}{1-\btwo^n} \sum_{i\in [n]} \partial_l f_i(\bw_{k,0})^2
        -8\sqrt{2n}\eta_k C_1L_0 \frac{1-\btwo}{(1-\btwo^n)^2}\btwo^n\sqrt{\sum_{i\in [n]} \vert \partial_l f_i(\bw_{k,0})^2\vert}
        \\
        &-4L_1C_1\frac{1-\btwo}{1-\btwo^n} \frac{\sqrt{1-\btwo}}{1-\sqrt{\btwo}}  \sqrt{\sum_{i\in [n]} \vert \partial_l f_i(\bw_{k,0})^2\vert} \left(\sum_{r=1}^{k-1}\btwo^{n}\sqrt{\btwo}^{(r-1)n}\eta_{k-r}\sum_{j=0}^{n-1} (\sqrt{D_1}\Vert \nabla f (\bw_{k-r,j}) \Vert+\sqrt{D_0})\right).
\end{align*}
\end{small}

The first claim is proved.

As for the upper bound, we have
\begin{small}
\begin{align*}
    &\bnu_{l,k,0}
    \\
    =&(1-\btwo) \partial_l f_{\tau_{k,0}} (\bw_{k,0})^2 + \sum_{t=1}^{k-1}\sum_{j=0}^{n-1} (1-\btwo)\btwo^{tn-j}  \partial_l f_{\tau_{k-t,j}} (\bw_{k-t,j})^2
    +\btwo^{(k-1)n+1} \max_{i\in [n]}\partial_l f_i(\bw_{1,0})^2
    \\
    \le & 2(1-\btwo) \partial_l f_{\tau_{k,0}} (\bw_{k,0})^2 + 2\sum_{t=1}^{k-1}\sum_{j=0}^{n-1} (1-\btwo)\btwo^{tn-j}  \partial_l f_{\tau_{k-t,j}} (\bw_{k,0})^2
    +2\btwo^{(k-1)n+1} \max_{i\in [n]}\partial_l f_i(\bw_{k,0})^2
    \\
    &+2\sum_{t=1}^{k-1}\sum_{j=0}^{n-1} (1-\btwo)\btwo^{tn-j} \left(\sum_{p=j}^{n-1} L_1\sqrt{D_1}\Vert \nabla f(\bw_{k-t,p})C_1\eta_{k-t}\sqrt{d}\right.
    \\
    &\left.+
    \sum_{r=1}^{t-1}\sum_{p=0}^{n-1}  L_1\sqrt{D_1}\Vert \nabla f(\bw_{k-r,p})\Vert C_1\eta_{k-r}\sqrt{d}
    +2\sqrt{2}(L_0+L_1\sqrt{D_0})C_1\sqrt{d}\eta_{k} (tn-j)\right)^2
    \\
    &+2\btwo^{(k-1)n+1} \left(
    \sum_{r=1}^{k-1}\sum_{p=0}^{n-1}  L_1\sqrt{D_1}\Vert \nabla f(\bw_{k-r,p})\Vert C_1\eta_{k-r}\sqrt{d}
    +2\sqrt{2}(L_0+L_1\sqrt{D_0})C_1\sqrt{d}\eta_{k} (k-1)n\right)^2
    \\
    \le & 2\max_{i\in [n]}\partial_l f_i(\bw_{k,0})^2 + 2\left(\sum_{t=1}^{k-1}\sum_{j=0}^{n-1} \sqrt{1-\btwo}\sqrt{\btwo}^{tn-j} \left(\sum_{p=j}^{n-1} L_1\sqrt{D_1}\Vert \nabla f(\bw_{k-t,p})C_1\eta_{k-t}\sqrt{d}\right.\right.
    \\
    &\left.+
    \sum_{r=1}^{t-1}\sum_{p=0}^{n-1}  L_1\sqrt{D_1}\Vert \nabla f(\bw_{k-r,p})\Vert C_1\eta_{k-r}\sqrt{d}
    +2\sqrt{2}(L_0+L_1\sqrt{D_0})C_1\sqrt{d}\eta_{k} (tn-j)\right)
    \\
    &+\left.\sqrt{\btwo}^{(k-1)n+1} \left(
    \sum_{r=1}^{k-1}\sum_{p=0}^{n-1}  L_1\sqrt{D_1}\Vert \nabla f(\bw_{k-r,p})\Vert C_1\eta_{k-r}\sqrt{d}
    +2\sqrt{2}(L_0+L_1\sqrt{D_0})C_1\sqrt{d}\eta_{k} (k-1)n\right)\right)^2
    \\
    \le & 2\max_{i\in [n]}\partial_l f_i(\bw_{k,0})^2+2 \left( 2\sqrt{2} \eta_k C_1(L_0+L_1\sqrt{D_0})\frac{\sqrt{1-\btwo}}{1-\sqrt{\btwo}}\frac{\sqrt{\btwo}}{1-\sqrt{\btwo}} \right.
    \\
    &+\left. L_1C_1\sqrt{D_1}\sum_{t=1}^{k-1}\eta_{k-t}\frac{\sqrt{1-\btwo}}{1-\sqrt{\btwo}}\sum_{j=0}^{n-1}\sqrt{\btwo}^{tn-j} \Vert \nabla f(\bw_{k-t,j}) \Vert \right)^2
    \\
    \le & 2\max_{i\in [n]}\partial_l f_i(\bw_{k,0})^2+2 \left( 2\sqrt{2} \eta_k C_1(L_0+L_1\sqrt{D_0})\frac{\sqrt{1-\btwo}}{1-\sqrt{\btwo}}\frac{\sqrt{\btwo}}{1-\sqrt{\btwo}} \right.
    \\
    &+\left. L_1C_1\sqrt{D_1}\sum_{t=1}^{k-1}\eta_{k-t}\frac{\sqrt{1-\btwo}}{1-\sqrt{\btwo}}\sum_{j=0}^{n-1}\sqrt{\btwo}^{(t-1)n} \Vert \nabla f(\bw_{k-t,j}) \Vert \right)^2.
\end{align*}
\end{small}
The proof is completed.
\end{proof}

We then immediately have the following corollary when $\max_{i\in [n]} \vert \partial_l f_i (\bw_{k,0}) \vert$ is large enough compared to the error term.

\begin{corollary}[Lemma \ref{lem: main_difference_nu}, formal]
\label{coro: large_derivative}
If 
\begin{align}
\nonumber
    \max_{i\in [n]} \vert \partial_l f_i (\bw_{k,0}) \vert\ge&
   4L_1C_1  \frac{\sqrt{1-\btwo}}{1-\sqrt{\btwo}} \left(\sum_{r=1}^{k-1}\sqrt{\btwo}^{(r-1)n}\eta_{k-r}\sum_{j=0}^{n-1} (\sqrt{D_1}\Vert \nabla f (\bw_{k-r,j}) \Vert+\sqrt{D_0})\right)
    \\
\nonumber
    &+2\sqrt{2} \eta_k C_1(L_0+L_1\sqrt{D_0})\frac{\sqrt{1-\btwo}}{1-\sqrt{\btwo}}\frac{\sqrt{\btwo}}{1-\sqrt{\btwo}} + 8\sqrt{2n}\eta_k C_1L_0 \frac{1}{1-\btwo^n}
    \\
\label{eq: large_small_crit}
    &+ \eta_k C_1\left(n(L_0+L_1\sqrt{D_0})+L_1\sqrt{D_1}\left(\sum_{p=0}^{n-1} \Vert \nabla f(\bw_{k,p}) \Vert \right)\right),
    \end{align}
    then 
    \begin{equation*}
        \frac{\btwo^n}{2} \frac{1}{n}\sum_{i\in [n]}  \partial_l f_i (\bw_{k,0})^2\le \bnu_{l,k,0} \le 4\max_{i\in [n]}  \partial_l f_i (\bw_{k,0})^2,
    \end{equation*}
    where $C_1$ is defined in Eq. (\ref{eq: define_constants}).
Furthermore, if Eq. (\ref{eq: large_small_crit}) holds, we have $\forall i \in \{0,\cdots,n-1\}$,
\begin{equation*}
    \btwo^{n-1} \bnu_{l,k,0}\le \bnu_{l,k,i}\le  \left(\btwo^{n-1}+8n\frac{1-\btwo^{n-1}}{\btwo^n}\right) \bnu_{l,k,0},
\end{equation*}
and 
\begin{equation*}
\frac{1}{\btwo}\left(1-(1-\btwo)\frac{2n}{\btwo^n}\right)\bnu_{l,k,0}\le \bnu_{l,k,-1}\le  \frac{1}{\btwo} \bnu_{l,k,0},
\end{equation*}
\end{corollary}
\begin{proof}
    The first claim is derived by directly applying the range of $  \max_{i\in [n]} \vert \partial_l f_i (\bw_{k,0}) \vert$ into Lemma \ref{lem: estimation_adaptor}.
    
    As for the second claim, we have 
    \begin{equation*}
        \bnu_{l,k,i} = \btwo^i \bnu_{l,k,0} + (1-\btwo) (\partial_l f_{\tau_{k,i}}(\bw_{k,i})^2+\cdots+\btwo^{i-1} \partial_l f_{\tau_{k,i}}(\bw_{k,1})^2).
    \end{equation*}
    On the other hand, since $\forall j \in \{0,\cdots,n-1\}$
    \begin{align*}
        \vert \partial_l f_i(\bw_{k,j}) \vert \le& \max_{p\in [n]} \vert \partial_l f_p (\bw_{k,0}) \vert +\eta_k C_1 \left(j(L_0+L_1\sqrt{D_0})+L_1\sqrt{D_1}\left(\sum_{p=0}^{j-1} \Vert \nabla f(\bw_{k,p}) \Vert \right)\right)
        \\
        \le &\max_{p\in [n]} \vert \partial_l f_p (\bw_{k,0}) \vert +\eta_k C_1 \left(n(L_0+L_1\sqrt{D_0})+L_1\sqrt{D_1}\left(\sum_{p=0}^{n-1} \Vert \nabla f(\bw_{k,p}) \Vert \right)\right),
    \end{align*}
    we have 
    \begin{small}
    \begin{align*}
        &\btwo^{n-1} \bnu_{l,k,0}\le \bnu_{l,k,i}
        \\
        \le &\btwo^i \bnu_{l,k,0} + 2(1-\btwo) \max_{p\in [n]}\partial_l f_p (\bw_{k,0})^2 (1+\cdots+\btwo^{i-1})
        \\
        &+2(1-\btwo)(1+\cdots+\btwo^{i-1})\eta_k^2 C_1^2\left(n(L_0+L_1\sqrt{D_0})+L_1\sqrt{D_1}\left(\sum_{p=0}^{n-1} \Vert \nabla f(\bw_{k,p}) \Vert \right)\right)^2
        \\
        =& \btwo^i \bnu_{l,k,0} + 2(1-\btwo^i) \max_{p\in [n]}\partial_l f_p (\bw_{k,0})^2 
       +2(1-\btwo^i)\eta_k^2 C_1^2\left(n(L_0+L_1\sqrt{D_0})+L_1\sqrt{D_1}\left(\sum_{p=0}^{n-1} \Vert \nabla f(\bw_{k,p}) \Vert \right)\right)^2.
    \end{align*}
        \end{small}
        
    Therefore, if Eq. (\ref{eq: large_small_crit}) holds,
    we then have
    \begin{align*}
      \bnu_{l,k,i} \le&  \btwo^i \bnu_{l,k,0} + 4(1-\btwo^i) \max_{p\in [n]}\partial_l f_p (\bw_{k,0})^2
      \\
      \le & \btwo^i \bnu_{l,k,0} + 4\frac{n}{n}(1-\btwo^i) \sum_{p\in [n]}\partial_l f_p (\bw_{k,0})^2
      \le \left(\btwo^i+8n\frac{1-\btwo^i}{\btwo^n}\right) \bnu_{l,k,0} 
      \\
      \le &\left(\btwo^{n-1}+8n\frac{1-\btwo^{n-1}}{\btwo^n}\right) \bnu_{l,k,0}.
    \end{align*}
    
Following the same routine, we have
\begin{equation*}
    \btwo \bnu_{l,k,-1}\le \bnu_{l,k,0},
\end{equation*}
and if Eq. (\ref{eq: large_small_crit}) holds,
\begin{align*}
    \bnu_{l,k,-1} =&\frac{1}{\btwo} \left(\bnu_{l,k,0}-(1-\btwo)  \partial_l f_{\tau_{k,0}}(\bw_{k,0})^2 \right)
    \ge \frac{1}{\btwo} \left(\bnu_{l,k,0}-(1-\btwo)\max_p  \partial_l f_{p}(\bw_{k,0})^2 \right)
    \\
    \ge& \bnu_{l,k,0} \frac{1}{\btwo} \left(1-(1-\btwo )\frac{2n}{\btwo^n}\right).
\end{align*}

The proof of the second claim is completed.
\end{proof}

\begin{remark}
\label{rm: notations}
By the notations in Eq. (\ref{eq: define_constants})., Eq. (\ref{eq: large_small_crit}) can be translated into
\begin{small}
\begin{align}
\nonumber
     \max_{i\in [n]} \vert \partial_l f_i (\bw_{k,0}) \vert\ge& C_3\eta_k+C_4 \sum_{r=1}^{k-1}\sqrt{\btwo}^{(r-1)n}\eta_{k-r}\sum_{j=0}^{n-1} \Vert \nabla f (\bw_{k-r,j}) \Vert
     \\
\label{eq: translated_crit}
     &+C_4n \sum_{r=1}^{k-1}\sqrt{\btwo}^{(r-1)n}\eta_{k-r}+\eta_k C_4\left(\sum_{j=0}^{n-1} \Vert \nabla f(\bw_{k,j}) \Vert \right).
\end{align}
\end{small}

Furthermore, we define $g(\btwo)$ as 
\begin{align*}
    g(\btwo) \triangleq \max\left\{\frac{1}{\sqrt{\btwo}^{n-1}}-1,1-\frac{1}{\sqrt{\btwo^{n-1}+8n\frac{1-\btwo^{n-1}}{\btwo^n}}},1-\sqrt{\btwo},\sqrt{\frac{\btwo}{\left(1-(1-\btwo)\frac{2n}{\btwo^n}\right)}}-1\right\},
\end{align*}
and the conclusion of Corollary \ref{coro: large_derivative} can be translated into
that if Eq. (\ref{eq: translated_crit}) holds,
\begin{equation*}
    \left\vert \frac{1}{\sqrt{\bnu_{l,k,i}}}-\frac{1}{\sqrt{\bnu_{l,k,0}}}\right\vert \le g(\btwo) \frac{1}{\sqrt{\bnu_{l,k,0}}},
\end{equation*}
and 
\begin{equation*}
   \left \vert \frac{1}{\sqrt{\bnu_{l,k,-1}}}-\frac{1}{\sqrt{\bnu_{l,k,0}}}\right\vert \le g(\btwo) \frac{1}{\sqrt{\bnu_{l,k,0}}}.
\end{equation*}
\end{remark}
Based on whether Eq. (\ref{eq: translated_crit}) is fulfilled, we divide $[d]$ into $\mathbb{L}_{large}^k$ and $\mathbb{L}_{small}^k$ ($\forall k\ge 1$), which are respectively defined as 
\begin{gather*}
    \mathbb{L}_{large}^k=\{l:l\in[d],\text{s.t. Eq. (\ref{eq: translated_crit}) holds}\},
    \\
    \mathbb{L}_{small}^k=\{l:l\in[d],\text{s.t. Eq. (\ref{eq: translated_crit}) doesn't hold}\}.
\end{gather*}

The following lemma characterizes the property of  $\mathbb{L}_{small}^k$.

\begin{lemma}
\label{lem: L_small}
Define $\bu_k\triangleq \frac{\bw_{k,0}-\bone\bw_{k,-1}}{1-\bone}$ (with $\bw_{1,-1}\triangleq \bw_{1,0}$). Then, 
    \begin{equation*}
         \sum_{k=1}^T \left\vert \sum_{l\in \mathbb{L}_{small}^k}\partial_{l}f(\bw_{k,0})  (\bu_{l,k+1}-\bu_{l,k})\right\vert \le  C_2\left(C_5 \sum_{k=1}^{T} \eta_k^2  \Vert \nabla f(\bw_{k,0}) \Vert+C_6\ln T+C_7\right),
    \end{equation*}
     where $C_2$, $C_5$, $C_6$, and $C_7$ are defined in Eq. (\ref{eq: define_constants}).
\end{lemma}
\begin{proof}
By directly applying the definition of  $\mathbb{L}_{large}^k$ and Lemma \ref{lem: u_gap}, we have
\begin{small}
\begin{align*}
    & \frac{1}{n}\left\vert \sum_{l\in \mathbb{L}_{small}^k}\partial_{l}f(\bw_{k,0})  (\bu_{l,k+1}-\bu_{l,k})\right\vert
   \\
   \le &  dC_2\eta_k\left(C_3\eta_k+C_4 \sum_{r=1}^{k-1}\sqrt{\btwo}^{(r-1)n}\eta_{k-r}\sum_{j=0}^{n-1} \Vert \nabla f (\bw_{k-r,j}) \Vert\right.
     \left.+C_4n \sum_{r=1}^{k-1}\sqrt{\btwo}^{(r-1)n}\eta_{k-r}+\eta_k C_4\left(\sum_{p=0}^{n-1} \Vert \nabla f(\bw_{k,p}) \Vert \right)\right).
\end{align*}
\end{small}
Summing over $k$ from $1$ to $t$ then leads to
\begin{small}
\begin{align*}
    &\frac{1}{n}\sum_{k=1}^T \left\vert \sum_{l\in \mathbb{L}_{small}^k}\partial_{l}f(\bw_{k,0})  (\bu_{l,k+1}-\bu_{l,k})\right\vert
    \\
    \le & \sum_{k=1}^T dC_2C_3\eta^2_k +dC_2C_4  \sum_{k=1}^T \eta_k \sum_{r=1}^{k-1} \sqrt{\btwo}^{(r-1)n} \eta_{k-r}\sum_{j=0}^{n-1} \Vert \nabla f (\bw_{k-r,j}) \Vert+C_2C_4n  \sum_{k=1}^T \eta_k\sum_{r=1}^{k-1}\sqrt{\btwo}^{(r-1)n}\eta_{k-r}
    \\
    &+ C_2C_4 \sum_{k=1}^T\eta_k^2\sum_{p=0}^{n-1} \Vert \nabla f(\bw_{k,p}) \Vert 
    \\
    \le & \sum_{k=1}^T dC_2C_3\eta^2_k + \frac{dC_2C_4 }{1-
     \sqrt{\btwo^n}}\sum_{k=1}^{T-1} \eta_k^2  \sum_{j=0}^{n-1} \Vert \nabla f (\bw_{k,j}) \Vert
    +\frac{C_2C_4n  }{1-
    \sqrt{\btwo^n}}\sum_{k=1}^{T-1} \eta_{k}^2+ C_2C_4 \sum_{k=1}^T\eta_k^2\sum_{p=0}^{n-1} \Vert \nabla f(\bw_{k,p}) \Vert 
    \\
    \le &\left(dC_2C_3+\frac{C_2C_4n  }{1-
    \sqrt{\btwo^n}}\right)\eta^2_1(1+\ln T)+\left(C_2C_4 +\frac{dC_2C_4 }{1-
    \sqrt{\btwo^n}}\right) \sum_{k=1}^{T} \eta_k^2  \sum_{j=0}^{n-1} \Vert \nabla f (\bw_{k,j}) \Vert,
\end{align*}
\end{small}
where in the second inequality we exchange the sum order. By Lemma \ref{lem: relationship_across_iteration}, the above inequality further leads to
\begin{small}
\begin{align}
\nonumber
    &\sum_{k=1}^T \left\vert \sum_{l\in \mathbb{L}_{small}^k}\partial_{l}f(\bw_{k,0})  (\bu_{l,k+1}-\bu_{l,k})\right\vert
    \\
    \nonumber
    \le &n\left(C_2C_4 +\frac{dC_2C_4 }{1-
     \sqrt{\btwo^n}}\right) \sum_{k=1}^{T} \eta_k^2  \sum_{j=0}^{n-1} \left((1+n
        \sqrt{d} C_1\eta_1L_1\sqrt{n} )\Vert \nabla f(\bw_{k,0}) \Vert+\left(nL_0+L_1\sqrt{n}\sqrt{D_0}\right)n
        \sqrt{d} C_1\eta_k\right)
    \\
    \nonumber
    &+n\left(dC_2C_3+\frac{C_2C_4n  }{1-
     \sqrt{\btwo^n}}\right)\eta_1^2(1+\ln T)
    \\
    \nonumber
    \le &n^2(1+n
        \sqrt{d} C_1\eta_1L_1\sqrt{n} )\left(C_2C_4+\frac{dC_2C_4}{1-
     \sqrt{\btwo^n}}\right) \sum_{k=1}^{T} \eta_k^2  \Vert \nabla f(\bw_{k,0}) \Vert+\left(dC_2C_3+\frac{C_2C_4n }{1-
     \sqrt{\btwo^n}}\right)\eta_1^2(1+\ln T)
    \\
    \nonumber
    &+n\left(C_2C_4+\frac{dC_2C_4}{1-
     \sqrt{\btwo^n}}\right)\left(nL_0+L_1\sqrt{n}\sqrt{D_0}\right)n^2
        \sqrt{d} C_1\sum_{k=1}^T\eta_k^3
    \\
    \nonumber
    \le &n^2(1+n
        \sqrt{d} C_1\eta_1L_1\sqrt{n}\sqrt{D_1})\left(C_2C_4+\frac{dC_2C_4\sqrt{D_1}}{1-
     \sqrt{\btwo^n}}\right) \sum_{k=1}^{T} \eta_k^2  \Vert \nabla f(\bw_{k,0}) \Vert+\left(dC_2C_3+\frac{C_2C_4n \sqrt{D_1}}{1-
     \sqrt{\btwo^n}}\right)
    \\
    \label{eq: result_case_1}
    &\times \eta^2_1(1+\ln T)+3n\left(C_2C_4+\frac{dC_2C_4}{1-
     \sqrt{\btwo^n}}\right)\left(nL_0+L_1\sqrt{n}\sqrt{D_0}\right)n^2
        \sqrt{d} C_1 \eta_1^3.
\end{align}
\end{small}

By the notations in Eq. (\ref{eq: define_constants}), the proof is completed.
\end{proof}

The next lemma characterizes the property of $\mathbb{L}^k_{large}$.

\begin{lemma}
    \label{lem: L_large}
    Define $\bu_k\triangleq \frac{\bw_{k,0}-\bone\bw_{k,-1}}{1-\bone}$ (with $\bw_{1,-1}\triangleq \bw_{1,0}$). We have
    \small
\begin{align*}
     &\sum_{k=1}^T\sum_{l\in \mathbb{L}_{large}^k}\partial_{l} f(\bw_{k,0})(\bu_{l,k+1}-\bu_{l,k})
     \\
     \le &-\sum_{k=1}^T\sum_{l\in [d]}\frac{\eta_k\partial_{l} f(\bw_{k,0})^2}{2\max_{i\in [n]} \vert \partial_l f_i(\bw_{k,0}) \vert+\xi}
   \\ 
   &+
    \sum_{k=1}^T\sum_{l\in \mathbb{L}_{large}^k} \eta_k g(\btwo)\left(n-1+\frac{1+\bone}{1-\bone}\right) \frac{\vert \partial_{l} f(\bw_{k,0})\vert}{\sqrt{\frac{\btwo^n}{2n}}\max_{i\in [n]} \vert \partial_l f_i(\bw_{k,0}) \vert+\xi} \left(\max_{i\in [n]} \vert \partial_l f_i(\bw_{k,0}) \vert\right)
    \\
    &+\left((C_8+\frac{1}{2}C_5) \sum_{k=1}^{T} \eta_k^2  \Vert \nabla f(\bw_{k,0}) \Vert+(C_9+\frac{1}{2}C_6)\ln T
    +(C_{10}+\frac{1}{2}C_7)\right).
\end{align*}
\normalsize
\end{lemma}

\begin{proof}
    Compared to the proof of Lemma \ref{lem: L_small}, the proof of this lemma is more complicated. To begin with, we provide a decomposition of $\bu_{k+1}-\bu_{k}$. According to the definition of $\bu_{k}$, we have
\begin{small}
\begin{align}
\nonumber
    &\bu_{k+1}-\bu_k
    \\
\nonumber
    =&\frac{(\bw_{k+1,0}-\bone \bw_{k+1,-1})-(\bw_{k,0}-\bone \bw_{k,-1})}{1-\bone}
    \\
\nonumber
    =&\frac{(\bw_{k+1,0}- \bw_{k,0})-\bone(\bw_{k+1,-1}- \bw_{k,-1})}{1-\bone}
    \\
\nonumber
    =&\frac{\sum_{i=0}^{n-1}(\bw_{k,i+1}- \bw_{k,i})-\bone\sum_{i=0}^{n-1}(\bw_{k,i}- \bw_{k,i-1})}{1-\bone}
    \\
\nonumber
    =&\frac{(\bw_{k+1,0}- \bw_{k+1,-1})+(1-\bone)\sum_{i=0}^{n-2}(\bw_{k,i+1}- \bw_{k,i})-\bone(\bw_{k,0}- \bw_{k,-1})}{1-\bone}
    \\
\nonumber
    \overset{(\star)}{=}&-\frac{\frac{\eta_k}{\sqrt{\bnu_{k,n-1}}}\odot\bom_{k,n-1}+(1-\bone)\sum_{i=0}^{n-2}\frac{\eta_k}{\sqrt{\bnu_{k,i}}}\odot\bom_{k,i}-\bone\frac{\eta_{k-1}}{\sqrt{\bnu_{k-1,n-1}}}\odot\bom_{k-1,n-1}}{1-\bone}
    \\
\nonumber
    =&-\frac{\eta_k}{\sqrt{\bnu_{k,0}}}\odot\frac{\bom_{k,n-1}+(1-\bone)\sum_{i=0}^{n-2}\bom_{k,i}-\bone\bom_{k-1,n-1}}{1-\bone}
    -\eta_k\left(\left(\frac{1}{\sqrt{\bnu_{k,n-1}}}-\frac{1}{\sqrt{\bnu_{k,0}}}\right)\odot\frac{\bom_{k,n-1}}{1-\bone}\right.
    \\
    \nonumber
    &~~~\left.+\sum_{i=0}^{n-2}\left(\frac{1}{\sqrt{\bnu_{k,i}}}-\frac{1}{\sqrt{\bnu_{k,0}}}\right)\odot\bom_{k,i}-\frac{\bone}{1-\bone}\left(\frac{1}{\sqrt{\bnu_{k-1,n-1}}}-\frac{1}{\sqrt{\bnu_{k,0}}}\right)\odot{\bom_{k-1,n-1}}\right)
    \\
\label{eq:decomposition_u_update}
    &~~~-\frac{\beta_1}{1-\beta_1}(\eta_{k-1}-\eta_{k})\frac{1}{\sqrt{\bnu_{k-1,n-1}}}\odot{\bom_{k-1,n-1}}.
\end{align}
\end{small}
Here equation ($\star$) is due to a direct application of the update rule of $\bw_{k,i}$.
We then analyze the above three terms respectively, namely, we define
\begin{small}
\begin{gather*}
    a^1_l\triangleq -\frac{\eta_k}{\sqrt{\bnu_{l,k,0}}}\frac{\bom_{l,k,n-1}+(1-\bone)\sum_{i=0}^{n-2}\bom_{l,k,i}-\bone\bom_{l,k-1,n-1}}{1-\bone}=-\frac{\eta_k}{\sqrt{\bnu_{l,k,0}}}\sum_{i=0}^{n-1}\partial_lf_{\tau_{k,i}} (\bw_{k,i}),
    \\
    a^2_l\triangleq-\eta_k\left(\left(\frac{1}{\sqrt{\bnu_{l,k,n-1}}}-\frac{1}{\sqrt{\bnu_{l,k,0}}}\right)\frac{\bom_{l,k,n-1}}{1-\bone}+\sum_{i=0}^{n-2}\left(\frac{1}{\sqrt{\bnu_{l,k,i}}}-\frac{1}{\sqrt{\bnu_{l,k,0}}}\right)\bom_{l,k,i}\right.~~~~~~~~~~~~~~~~~~~~~~~~
    \\
  ~~~~~~~~~~~~~~~~~~~~~~~~~~~~~~~~~~~~~~~~~~~~~~~~~~~~~~~~~~~~~~~~~~~~~~~~~~~~~~-\left.\frac{\bone}{1-\bone}\left(\frac{1}{\sqrt{\bnu_{l,k-1,n-1}}}-\frac{1}{\sqrt{\bnu_{l,k,0}}}\right){\bom_{l,k-1,n-1}}\right),
    \\
    a^3_l\triangleq -\frac{\beta_1}{1-\beta_1}(\eta_{k-1}-\eta_{k})\frac{1}{\sqrt{\bnu_{l,k-1,n-1}}}{\bom_{l,k-1,n-1}}.
\end{gather*}
\end{small}

One can then easily observe that by Eq. (\ref{eq:decomposition_u_update}),
\begin{equation*}
    \sum_{l\in \mathbb{L}_{large}^k}\partial_{l} f(\bw_{k,0})(\bu_{l,k+1}-\bu_{l,k})=\sum_{l\in \mathbb{L}_{large}^k}\partial_{l} f(\bw_{k,0})a_l^1+\sum_{l\in \mathbb{L}_{large}^k}\partial_{l} f(\bw_{k,0})a_l^2+\sum_{l\in \mathbb{L}_{large}^k}\partial_{l} f(\bw_{k,0})a_l^3.
\end{equation*}

\textbf{\color{blue}\ding{172}  Tackling Term $\sum_{l\in \mathbb{L}_{large}^k}\partial_{l} f(\bw_{k,0})a_l^1$:}

We have
\begin{small}
\begin{align*}
    &\sum_{l\in \mathbb{L}_{large}^k} \partial_{l} f(\bw_{k,0}) a_l^1
    \\
    =&-\sum_{l\in \mathbb{L}_{large}^k}\partial_{l} \frac{\eta_k}{\sqrt{\bnu_{l,k,0}}}\partial_{l} f(\bw_{k,0}) \left(\sum_{i=0}^{n-1}\partial_lf_{\tau_{k,i}} (\bw_{k,0})\right)
    -\sum_{l\in \mathbb{L}_{large}^k}\frac{\eta_k}{\sqrt{\bnu_{l,k,0}}}\partial_{l} f(\bw_{k,0}) \left(\sum_{i=0}^{n-1}(\partial_lf_{\tau_{k,i}} (\bw_{k,i})-\partial_lf_{\tau_{k,i}} (\bw_{k,0}))\right)
    \\
    =&-\sum_{l\in \mathbb{L}_{large}^k}\frac{\eta_k}{\sqrt{\bnu_{l,k,0}}}\partial_{l} f(\bw_{k,0})^2 -\sum_{l\in \mathbb{L}_{large}^k}\frac{\eta_k}{\sqrt{\bnu_{l,k,0}}}\partial_{l} f(\bw_{k,0}) \left(\sum_{i=0}^{n-1}(\partial_lf_{\tau_{k,i}} (\bw_{k,i})-\partial_lf_{\tau_{k,i}} (\bw_{k,0}))\right)
    \\
    \overset{(\star)}{=}&-\sum_{l\in \mathbb{L}_{large}^k}\frac{\eta_k}{\sqrt{\bnu_{l,k,0}}}\partial_{l} f(\bw_{k,0})^2+\mathcal{O}\left(\eta_k^2\right)+\mathcal{O}\left(\eta_k^2 \Vert \nabla f(\bw_{k,0})\Vert\right),
\end{align*}
\end{small}
where Eq. ($\star$) is due to 
\begin{small}
\begin{align*}
    &\left\vert\sum_{l\in \mathbb{L}_{large}^k}\frac{\eta_k}{\sqrt{\bnu_{l,k,0}}}\partial_{l} f(\bw_{k,0}) \left(\sum_{i=0}^{n-1}(\partial_lf_{\tau_{k,i}} (\bw_{k,i})-\partial_lf_{\tau_{k,i}} (\bw_{k,0}))\right)\right\vert
    \\
    \overset{(\ast)}{\le} &\eta_k \sqrt{\frac{2n^2}{\btwo^n}}\left(\sum_{l\in \mathbb{L}_{large}^k}\sum_{i=0}^{n-1}\vert\partial_lf_{\tau_{k,i}} (\bw_{k,i})-\partial_lf_{\tau_{k,i}} (\bw_{k,0})\vert  \right)
    \\
    \le &\eta_k \sqrt{\frac{2n^2}{\btwo^n}}\left(\sqrt{d}\sum_{i=0}^{n-1}\Vert\nabla f_{\tau_{k,i}} (\bw_{k,i})-\nabla f_{\tau_{k,i}} (\bw_{k,0})\Vert  \right)
    \\
    \overset{(\circ)}{\le} &\eta_k \sqrt{\frac{2n^2}{\btwo^n}}\sqrt{d}\sum_{i=0}^{n-1}(L_0+L_1\Vert \nabla f_{\tau_{k,i}}(\bw_{k,0})\Vert)\Vert\bw_{k,i}-\bw_{k,0}\Vert  
    \\
    \le &\eta_k \sqrt{\frac{2n^2}{\btwo^n}}\sqrt{d}(nL_0+L_1\sqrt{D_1}\sqrt{n}\Vert \nabla f(\bw_{k,0})\Vert+\sqrt{n}L_1\sqrt{D_0})n\sqrt{d}C_1\eta_k
    \\
   \overset{(\bullet)}{\le}  & \sqrt{\frac{2n^2}{\btwo^n}}d(n^2L_0+n\sqrt{n}L_1\sqrt{D_0}) C_1\eta_k^2+\eta_k^2d\sqrt{\frac{2n^2}{\btwo^n}}L_1\sqrt{D_1}n\sqrt{n}\Vert \nabla f(\bw_{k,0})\Vert.
\end{align*}
\end{small}
Here Eq. ($\ast$) is due to Corollary \ref{coro: large_derivative}, Eq. ($\circ$) is due to $f_i$ is $(L_0,L_1)$-smooth, $\forall i$, and Eq. ($\bullet$) is due to Lemma \ref{lem: bounded_update}.

\textbf{\color{blue} \ding{173} Tackling Term $\sum_{l\in \mathbb{L}_{large}^k}\partial_{l} f(\bw_{k,0})a_l^2$:}

We have for any $l\in \mathbb{L}_{max}$,
\begin{align*}
    &\vert\partial_{l} f(\bw_{k,0})a_l^2\vert
    \\
    \le & \eta_k \vert \partial_{l} f(\bw_{k,0})\vert \left(\left\vert\frac{1}{\sqrt{\bnu_{l,k,n-1}}}-\frac{1}{\sqrt{\bnu_{l,k,0}}}\right\vert\frac{\vert\bom_{l,k,n-1}\vert}{1-\bone}+\sum_{i=0}^{n-2}\left\vert\frac{1}{\sqrt{\bnu_{l,k,i}}}-\frac{1}{\sqrt{\bnu_{l,k,0}}}\right\vert \vert\bom_{l,k,i}\vert\right.
    \\
 &-\left.\frac{\bone}{1-\bone}\left\vert\frac{1}{\sqrt{\bnu_{l,k-1,n-1}}}+\frac{1}{\sqrt{\bnu_{l,k,0}}}\right\vert{\vert\bom_{l,k-1,n-1}\vert}\right)
 \\
 \overset{(\star)}{\le}&\eta_kg(\btwo) \frac{\vert \partial_{l} f(\bw_{k,0})\vert}{\sqrt{\bnu_{l,k,0}}} \left(\frac{\vert\bom_{l,k,n-1}\vert}{1-\bone}+\sum_{i=0}^{n-2}\vert\bom_{l,k,i}\vert+\frac{\bone}{1-\bone}{\vert\bom_{l,k-1,n-1}\vert}\right)
 \\
 \overset{(\ast)}{\le }& \eta_k g(\btwo)\left(n-1+\frac{1+\bone}{1-\bone}\right) \frac{\vert \partial_{l} f(\bw_{k,0})\vert}{\sqrt{\bnu_{l,k,0}}} \left(\max_{i\in [n]} \vert \partial_l f_i(\bw_{k,0}) \vert\right)
 \\
 &+\eta_k^2 g(\btwo)\left(n-1+\frac{1+\bone}{1-\bone}\right)\frac{\sqrt{2}n}{\btwo^{\frac{n}{2}}} \left(n+\frac{2\sqrt{2}\bone}{1-\bone}\right)C_1(L_0+L_1\sqrt{D_0})\sqrt{d}
 \\
 &+\eta_k^2 g(\btwo)\left(n-1+\frac{1+\bone}{1-\bone}\right)\frac{\sqrt{2}n}{\btwo^{\frac{n}{2}}} L_1C_1\sqrt{D_1}\sum_{j=0}^{n-1}\Vert \nabla f(\bw_{k,j}) \Vert
  \\
 &+\eta_k g(\btwo)\left(n-1+\frac{1+\bone}{1-\bone}\right)\frac{\sqrt{2}n}{\btwo^{\frac{n}{2}}} L_1C_1\sqrt{D_1}\sum_{t=1}^{k-1}\eta_{k-t}\sum_{j=0}^{n-1} \bone^{tn-1-j}\Vert \nabla f(\bw_{k-t,j}) \Vert,
\end{align*}
where Inequality ($\star$) is due to Corollary \ref{coro: large_derivative}, and $g(\btwo)$ is defined in Lemma \ref{rm: notations} , and Inequality $(\ast)$ is due to Lemma \ref{lem: estimation_momentum}, by which we have $\forall i \in \{-1,\cdots,n-1\}$
\begin{align*}
    \vert \bom_{l,k,i}\vert \le& \max_{i'\in [n]} \vert \partial_l f_{i'}(\bw_{k,0})\vert+
    \left(n+\frac{2\sqrt{2}\bone}{1-\bone}\right)C_1(L_0+L_1\sqrt{D_0})\sqrt{d}\eta_k+L_1C_1\sqrt{D_1}\eta_k\sum_{j=0}^{n-1}\Vert \nabla f(\bw_{k,j}) \Vert
    \\
    &+L_1C_1\sqrt{D_1}\sum_{t=1}^{k-1}\eta_{k-t}\sum_{j=0}^{n-1} \bone^{tn-1-j}\Vert \nabla f(\bw_{k-t,j}) \Vert.
\end{align*}

Therefore, summing over $\mathbb{L}_{large}^k$ and $k$ leads to  
\begin{align*}
    &  \sum_{k=1}^T \left\vert \sum_{l\in \mathbb{L}_{large}^k}\partial_{l} f(\bw_{k,0})a_l^2\right\vert
    \\
    \le & \sum_{k=1}^T \sum_{l\in \mathbb{L}_{large}^k} \eta_k g(\btwo)\left(n-1+\frac{1+\bone}{1-\bone}\right) \frac{\vert \partial_{l} f(\bw_{k,0})\vert}{\sqrt{\bnu_{l,k,0}}} \left(\max_{i\in [n]} \vert \partial_l f_i(\bw_{k,0}) \vert\right)
    \\
    &+\sum_{k=1}^T\eta_k^2 g(\btwo)\left(n-1+\frac{1+\bone}{1-\bone}\right)\frac{\sqrt{2}n}{\btwo^{\frac{n}{2}}} \left(n+\frac{2\sqrt{2}\bone}{1-\bone}\right)C_1(L_0+L_1\sqrt{D_0})d\sqrt{d}
 \\
 &+dg(\btwo)\left(n-1+\frac{1+\bone}{1-\bone}\right)\frac{\sqrt{2}n}{\btwo^{\frac{n}{2}}} L_1C_1\sqrt{D_1}\sum_{k=1}^T \eta_k^2 \sum_{j=0}^{n-1}\Vert \nabla f(\bw_{k,j}) \Vert
  \\
 &+ dg(\btwo)\left(n-1+\frac{1+\bone}{1-\bone}\right)\frac{\sqrt{2}n}{\btwo^{\frac{n}{2}}} L_1C_1\sqrt{D_1}\sum_{k=1}^T\eta_k\sum_{t=1}^{k-1}\eta_{k-t}\sum_{j=0}^{n-1} \bone^{(t-1)n}\Vert \nabla f(\bw_{k-t,j}) \Vert
\\
    \le & \sum_{k=1}^T \sum_{l\in \mathbb{L}_{large}^k} \eta_k g(\btwo)\left(n-1+\frac{1+\bone}{1-\bone}\right) \frac{\vert \partial_{l} f(\bw_{k,0})\vert}{\sqrt{\bnu_{l,k,0}}} \left(\max_{i\in [n]} \vert \partial_l f_i(\bw_{k,0}) \vert\right)
    \\
    &+ g(\btwo)\left(n-1+\frac{1+\bone}{1-\bone}\right)\frac{\sqrt{2}n}{\btwo^{\frac{n}{2}}} \left(n+\frac{2\sqrt{2}\bone}{1-\bone}\right)C_1(L_0+L_1\sqrt{D_0})d\sqrt{d}\eta_1(1+\ln T)
 \\
 &+dg(\btwo)\left(n-1+\frac{1+\bone}{1-\bone}\right)\frac{\sqrt{2}n}{\btwo^{\frac{n}{2}}} L_1C_1\sqrt{D_1}\left(1+\frac{1}{1-\btwo^n}\right)\sum_{k=1}^T \eta_k^2 \sum_{j=0}^{n-1}\Vert \nabla f(\bw_{k,j}) \Vert
 \\
 \overset{(\star)}{\le} &\sum_{k=1}^T \sum_{l\in \mathbb{L}_{large}^k} \eta_k g(\btwo)\left(n-1+\frac{1+\bone}{1-\bone}\right) \frac{\vert \partial_{l} f(\bw_{k,0})\vert}{\sqrt{\bnu_{l,k,0}}} \left(\max_{i\in [n]} \vert \partial_l f_i(\bw_{k,0}) \vert\right)
    \\
    &+ g(\btwo)\left(n-1+\frac{1+\bone}{1-\bone}\right)\frac{\sqrt{2}n}{\btwo^{\frac{n}{2}}} \left(n+\frac{2\sqrt{2}\bone}{1-\bone}\right)C_1(L_0+L_1\sqrt{D_0})d\sqrt{d}\eta_1(1+\ln T)
 \\
 &+dg(\btwo)\left(n-1+\frac{1+\bone}{1-\bone}\right)\frac{\sqrt{2}n}{\btwo^{\frac{n}{2}}} L_1C_1\sqrt{D_1}\left(1+\frac{1}{1-\btwo^n}\right)
 \\
 &\cdot\sum_{k=1}^T \eta_k^2 \sum_{j=0}^{n-1}\left((1+n
        \sqrt{d} C_1\eta_1L_1\sqrt{n}\sqrt{D_1})\Vert \nabla f(\bw_{k,0}) \Vert+\left(nL_0+L_1\sqrt{n}\sqrt{D_0}\right)n
        \sqrt{d} C_1\eta_k\right)
    \\
    \le  &\sum_{k=1}^T \sum_{l\in \mathbb{L}_{large}^k} \eta_k g(\btwo)\left(n-1+\frac{1+\bone}{1-\bone}\right) \frac{\vert \partial_{l} f(\bw_{k,0})\vert}{\sqrt{\bnu_{l,k,0}}} \left(\max_{i\in [n]} \vert \partial_l f_i(\bw_{k,0}) \vert\right)
    \\
    &+ g(\btwo)\left(n-1+\frac{1+\bone}{1-\bone}\right)\frac{\sqrt{2}n}{\btwo^{\frac{n}{2}}} \left(n+\frac{2\sqrt{2}\bone}{1-\bone}\right)C_1(L_0+L_1\sqrt{D_0})d\sqrt{d}\eta^2_1(1+\ln T)
 \\
 &+dg(\btwo)\left(n-1+\frac{1+\bone}{1-\bone}\right)\frac{\sqrt{2}n}{\btwo^{\frac{n}{2}}} L_1C_1\sqrt{D_1}\left(1+\frac{1}{1-\btwo^n}\right)(n+n^{\frac{5}{2}}
        \sqrt{d} C_1\eta_1L_1\sqrt{D_1})\sum_{k=1}^T \eta_k^2\Vert \nabla f(\bw_{k,0}) \Vert
 \\
&+3dg(\btwo)\left(n-1+\frac{1+\bone}{1-\bone}\right)\frac{\sqrt{2}n}{\btwo^{\frac{n}{2}}} L_1C_1\sqrt{D_1}\left(1+\frac{1}{1-\btwo^n}\right)n\left(nL_0+L_1\sqrt{n}\sqrt{D_0}\right)n
        \sqrt{d} C_1\eta_1^3.
\end{align*}
where Inequality $(\star)$ is due to Lemma \ref{lem: relationship_across_iteration}.

\textbf{\color{blue} \ding{174} Tackling Term $\sum_{l\in \mathbb{L}_{large}^k}\partial_{l} f(\bw_{k,0})a_l^3$:}

For any $l\in \mathbb{L}_{large}^k$,
\begin{align*}
&\vert\partial_{l} f(\bw_{k,0})a_l^3\vert
\le
    \frac{\beta_1}{1-\beta_1}\vert\eta_{k-1}-\eta_{k}\vert\frac{1}{\sqrt{\bnu_{l,k-1,n-1}}}\vert\bom_{l,k-1,n-1}\vert\vert\partial_{l} f(\bw_{k,0})\vert
\\
  \le&
    \frac{\beta_1\eta_1}{(1-\beta_1)}\frac{1}{\sqrt{k}\sqrt{k-1}(\sqrt{k}+\sqrt{k-1})}C_1\vert\partial_{l} f(\bw_{k,0})\vert
\\
 =& \frac{\beta_1\eta_k}{(1-\beta_1)}\frac{1}{\sqrt{k-1}(\sqrt{k}+\sqrt{k-1})}C_1\vert\partial_{l} f(\bw_{k,0})\vert  .
\end{align*}
Summing over $k$ and $\mathbb{L}_{large}^k$ then leads to \begin{small}
\begin{align*}
&\sum_{k=1}^T\sum_{l\in \mathbb{L}_{large}^k}\vert\partial_{l} f(\bw_{k,0})a_l^3\vert
\le
   \frac{\beta_1}{(1-\beta_1)}\sum_{k=1}^T\sum_{l\in \mathbb{L}_{large}^k}\frac{\eta_k}{\sqrt{k-1}(\sqrt{k}+\sqrt{k-1})}C_1\vert\partial_{l} f(\bw_{k,0})\vert
\\
\le &  2\frac{\beta_1}{(1-\beta_1)\eta_1}\sqrt{d}C_1\sum_{k=1}^T\eta_k^2\Vert\nabla  f(\bw_{k,0})\Vert.
\end{align*}
\end{small}
 Put \textbf{\color{blue}\ding{172}, \ding{173}, and \ding{174}} together and applying the notations in Eq. (\ref{eq: define_constants}), we then have
 \begin{small}
\begin{align}
\nonumber
    &\sum_{k=1}^T\sum_{l\in \mathbb{L}_{large}^k} \partial_l f(\bw_{k,0}) (\bu_{l,k+1}-\bu_{l,k})
    \\
\nonumber
    \le & -\sum_{k=1}^T\sum_{l\in \mathbb{L}_{large}^k}\frac{\eta_k}{\sqrt{\bnu_{l,k,0}}}\partial_{l} f(\bw_{k,0})^2+\sum_{k=1}^T\sum_{l\in \mathbb{L}_{large}^k} \eta_k g(\btwo)\left(n-1+\frac{1+\bone}{1-\bone}\right) \frac{\vert \partial_{l} f(\bw_{k,0})\vert}{\sqrt{\bnu_{l,k,0}}} \left(\max_{i\in [n]} \vert \partial_l f_i(\bw_{k,0}) \vert\right)
    \\
\label{eq: mid_proof_for_2}
    &+C_8 \sum_{k=1}^T\eta_k^2\Vert\nabla f(\bw_{k,0})\Vert +C_9 \ln T+C_{10}.
\end{align}
\end{small}

We then focus on the first two terms of the RHS of the above inequality. Specifically, we have $\forall k\ge 1$,
\small
\begin{align*}
    &\sum_{l\in \mathbb{L}_{large}^k}\frac{\eta_k\partial_{l} f(\bw_{k,0})^2}{\sqrt{\bnu_{l,k,0}}+\xi}
    -
    \sum_{l\in \mathbb{L}_{large}^k} \eta_k g(\btwo)\left(n-1+\frac{1+\bone}{1-\bone}\right) \frac{\vert \partial_{l} f(\bw_{k,0})\vert}{\sqrt{\bnu_{l,k,0}}+\xi} \left(\max_{i\in [n]} \vert \partial_l f_i(\bw_{k,0}) \vert\right)
    \\
    \overset{(\star)}{\ge} & \sum_{l\in \mathbb{L}_{large}^k}\frac{\eta_k\partial_{l} f(\bw_{k,0})^2}{\sqrt{\bnu_{l,k,0}}+\xi}
    -
    \sum_{l\in \mathbb{L}_{large}^k} \eta_k g(\btwo)\left(n-1+\frac{1+\bone}{1-\bone}\right) \frac{\vert \partial_{l} f(\bw_{k,0})\vert}{\sqrt{\frac{\btwo^n}{2n}}\max_{i\in [n]} \vert \partial_l f_i(\bw_{k,0}) \vert+\xi} \left(\max_{i\in [n]} \vert \partial_l f_i(\bw_{k,0}) \vert\right)
    \\
    \ge & \sum_{l\in \mathbb{L}_{large}^k}\frac{\eta_k\partial_{l} f(\bw_{k,0})^2}{2\max_{i\in [n]} \vert \partial_l f_i(\bw_{k,0}) \vert+\xi}
    -
    \sum_{l\in \mathbb{L}_{large}^k} \eta_k g(\btwo)\left(n-1+\frac{1+\bone}{1-\bone}\right) \frac{\vert \partial_{l} f(\bw_{k,0})\vert}{\sqrt{\frac{\btwo^n}{2n}}\max_{i\in [n]} \vert \partial_l f_i(\bw_{k,0}) \vert+\xi} \left(\max_{i\in [n]} \vert \partial_l f_i(\bw_{k,0}) \vert\right)
    \\
     \overset{(\circ)}{=} & \sum_{l\in [d]}\frac{\eta_k\partial_{l} f(\bw_{k,0})^2}{2\max_{i\in [n]} \vert \partial_l f_i(\bw_{k,0}) \vert+\xi}
    -
    \sum_{l\in \mathbb{L}_{large}^k} \eta_k g(\btwo)\left(n-1+\frac{1+\bone}{1-\bone}\right) \frac{\vert \partial_{l} f(\bw_{k,0})\vert}{\sqrt{\frac{\btwo^n}{2n}}\max_{i\in [n]} \vert \partial_l f_i(\bw_{k,0}) \vert+\xi} \left(\max_{i\in [n]} \vert \partial_l f_i(\bw_{k,0}) \vert\right)
    \\
    &-\frac{nd\eta_k}{2} \left(C_3\eta_k+C_4 \sum_{r=1}^{k-1}\sqrt{\btwo}^{(r-1)n}\eta_{k-r}\sum_{j=0}^{n-1} \Vert \nabla f (\bw_{k-r,j}) \Vert+C_4n \sum_{r=1}^{k-1}\sqrt{\btwo}^{(r-1)n}\eta_{k-r}+\eta_k C_4\sum_{j=0}^{n-1} \Vert \nabla f(\bw_{k,j}) \Vert \right),
\end{align*}
\normalsize
where Inequality $(\star)$ is due to Corollary \ref{coro: large_derivative} and Equality $(\circ)$ is due to
\small
\begin{align*}
    & \sum_{l\in \mathbb{L}_{small}^k}\frac{\eta_k\partial_{l} f(\bw_{k,0})^2}{2\max_{i\in [n]} \vert \partial_l f_i(\bw_{k,0}) \vert+\xi} \le  \sum_{l\in \mathbb{L}_{small}^k}\frac{\eta_k\partial_{l} f(\bw_{k,0})^2}{2\max_{i\in [n]} \vert \partial_l f_i(\bw_{k,0}) \vert+\xi}\le \frac{n}{2}\eta_k \sum_{l\in \mathbb{L}_{small}^k} \max_{i\in [n]} \vert \partial_l f_i(\bw_{k,0}) \vert
    \\ 
    \le & \frac{nd\eta_k}{2} \left(C_3\eta_k+C_4 \sum_{r=1}^{k-1}\sqrt{\btwo}^{(r-1)n}\eta_{k-r}\sum_{j=0}^{n-1} \Vert \nabla f (\bw_{k-r,j}) \Vert+C_4n \sum_{r=1}^{k-1}\sqrt{\btwo}^{(r-1)n}\eta_{k-r}+\eta_k C_4\sum_{j=0}^{n-1} \Vert \nabla f(\bw_{k,j}) \Vert \right).
\end{align*}
\normalsize

Summing the both sides of the above inequality then leads to
\small
\begin{align*}
    &\sum_{k=1}^T \sum_{l\in \mathbb{L}_{large}^k}\frac{\eta_k\partial_{l} f(\bw_{k,0})^2}{\sqrt{\bnu_{l,k,0}}+\xi}
    -
    \sum_{k=1}^T\sum_{l\in \mathbb{L}_{large}^k} \eta_k g(\btwo)\left(n-1+\frac{1+\bone}{1-\bone}\right) \frac{\vert \partial_{l} f(\bw_{k,0})\vert}{\sqrt{\bnu_{l,k,0}}+\xi} \left(\max_{i\in [n]} \vert \partial_l f_i(\bw_{k,0}) \vert\right)
    \\
    \ge &\sum_{k=1}^T\sum_{l\in [d]}\frac{\eta_k\partial_{l} f(\bw_{k,0})^2}{2\max_{i\in [n]} \vert \partial_l f_i(\bw_{k,0}) \vert+\xi}
    \\
    &-
    \sum_{k=1}^T\sum_{l\in \mathbb{L}_{large}^k} \eta_k g(\btwo)\left(n-1+\frac{1+\bone}{1-\bone}\right) \frac{\vert \partial_{l} f(\bw_{k,0})\vert}{\sqrt{\frac{\btwo^n}{2n}}\max_{i\in [n]} \vert \partial_l f_i(\bw_{k,0}) \vert+\xi} \left(\max_{i\in [n]} \vert \partial_l f_i(\bw_{k,0}) \vert\right)
    \\
    &-\sum_{k=1}^T\frac{nd\eta_k}{2} \left(C_3\eta_k+C_4 \sum_{r=1}^{k-1}\sqrt{\btwo}^{(r-1)n}\eta_{k-r}\sum_{j=0}^{n-1} \Vert \nabla f (\bw_{k-r,j}) \Vert+C_4n \sum_{r=1}^{k-1}\sqrt{\btwo}^{(r-1)n}\eta_{k-r}+\eta_k C_4\sum_{j=0}^{n-1} \Vert \nabla f(\bw_{k,j}) \Vert \right)
    \\
    \ge  &\sum_{k=1}^T\sum_{l\in [d]}\frac{\eta_k\partial_{l} f(\bw_{k,0})^2}{2\max_{i\in [n]} \vert \partial_l f_i(\bw_{k,0}) \vert+\xi}
    \\
    &-
    \sum_{k=1}^T\sum_{l\in \mathbb{L}_{large}^k} \eta_k g(\btwo)\left(n-1+\frac{1+\bone}{1-\bone}\right) \frac{\vert \partial_{l} f(\bw_{k,0})\vert}{\sqrt{\frac{\btwo^n}{2n}}\max_{i\in [n]} \vert \partial_l f_i(\bw_{k,0}) \vert+\xi} \left(\max_{i\in [n]} \vert \partial_l f_i(\bw_{k,0}) \vert\right)
    \\
    &-\frac{1}{2}\left(C_5 \sum_{k=1}^{T} \eta_k^2  \Vert \nabla f(\bw_{k,0}) \Vert+C_6\ln T
    +C_7\right). 
\end{align*}

Applying the above inequality back to 
Eq. (\ref{eq: mid_proof_for_2}), the proof is completed.
\end{proof}

The following lemma will be useful when translating $\langle \nabla f(\bw_{k,0}), \frac{1}{\sqrt{\bnu_{k,0}}}\odot \nabla f(\bw_{k,0})\rangle $ to $\min \left\{\frac{\Vert\nabla f(\bw_{k,0})\Vert }{\sqrt{D_1 } },\frac{\Vert\nabla f(\bw_{k,0})\Vert^2 }{\sqrt{D_0 } }\right\}$.

\begin{lemma}
\label{lem: omit_proof}
Let all conditions in Theorem \ref{thm:rate} hold. Then, 
either there exists a iteration $k\in [T]$, such that either
\begin{equation*}
    \Vert \nabla f(\bw_{k,0})\Vert
    \le
  2 \sqrt{d}(2\sqrt{2}+1)\sqrt{D_0} g(\btwo)\left(n-1+\frac{1+\bone}{1-\bone}\right) \sqrt{\frac{2n}{\btwo^n}},
\end{equation*}
or for all iteration $k\in [1, T]$, we have that
\begin{small}
\begin{align*}
    &\sum_{l\in [d]}\frac{\eta_k\partial_{l} f(\bw_{k,0})^2}{2\max_{i\in [n]} \vert \partial_l f_i(\bw_{k,0}) \vert+\xi}
    -
    \sum_{l\in [d]} \eta_k g(\btwo)\left(n-1+\frac{1+\bone}{1-\bone}\right) \frac{\vert \partial_{l} f(\bw_{k,0})\vert}{\sqrt{\frac{\btwo^n}{2n}}\max_{i\in [n]} \vert \partial_l f_i(\bw_{k,0}) \vert+\xi} \left(\max_{i\in [n]} \vert \partial_l f_i(\bw_{k,0}) \vert\right)
   \\
   \ge & \eta_k \frac{1}{2(2\sqrt{2}+1)}
    \min\left\{ \frac{ \Vert \nabla f(\bw_{k,0})\Vert}{\sqrt{D_1}}, \frac{ \Vert \nabla f(\bw_{k,0})\Vert^2}{\xi+\sqrt{D_0}}
   \right\}.
\end{align*}
\end{small}
\end{lemma}
\begin{proof}
To begin with, we have 
\begin{small}
    \begin{align*}
        &\sum_{l\in [d]}\frac{\eta_k\partial_{l} f(\bw_{k,0})^2}{2\max_{i\in [n]} \vert \partial_l f_i(\bw_{k,0}) \vert+\xi}
    -
    \sum_{l\in [d]} \eta_k g(\btwo)\left(n-1+\frac{1+\bone}{1-\bone}\right) \frac{\vert \partial_{l} f(\bw_{k,0})\vert}{\sqrt{\frac{\btwo^n}{2n}}\max_{i\in [n]} \vert \partial_l f_i(\bw_{k,0}) \vert+\xi} \left(\max_{i\in [n]} \vert \partial_l f_i(\bw_{k,0}) \vert\right)
    \\
    \overset{(\star)}{\ge } &\sum_{l\in [d]}\frac{\eta_k\partial_{l} f(\bw_{k,0})^2}{2\sqrt{D_1\Vert \nabla f(\bw_{k,0}) \Vert^2+D_0 }+\xi}
    -
    \sum_{l\in [d]} \eta_k g(\btwo)\left(n-1+\frac{1+\bone}{1-\bone}\right) \frac{\vert \partial_{l} f(\bw_{k,0})\vert}{\sqrt{\frac{\btwo^n}{2n}}\max_{i\in [n]} \vert \partial_l f_i(\bw_{k,0}) \vert+\xi} \left(\max_{i\in [n]} \vert \partial_l f_i(\bw_{k,0}) \vert\right)
    \\
    =&\frac{\eta_k\Vert \nabla f(\bw_{k,0})\Vert^2}{2\sqrt{D_1\Vert \nabla f(\bw_{k,0}) \Vert^2+D_0 }+\xi}
    -
    \sum_{l\in [d]} \eta_k g(\btwo)\left(n-1+\frac{1+\bone}{1-\bone}\right) \frac{\vert \partial_{l} f(\bw_{k,0})\vert}{\sqrt{\frac{\btwo^n}{2n}}\max_{i\in [n]} \vert \partial_l f_i(\bw_{k,0}) \vert+\xi} \left(\max_{i\in [n]} \vert \partial_l f_i(\bw_{k,0}) \vert\right),
    \end{align*}
    \end{small}
    where Inequality $(\star)$ is due to that
    \begin{align*}
        &\max_{i\in [n]} \vert \partial_l f_i(\bw_{k,0}) \vert = \sqrt{\max_{i\in [n]} \vert \partial_l f_i(\bw_{k,0})\vert^2}
        \\
        \le & \sqrt{\sum_{i\in [n]} \sum_{l'=1}^d \vert \partial_{l'} f_i(\bw_{k,0})\vert^2}= \sqrt{\sum_{i\in [n]} \Vert \nabla f_i(\bw_{k,0}) \Vert^2}\le  \sqrt{D_1 \Vert \nabla f(\bw_{k,0}) \Vert^2+D_0}.
    \end{align*}

    We respectively consider the case $\xi\le \sqrt{D_0}$ and $\xi> \sqrt{D_0}$.
    
    \textbf{\color{green}Case I: $\xi\le \sqrt{D_0}$.} In this case, we have that 
    \begin{small}
    \begin{align*}
        &\frac{\eta_k\Vert \nabla f(\bw_{k,0})\Vert^2}{2\sqrt{D_1\Vert \nabla f(\bw_{k,0}) \Vert^2+D_0 }+\xi}
    -
    \sum_{l\in [d]} \eta_k g(\btwo)\left(n-1+\frac{1+\bone}{1-\bone}\right) \frac{\vert \partial_{l} f(\bw_{k,0})\vert}{\sqrt{\frac{\btwo^n}{2n}}\max_{i\in [n]} \vert \partial_l f_i(\bw_{k,0}) \vert+\xi} \left(\max_{i\in [n]} \vert \partial_l f_i(\bw_{k,0}) \vert\right)
    \\
    \ge & \frac{\eta_k\Vert \nabla f(\bw_{k,0})\Vert^2}{2\sqrt{D_1\Vert \nabla f(\bw_{k,0}) \Vert^2+D_0 }+\sqrt{D_0}}
    -
    \sum_{l\in [d]} \eta_k g(\btwo)\left(n-1+\frac{1+\bone}{1-\bone}\right) \frac{\vert \partial_{l} f(\bw_{k,0})\vert}{\sqrt{\frac{\btwo^n}{2n}}\max_{i\in [n]} \vert \partial_l f_i(\bw_{k,0}) \vert} \left(\max_{i\in [n]} \vert \partial_l f_i(\bw_{k,0}) \vert\right)
    \\
    = &  \frac{\eta_k\Vert \nabla f(\bw_{k,0})\Vert^2}{2\sqrt{D_1\Vert \nabla f(\bw_{k,0}) \Vert^2+D_0 }+\sqrt{D_0}}
    -
    \sum_{l\in [d]} \eta_k g(\btwo)\left(n-1+\frac{1+\bone}{1-\bone}\right) \sqrt{\frac{2n}{\btwo^n}}\vert \partial_{l} f(\bw_{k,0})\vert
    \\
    \ge &  \frac{\eta_k\Vert \nabla f(\bw_{k,0})\Vert^2}{2\sqrt{D_1\Vert \nabla f(\bw_{k,0}) \Vert^2+D_0 }+\sqrt{D_0}}
    -
   \sqrt{d}\eta_k g(\btwo)\left(n-1+\frac{1+\bone}{1-\bone}\right) \sqrt{\frac{2n}{\btwo^n}}\Vert \nabla f(\bw_{k,0})\Vert.
    \end{align*}
    \end{small}
We further discuss the case depending on whether $\Vert \nabla f(\bw_{k,0}) \Vert^2\le \frac{D_0}{D_1}$ or not.

\textbf{\color{pink}Case I.1: $\Vert \nabla f(\bw_{k,0}) \Vert^2\le \frac{D_0}{D_1}$.} In this case, the last line of the above equations can be further lower bounded by
\begin{align*}
    &\frac{\eta_k\Vert \nabla f(\bw_{k,0})\Vert^2}{2\sqrt{D_1\Vert \nabla f(\bw_{k,0}) \Vert^2+D_0 }+\sqrt{D_0}}
    -
   \sqrt{d}\eta_k g(\btwo)\left(n-1+\frac{1+\bone}{1-\bone}\right) \sqrt{\frac{2n}{\btwo^n}}\Vert \nabla f(\bw_{k,0})\Vert
   \\
   \ge &\frac{\eta_k\Vert \nabla f(\bw_{k,0})\Vert^2}{(2\sqrt{2}+1)\sqrt{D_0}}
    -
   \sqrt{d}\eta_k g(\btwo)\left(n-1+\frac{1+\bone}{1-\bone}\right) \sqrt{\frac{2n}{\btwo^n}}\Vert \nabla f(\bw_{k,0})\Vert
   \\
   =& \eta_k\left(\frac{\Vert \nabla f(\bw_{k,0})\Vert}{(2\sqrt{2}+1)\sqrt{D_0}}
    -
   \sqrt{d} g(\btwo)\left(n-1+\frac{1+\bone}{1-\bone}\right) \sqrt{\frac{2n}{\btwo^n}}\right)\Vert \nabla f(\bw_{k,0})\Vert
\end{align*}

\textbf{\color{pink}Case I.2: $\Vert \nabla f(\bw_{k,0}) \Vert^2> \frac{D_0}{D_1}$.}
\begin{align*}
    &\frac{\eta_k\Vert \nabla f(\bw_{k,0})\Vert^2}{2\sqrt{D_1\Vert \nabla f(\bw_{k,0}) \Vert^2+D_0 }+\sqrt{D_0}}
    -
   \sqrt{d}\eta_k g(\btwo)\left(n-1+\frac{1+\bone}{1-\bone}\right) \sqrt{\frac{2n}{\btwo^n}}\Vert \nabla f(\bw_{k,0})\Vert
   \\
   \ge &\frac{\eta_k\Vert \nabla f(\bw_{k,0})\Vert^2}{(2\sqrt{2}+1)\sqrt{D_1}\Vert \nabla f(\bw_{k,0})\Vert}
    -
   \sqrt{d}\eta_k g(\btwo)\left(n-1+\frac{1+\bone}{1-\bone}\right) \sqrt{\frac{2n}{\btwo^n}}\Vert \nabla f(\bw_{k,0})\Vert
   \\
   =&\eta_k\left( \frac{1}{(2\sqrt{2}+1)\sqrt{D_1}}
    -
   \sqrt{d} g(\btwo)\left(n-1+\frac{1+\bone}{1-\bone}\right) \sqrt{\frac{2n}{\btwo^n}}\right)\Vert \nabla f(\bw_{k,0})\Vert
   \\
   \overset{(\ast)}{\ge} & \eta_k \frac{1}{2(2\sqrt{2}+1)\sqrt{D_1}}
    \Vert \nabla f(\bw_{k,0})\Vert,
\end{align*}
where Inequality ($\ast$) is due to the constraint on $\btwo$.

Therefore, we have either (1). there exists a iteration $k\in [T]$, such that 
\begin{equation*}
    \Vert \nabla f(\bw_{k,0})\Vert
    \le
  2 \sqrt{d}(2\sqrt{2}+1)\sqrt{D_0} g(\btwo)\left(n-1+\frac{1+\bone}{1-\bone}\right) \sqrt{\frac{2n}{\btwo^n}},
\end{equation*}
or (2).for all $k\in [1, T]$, 
\begin{small}
\begin{align*}
    &\sum_{l\in [d]}\frac{\eta_k\partial_{l} f(\bw_{k,0})^2}{2\max_{i\in [n]} \vert \partial_l f_i(\bw_{k,0}) \vert+\xi}
    -
    \sum_{l\in [d]} \eta_k g(\btwo)\left(n-1+\frac{1+\bone}{1-\bone}\right) \frac{\vert \partial_{l} f(\bw_{k,0})\vert}{\sqrt{\frac{\btwo^n}{2n}}\max_{i\in [n]} \vert \partial_l f_i(\bw_{k,0}) \vert+\xi} \left(\max_{i\in [n]} \vert \partial_l f_i(\bw_{k,0}) \vert\right)
   \\
   \ge & \eta_k \frac{1}{2(2\sqrt{2}+1)}
   \min\left\{ \frac{ \Vert \nabla f(\bw_{k,0})\Vert}{\sqrt{D_1}}, \frac{ \Vert \nabla f(\bw_{k,0})\Vert^2}{\sqrt{D_0}}
   \right\}.
\end{align*}
\end{small}

\textbf{\color{green}Case II: $\xi>\sqrt{D_0}$.} In this case, we have that 
\begin{small}
    \begin{align*}
        &\frac{\eta_k\Vert \nabla f(\bw_{k,0})\Vert^2}{2\sqrt{D_1\Vert \nabla f(\bw_{k,0}) \Vert^2+D_0 }+\xi}
    -
    \sum_{l\in [d]} \eta_k g(\btwo)\left(n-1+\frac{1+\bone}{1-\bone}\right) \frac{\vert \partial_{l} f(\bw_{k,0})\vert}{\sqrt{\frac{\btwo^n}{2n}}\max_{i\in [n]} \vert \partial_l f_i(\bw_{k,0}) \vert+\xi} \left(\max_{i\in [n]} \vert \partial_l f_i(\bw_{k,0}) \vert\right)
    \\
    \ge &\frac{\eta_k\Vert \nabla f(\bw_{k,0})\Vert^2}{2\sqrt{D_1\Vert \nabla f(\bw_{k,0}) \Vert^2+\xi^2 }+\xi}
    -
    \sum_{l\in [d]} \eta_k g(\btwo)\left(n-1+\frac{1+\bone}{1-\bone}\right) \frac{\vert \partial_{l} f(\bw_{k,0})\vert}{\sqrt{\frac{\btwo^n}{2n}}\max_{i\in [n]} \vert \partial_l f_i(\bw_{k,0}) \vert+\xi} \left(\max_{i\in [n]} \vert \partial_l f_i(\bw_{k,0}) \vert\right).
    \end{align*}
\end{small}
Similar as \textbf{\color{green} Case I}, we further divides the case regarding the value of $\Vert \nabla f(\bw_{k,0}) \Vert$.

\textbf{\color{pink} Case II.1: $D_1\Vert \nabla f(\bw_{k,0}) \Vert^2\le \xi^2$.} In this case, we have 
\begin{small}
\begin{align*}
    &\frac{\eta_k\Vert \nabla f(\bw_{k,0})\Vert^2}{2\sqrt{D_1\Vert \nabla f(\bw_{k,0}) \Vert^2+\xi^2 }+\xi}
    -
    \sum_{l\in [d]} \eta_k g(\btwo)\left(n-1+\frac{1+\bone}{1-\bone}\right) \frac{\vert \partial_{l} f(\bw_{k,0})\vert}{\sqrt{\frac{\btwo^n}{2n}}\max_{i\in [n]} \vert \partial_l f_i(\bw_{k,0}) \vert+\xi} \left(\max_{i\in [n]} \vert \partial_l f_i(\bw_{k,0}) \vert\right)
    \\
    \ge &\frac{\eta_k\Vert \nabla f(\bw_{k,0})\Vert^2}{(2\sqrt{2}+1)\xi}
    -
    \sum_{l\in [d]} \eta_k g(\btwo)\left(n-1+\frac{1+\bone}{1-\bone}\right) \frac{\vert \partial_{l} f(\bw_{k,0})\vert}{\xi} \left(\max_{i\in [n]} \vert \partial_l f_i(\bw_{k,0}) \vert\right)
    \\
    \ge &\frac{\eta_k\Vert \nabla f(\bw_{k,0})\Vert^2}{(2\sqrt{2}+1)\xi}
    -
   \eta_k g(\btwo)\left(n-1+\frac{1+\bone}{1-\bone}\right) \frac{\Vert \nabla f(\bw_{k,0})\Vert}{\xi} \sqrt{D_1\Vert \nabla f(\bw_{k,0})\Vert^2+D_0}
   \\
   =&\frac{\eta_k\Vert \nabla f(\bw_{k,0})\Vert}{\xi}\left(\frac{\Vert \nabla f(\bw_{k,0})\Vert}{2\sqrt{2}+1}
    -
   g(\btwo)\left(n-1+\frac{1+\bone}{1-\bone}\right)  \sqrt{D_1\Vert \nabla f(\bw_{k,0})\Vert^2+D_0}\right).
\end{align*}
\end{small}

\textbf{\color{pink} Case II.2: $D_1\Vert \nabla f(\bw_{k,0}) \Vert^2>\xi^2$.} This case is quite similar to \textbf{\color{pink}Case I.2}, and we have
\begin{small}
\begin{align*}
    &\frac{\eta_k\Vert \nabla f(\bw_{k,0})\Vert^2}{2\sqrt{D_1\Vert \nabla f(\bw_{k,0}) \Vert^2+\xi^2 }+\xi}
    -
    \sum_{l\in [d]} \eta_k g(\btwo)\left(n-1+\frac{1+\bone}{1-\bone}\right) \frac{\vert \partial_{l} f(\bw_{k,0})\vert}{\sqrt{\frac{\btwo^n}{2n}}\max_{i\in [n]} \vert \partial_l f_i(\bw_{k,0}) \vert+\xi} \left(\max_{i\in [n]} \vert \partial_l f_i(\bw_{k,0}) \vert\right)
    \\
    \ge &\frac{\eta_k\Vert \nabla f(\bw_{k,0})\Vert^2}{(2\sqrt{2}+1)\sqrt{D_1} \Vert \nabla f(\bw_{k,0})\Vert}
    -
    \sum_{l\in [d]} \eta_k g(\btwo)\left(n-1+\frac{1+\bone}{1-\bone}\right) \frac{\vert \partial_{l} f(\bw_{k,0})\vert}{\sqrt{\frac{\btwo^n}{2n}}\max_{i\in [n]} \vert \partial_l f_i(\bw_{k,0}) \vert} \left(\max_{i\in [n]} \vert \partial_l f_i(\bw_{k,0}) \vert\right)
    \\
    \ge &\frac{\eta_k\Vert \nabla f(\bw_{k,0})\Vert}{(2\sqrt{2}+1)\sqrt{D_1} }
    -
   \sqrt{d}\sqrt{\frac{2n}{\btwo^n}} \eta_k g(\btwo)\left(n-1+\frac{1+\bone}{1-\bone}\right) \Vert \nabla f(\bw_{k,0})\Vert
   \\
   = & \eta_k\left( \frac{1}{(2\sqrt{2}+1)\sqrt{D_1}}
    -
   \sqrt{d} g(\btwo)\left(n-1+\frac{1+\bone}{1-\bone}\right) \sqrt{\frac{2n}{\btwo^n}}\right)\Vert \nabla f(\bw_{k,0})\Vert
   \\
   \ge &\eta_k \frac{1}{2(2\sqrt{2}+1)\sqrt{D_1}}
  \Vert \nabla f(\bw_{k,0})\Vert.
\end{align*}
\end{small}

Therefore, we have either (1). there exists a iteration $k\in [T]$, such that 
\begin{equation*}
    \Vert \nabla f(\bw_{k,0})\Vert
    \le
  2 \sqrt{d}(2\sqrt{2}+1)\sqrt{D_0} g(\btwo)\left(n-1+\frac{1+\bone}{1-\bone}\right) \sqrt{\frac{2n}{\btwo^n}},
\end{equation*}
or (2). for all $k\in [1, T]$, 
\begin{small}
\begin{align*}
    &\sum_{l\in [d]}\frac{\eta_k\partial_{l} f(\bw_{k,0})^2}{2\max_{i\in [n]} \vert \partial_l f_i(\bw_{k,0}) \vert+\xi}
    -
    \sum_{l\in [d]} \eta_k g(\btwo)\left(n-1+\frac{1+\bone}{1-\bone}\right) \frac{\vert \partial_{l} f(\bw_{k,0})\vert}{\sqrt{\frac{\btwo^n}{2n}}\max_{i\in [n]} \vert \partial_l f_i(\bw_{k,0}) \vert+\xi} \left(\max_{i\in [n]} \vert \partial_l f_i(\bw_{k,0}) \vert\right)
   \\
   \ge & \eta_k \frac{1}{2(2\sqrt{2}+1)}
    \min\left\{ \frac{ \Vert \nabla f(\bw_{k,0})\Vert}{\sqrt{D_1}}, \frac{ \Vert \nabla f(\bw_{k,0})\Vert^2}{\xi}
   \right\}.
\end{align*}
\end{small}

\textbf{As a conclusion of {\color{green}Case I} and {\color{green}Case II}}, we have that either there exists a iteration $k\in [T]$, such that 
\begin{equation*}
    \Vert \nabla f(\bw_{k,0})\Vert
    \le
  2 \sqrt{d}(2\sqrt{2}+1)\sqrt{D_0} g(\btwo)\left(n-1+\frac{1+\bone}{1-\bone}\right) \sqrt{\frac{2n}{\btwo^n}},
\end{equation*}
or for all iteration $k\in [1, T]$, we have that
\begin{small}
\begin{align*}
    &\sum_{l\in [d]}\frac{\eta_k\partial_{l} f(\bw_{k,0})^2}{2\max_{i\in [n]} \vert \partial_l f_i(\bw_{k,0}) \vert+\xi}
    -
    \sum_{l\in [d]} \eta_k g(\btwo)\left(n-1+\frac{1+\bone}{1-\bone}\right) \frac{\vert \partial_{l} f(\bw_{k,0})\vert}{\sqrt{\frac{\btwo^n}{2n}}\max_{i\in [n]} \vert \partial_l f_i(\bw_{k,0}) \vert+\xi} \left(\max_{i\in [n]} \vert \partial_l f_i(\bw_{k,0}) \vert\right)
   \\
   \ge & \eta_k \frac{1}{2(2\sqrt{2}+1)}
    \min\left\{ \frac{ \Vert \nabla f(\bw_{k,0})\Vert}{\sqrt{D_1}}, \frac{ \Vert \nabla f(\bw_{k,0})\Vert^2}{\xi+\sqrt{D_0}}
   \right\}.
\end{align*}
\end{small}
The proof is completed.
\end{proof}

\subsection{Proof of Theorem \ref{thm:rate}}
\label{appen: convergence}

\begin{proof}[Proof of Theorem \ref{thm:rate}]

We start by the descent lemma of $f(\bu_{k})$. Specifically, by Lemma \ref{lem: descent}, we have
\begin{align*}
    &f(\bu_{k+1})
    \\
    \le&f(\bu_{k})+\left\langle\nabla f(\bw_{k,0}),\bu_{k+1}-\bu_{k}\right\rangle+\frac{nL_0+L_1\sum_{i\in [n]}\Vert \nabla f_i(\bw_{k,0})\Vert}{2}(\Vert \bw_{k,0}-\bu_{k} \Vert+\Vert \bw_{k,0}-\bu_{k+1} \Vert )\Vert \bu_{k+1}-\bu_{k} \Vert 
    \\
    \le &f(\bu_{k})+\left\langle\nabla f(\bw_{k,0}),\bu_{k+1}-\bu_{k}\right\rangle
    +\frac{nL_0+L_1\sum_{i\in [n]}\Vert \nabla f_i(\bw_{k,0})\Vert}{2}3C_2^2d\eta_k^2
    \\
    \le &f(\bu_{k})
    +\left\langle\nabla f(\bw_{k,0}),\bu_{k+1}-\bu_{k}\right\rangle
    +\frac{nL_0+L_1\sqrt{n}\sqrt{\sum_{i\in [n]}\Vert \nabla f_i(\bw_{k,0})\Vert^2}}{2}3C_2^2d\eta_k^2
    \\
    \le &f(\bu_{k})
    +\left\langle\nabla f(\bw_{k,0}),\bu_{k+1}-\bu_{k}\right\rangle
    +\frac{nL_0+L_1\sqrt{n}\sqrt{D_0+D_1\Vert \nabla f(\bw_{k,0})\Vert^2}}{2}3C_2^2d\eta_k^2
    \\
     \le &f(\bu_{k})
    +\left\langle\nabla f(\bw_{k,0}),\bu_{k+1}-\bu_{k}\right\rangle
    +\frac{nL_0+L_1\sqrt{n}(\sqrt{D_0}+\sqrt{D_1}\Vert \nabla f(\bw_{k,0})\Vert)}{2}3C_2^2d\eta_k^2
    \\
    \overset{(\ast)}{=}&f(\bu_{k})
    +\sum_{l\in \mathbb{L}_{large}^k}\partial_{l} f(\bw_{k,0})(\bu_{l,k+1}-\bu_{l,k})
    +\sum_{l\in \mathbb{L}_{small}^k}\partial_{l} f(\bw_{k,0})(\bu_{l,k+1}-\bu_{l,k})
    \\
    &+\frac{nL_0+L_1\sqrt{n}\sqrt{D_0}}{2}3C_2^2d\eta_k^2+\frac{3L_1\sqrt{n}\sqrt{D_1}C_2^2d\eta_k^2}{2}
    \Vert \nabla f(\bw_{k,0})\Vert.
\end{align*}

Summing the above inequality over $k$ from $1$ to $T$ then leads to
\begin{align*}
    f(\bu_{T+1})\le& f(\bu_1) +\sum_{k=1}^T\sum_{l\in \mathbb{L}_{large}^k}\partial_{l} f(\bw_{k,0})(\bu_{l,k+1}-\bu_{l,k})
    +\sum_{k=1}^T\sum_{l\in \mathbb{L}_{small}^k}\partial_{l} f(\bw_{k,0})(\bu_{l,k+1}-\bu_{l,k})
    \\
    &+\sum_{k=1}^T\frac{nL_0+L_1\sqrt{n}\sqrt{D_0}}{2}3C_2^2d\eta_k^2+\sum_{k=1}^T\frac{3L_1\sqrt{n}\sqrt{D_1}C_2^2d\eta_k^2}{2}
    \Vert \nabla f(\bw_{k,0})\Vert.
\end{align*}

Bounding the second term and the third term of the RHS of the above inequality respectively by Lemma \ref{lem: L_large} and Lemma \ref{lem: L_small}, we then arrive at
\small
\begin{align*}
    f(\bu_{T+1})\le& f(\bu_1)-\sum_{k=1}^T\sum_{l\in [d]}\frac{\eta_k\partial_{l} f(\bw_{k,0})^2}{2\max_{i\in [n]} \vert \partial_l f_i(\bw_{k,0}) \vert+\xi}
   \\ 
   &+
    \sum_{k=1}^T\sum_{l\in \mathbb{L}_{large}^k} \eta_k g(\btwo)\left(n-1+\frac{1+\bone}{1-\bone}\right) \frac{\vert \partial_{l} f(\bw_{k,0})\vert}{\sqrt{\frac{\btwo^n}{2n}}\max_{i\in [n]} \vert \partial_l f_i(\bw_{k,0}) \vert+\xi} \left(\max_{i\in [n]} \vert \partial_l f_i(\bw_{k,0}) \vert\right)
    \\
    &+\left((C_8+(\frac{1}{2}+C_2)C_5) \sum_{k=1}^{T} \eta_k^2  \Vert \nabla f(\bw_{k,0}) \Vert+(C_9+(\frac{1}{2}+C_2)C_6)\ln T
    +(C_{10}+(\frac{1}{2}+C_2)C_7)\right)
\\&+\sum_{k=1}^T\frac{nL_0+L_1\sqrt{n}\sqrt{D_0}}{2}3C_2^2d\eta_k^2+\sum_{k=1}^T\frac{3L_1\sqrt{n}\sqrt{D_1}C_2^2d\eta_k^2}{2}
    \Vert \nabla f(\bw_{k,0})\Vert.
\end{align*}
\normalsize

Suppose now there does not exist an iteration $k\in [T]$, such that
\begin{equation*}
    \Vert \nabla f(\bw_{k,0})\Vert
    \le
  2 \sqrt{d}(2\sqrt{2}+1)\sqrt{D_0} g(\btwo)\left(n-1+\frac{1+\bone}{1-\bone}\right) \sqrt{\frac{2n}{\btwo^n}},
\end{equation*}
since otherwise, the proof has been completed. By Lemma \ref{lem: omit_proof}, we then have 
\begin{small}
\begin{align*}
   & \sum_{l\in \mathbb{L}_{large}^k}\frac{\eta_k\partial_{l} f(\bw_{k,0})^2}{\sqrt{\bnu_{l,k,0}}+\xi}
    -
    \sum_{l\in \mathbb{L}_{large}^k} \eta_k g(\btwo)\left(n-1+\frac{1+\bone}{1-\bone}\right) \frac{\vert \partial_{l} f(\bw_{k,0})\vert}{\sqrt{\bnu_{l,k,0}}+\xi} \left(\max_{i\in [n]} \vert \partial_l f_i(\bw_{k,0}) \vert\right)
    \\
    \ge& \eta_k \frac{1}{2(2\sqrt{2}+1)}
    \min\left\{ \frac{ \Vert \nabla f(\bw_{k,0})\Vert}{\sqrt{D_1}}, \frac{ \Vert \nabla f(\bw_{k,0})\Vert^2}{\xi+\sqrt{D_0}}
   \right\}.
\end{align*}
\end{small}




Therefore, we have
\begin{small}
\begin{align*}
    &f(\bu_{T+1})-f(\bu_1)
    \\
    \le&- \sum_{k=1}^T \eta_k \frac{1}{2(2\sqrt{2}+1)}
    \min\left\{ \frac{ \Vert \nabla f(\bw_{k,0})\Vert}{\sqrt{D_1}}, \frac{ \Vert \nabla f(\bw_{k,0})\Vert^2}{\xi+\sqrt{D_0}}
   \right\}+\left((\frac{1}{2}+C_2)C_5+C_8\right)\sum_{k=1}^T\eta_k^2\Vert\nabla f(\bw_{k,0})\Vert
   \\
   &+\left((\frac{1}{2}+C_2)C_6+C_9\right)\ln T+ \left((\frac{1}{2}+C_2)C_7+C_{10}\right)+\sum_{k=1}^T\frac{nL_0+L_1\sqrt{n}\sqrt{D_0}}{2}3C_2^2d\eta_k^2+\sum_{k=1}^T\frac{3L_1\sqrt{n}\sqrt{D_1}C_2^2d\eta_k^2}{2}
    \Vert \nabla f(\bw_{k,0})\Vert
\\
    \le&- \sum_{k=1}^T \eta_k \frac{1}{2(2\sqrt{2}+1)}
    \min\left\{ \frac{ \Vert \nabla f(\bw_{k,0})\Vert}{\sqrt{D_1}}, \frac{ \Vert \nabla f(\bw_{k,0})\Vert^2}{\xi+\sqrt{D_0}}
   \right\}+\left((\frac{1}{2}+C_2)C_5+C_8+\frac{3L_1\sqrt{n}\sqrt{D_1}C_2^2d}{2}\right)\sum_{k=1}^T\eta_k^2\Vert\nabla f(\bw_{k,0})\Vert
   \\
   &+\left((\frac{1}{2}+C_2)C_6+C_9+\frac{nL_0+L_1\sqrt{n}\sqrt{D_0}}{2}3C_2^2d\eta_1^2\right)\ln T+ \left((\frac{1}{2}+C_2)C_7+C_{10}+\frac{nL_0+L_1\sqrt{n}\sqrt{D_0}}{2}3C_2^2d\eta_1^2\right)
   \\
   \le & \sum_{k=1}^T \eta_k \frac{1}{2(2\sqrt{2}+1)}
    \min\left\{ \frac{ \Vert \nabla f(\bw_{k,0})\Vert}{\sqrt{D_1}}, \frac{ \Vert \nabla f(\bw_{k,0})\Vert^2}{\xi+\sqrt{D_0}}
   \right\}+C_{11}\sum_{k=1}^T\eta_k^2\Vert\nabla f(\bw_{k,0})\Vert
   +C_{12}\ln T+ C_{13},
\end{align*}
\end{small}
where $C_{11}$, $C_{12}$, and $C_{13}$ is defined as 
\begin{small}
\begin{gather*}
    C_{11}\triangleq (\frac{1}{2}+C_2)C_5+C_8+\frac{3L_1\sqrt{n}\sqrt{D_1}C_2^2d}{2},
    \\
    C_{12}\triangleq (\frac{1}{2}+C_2)C_6+C_9+\frac{nL_0+L_1\sqrt{n}\sqrt{D_0}}{2}3C_2^2d\eta_1^2,
    \\
    C_{13}\triangleq (\frac{1}{2}+C_2)C_7+C_{10}+\frac{nL_0+L_1\sqrt{n}\sqrt{D_0}}{2}3C_2^2d\eta_1^2.
\end{gather*}
\end{small}
On the other hand, as for $\forall k\in [T]$,
\begin{equation*}
    \eta_k^2 \Vert \nabla f(\bw_{k,0}) \Vert \le  \frac{1}{4}\frac{\sqrt{D_0}+\xi}{\sqrt{D_1}}\eta_k^2+\frac{\sqrt{D_1}}{\sqrt{D_0}+\xi}\eta_k^2 \Vert \nabla f(\bw_{k,0}) \Vert^2,
\end{equation*}
we have that
\begin{align*}
    \eta_k^2 \Vert \nabla f(\bw_{k,0}) \Vert \le&\frac{1}{4}\frac{\sqrt{D_0}+\xi}{\sqrt{D_1}}\eta_k^2+\eta_k^2\min\left\{  \Vert \nabla f(\bw_{k,0}) \Vert , \frac{\sqrt{D_1}}{\sqrt{D_0}+\xi}\Vert \nabla f(\bw_{k,0}) \Vert^2\right\}
    \\
   =& \frac{1}{4}\frac{\sqrt{D_0}+\xi}{\sqrt{D_1}}\eta_k^2+\sqrt{D_1}\eta_k^2\min\left\{  \frac{\Vert \nabla f(\bw_{k,0})\Vert}{\sqrt{D_1}} , \frac{\Vert \nabla f(\bw_{k,0}) \Vert^2}{\sqrt{D_0}+\xi}\right\},
\end{align*}
and thus,
\begin{small}
\begin{align*}
    &f(\bu_{T+1})-f(\bu_1)
    \\
    \le & - \sum_{k=1}^T \eta_k \frac{1}{2(2\sqrt{2}+1)}
    \min\left\{ \frac{ \Vert \nabla f(\bw_{k,0})\Vert}{\sqrt{D_1}}, \frac{ \Vert \nabla f(\bw_{k,0})\Vert^2}{\xi+\sqrt{D_0}}
   \right\}+C_{11}\sum_{k=1}^T\eta_k^2\Vert\nabla f(\bw_{k,0})\Vert
   +C_{12}\ln T+ C_{13}
   \\
    \le & - \sum_{k=1}^T \eta_k \frac{1}{2(2\sqrt{2}+1)}
    \min\left\{ \frac{ \Vert \nabla f(\bw_{k,0})\Vert}{\sqrt{D_1}}, \frac{ \Vert \nabla f(\bw_{k,0})\Vert^2}{\xi+\sqrt{D_0}}
   \right\}+\frac{\sqrt{D_0}+\xi}{4\sqrt{D_1}}C_{11}\sum_{k=1}^T\eta_k^2+C_{12}\ln T
   + C_{13}
   \\
   &+\sqrt{D_1}C_{11}\sum_{k=1}^T\eta_k^2\min\left\{ \frac{ \Vert \nabla f(\bw_{k,0})\Vert}{\sqrt{D_1}}, \frac{ \Vert \nabla f(\bw_{k,0})\Vert^2}{\xi+\sqrt{D_0}}
   \right\}
   \\
   \le & - \sum_{k=1}^T \eta_k \left(\frac{1}{2(2\sqrt{2}+1)}-\sqrt{D_1}C_{11}\eta_k\right)
    \min\left\{ \frac{ \Vert \nabla f(\bw_{k,0})\Vert}{\sqrt{D_1}}, \frac{ \Vert \nabla f(\bw_{k,0})\Vert^2}{\xi+\sqrt{D_0}}
   \right\}+\left(C_{12}+\frac{\sqrt{D_0}
   +\xi}{4\sqrt{D_1}}C_{11}\eta_1^2\right)\ln T
   \\
   &
    + \left(C_{13}+\frac{\sqrt{D_0}
   +\xi}{4\sqrt{D_1}}C_{11}\eta_1^2\right)
   \\
   \le &- \sum_{k=1}^T \eta_k \frac{1}{4(2\sqrt{2}+1)}
    \min\left\{ \frac{ \Vert \nabla f(\bw_{k,0})\Vert}{\sqrt{D_1}}, \frac{ \Vert \nabla f(\bw_{k,0})\Vert^2}{\xi+\sqrt{D_0}}
   \right\}+\left(C_{12}+\frac{\sqrt{D_0}
   +\xi}{4\sqrt{D_1}}C_{11}\eta_1^2\right)\ln T
   \\
   &
    + \left(C_{13}+\frac{\sqrt{D_0}
   +\xi}{4\sqrt{D_1}}C_{11}\eta_1^2\right).
\end{align*}
\end{small}
Dividing $\sum_{k=1}^{T} \eta_k$ to the both sides of the above inequality, the proof is completed.
\end{proof}

\section{Experiment Details}
\label{appen: exper}
This section collects experiments and their corresponding settings, and is arranged as follows: to begin with, we show that Adam works well under the different reshuffling order; we then provide the experiment settings of Figure \ref{fig: attention}.
\subsection{Adam works well under different reshuffling order}
We run Adam on ResNet 110 for CIFAR 10 across different random seeds and plot the 10-run mean and variance in Figure \ref{fig: multi-experiments}. One can observe that the performance of Adam is robust with respect to random seed, and support Theorem \ref{thm:rate} in terms of trajectory-wise convergence. The experiment is based on \href{https://github.com/akamaster/pytorch_resnet_cifar10}{this repo}, where we adopt the default hyperparameters settings.

\begin{figure}[htb!]
\centering 
\includegraphics[height=4.5cm,width=12cm]{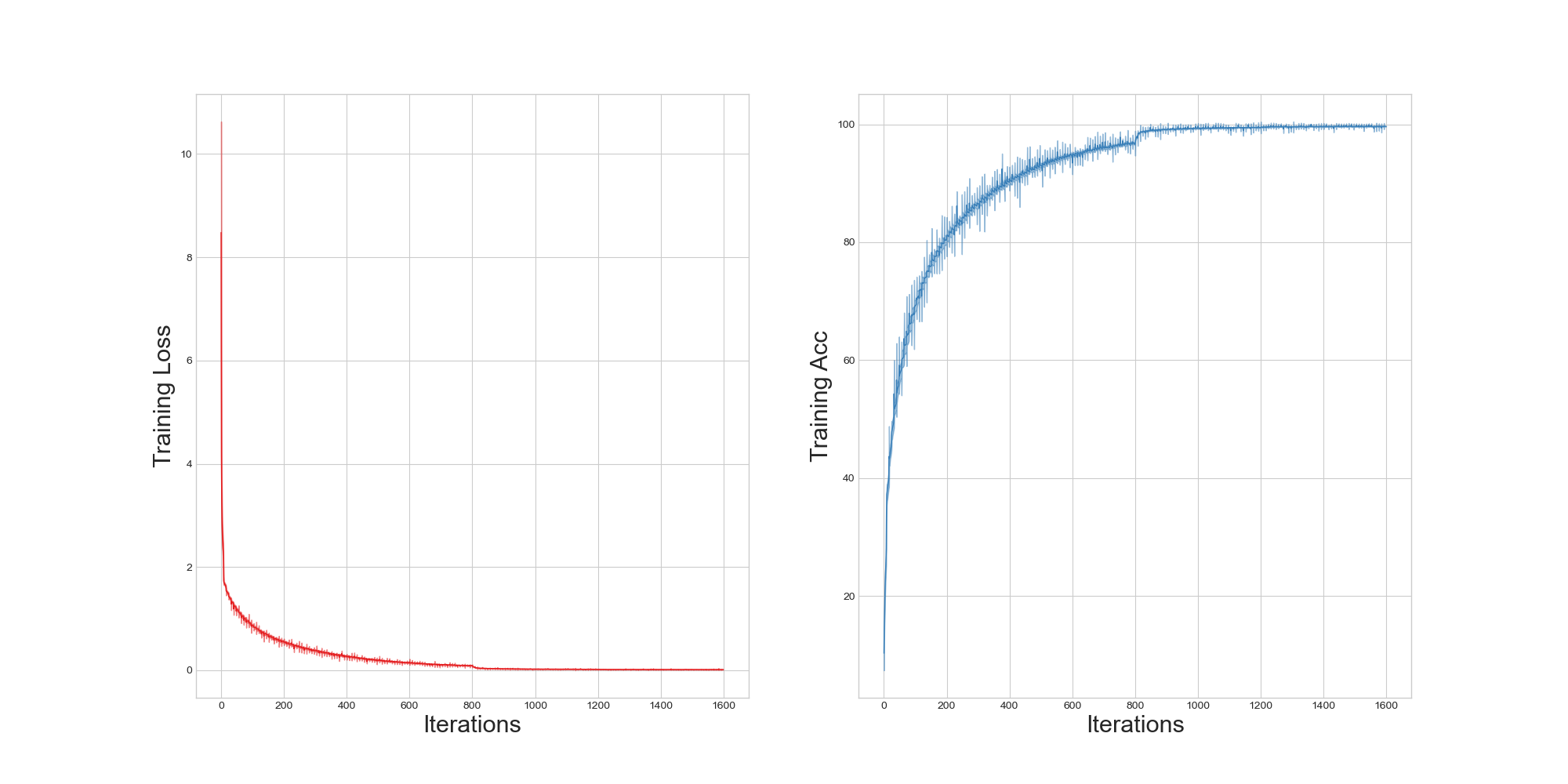}
\caption{Performance of Adam with different shuffling orders. We respectively plot the training loss and the training accuracy of Adam together with their variances over 10 runs with different random shuffling order. The result indicate the performance of Adam is robust w.r.t. the shuffling order.}
\label{fig: multi-experiments}

\end{figure}

\subsection{Local smoothness vs. gradient norm}
\label{appen: setting_details}

In this section, we provide the models and hyperparameter settings of Figures \ref{fig: attention}. We will also illustrate how we evaluate the local smoothness.

\textbf{Models and hyper-parameter settings  in Figures \ref{fig: attention}.} In Figure \ref{fig: attention}, we use exactly the same setting as \citet{vaswani2017attention}  on WMT 2014 dataset, based on \href{https://github.com/bkoch4142/attention-is-all-you-need-paper}{this repo}.

\textbf{How we evaluate the local smoothness.} We use the same method as \citet{zhang2019gradient}. Specifically, with a finite-difference step $\alpha$, we calculate the smoothness at $\bw_k$ as 
\begin{equation*}
    \text{local smoothness}=\max_{\gamma\in \{\alpha,2\alpha,\cdots,1\}} \frac{\Vert\nabla f(\bw_k+\gamma (\bw_{k+1}-\bw_k)) -\nabla f(\bw_k)\Vert}{\gamma\Vert \bw_{k+1}-\bw_k\Vert}.
\end{equation*}

\end{document}